\documentclass[12pt]{article}
\usepackage[utf8]{inputenc}
\usepackage{fullpage}
\usepackage{times}
\usepackage[normalem]{ulem}
\usepackage{fancyhdr,graphicx,amsmath,amssymb, mathtools, scrextend, titlesec, enumitem}
\usepackage{bm}
\usepackage{tcolorbox}
\usepackage{amsthm}
\usepackage{mathrsfs}  

\usepackage{amsfonts}
\numberwithin{equation}{section}
\newtheorem{theorem}{Theorem}[section]
\newtheorem{definition}{Definition}[section]

\newtheorem{corollary}{Corollary}[section]
\newtheorem{remark}{Remark}[section]
\newtheorem{lemma}{Lemma}[section]

\usepackage[hypertexnames=false]{hyperref}
\hypersetup{
    colorlinks = true,
    linkbordercolor = {white},
    linkcolor = {blue!50!black},
citecolor     = {blue!50!black},
}

\usepackage{caption}
\usepackage{subcaption}

\usepackage{algpseudocode}
\usepackage{tikz}

\usepackage{easy-todo} 

\usepackage{chngcntr}
\counterwithin{figure}{section}

\usepackage{ulem}

\usepackage{svg}
\usepackage{yhmath}

\makeatletter
\@namedef{subjclassname@2020}{%
  \textup{2020} Mathematics Subject Classification}
\makeatother

\usepackage{enumitem}
\usepackage{arxiv}
\usepackage[utf8]{inputenc}
\usepackage{ stmaryrd }
\usepackage[ruled,linesnumbered]{algorithm2e}
\usepackage{bm}

\usetikzlibrary{positioning, arrows.meta, shapes.multipart}
\newcommand{\N}{\mathcal{N}}
\newcommand{\T}{\mathcal{T}}
\newcommand{\I}{\mathcal{I}}

\newcommand{\scpr}[1]{\left\langle #1 \right\rangle}
\newcommand{\bra}[1]{\left( #1 \right)}

\newcommand{\spa}{\text{span}}
\newcommand{\anorm}[1]{\lVert #1 \rVert_{a_{\bfy,h}}}

\DeclareMathOperator\supp{supp}
\newcommand*{\jump}[1]{\ensuremath{ [\![ #1 ]\!] }}
\newcommand{\dx}{\mathrm{d}x}

\newcommand{\ay}{{a_{\bfy,h}}}

\newcommand*{\divg}{\ensuremath{\nabla\cdot}}

\newcommand*{\abs}[2][{}]{\ensuremath{#1\left\lvert#2#1\right\rvert}}
\newcommand*{\norm}[2][{}]{\ensuremath{#1\left\Vert#2#1\right\Vert}}

\newcommand*{\bfu}{\ensuremath{\mathbf{u}}}

\newcommand*{\bfx}{\ensuremath{\mathbf{x}}}
\newcommand*{\bfz}{\ensuremath{\mathbf{z}}}
\newcommand*{\bfy}{\ensuremath{\mathbf{y}}}
\newcommand*{\bfv}{\ensuremath{\mathbf{v}}}
\newcommand*{\bfw}{\ensuremath{\mathbf{w}}}
\newcommand*{\bff}{\ensuremath{\mathbf{f}}}

\newcommand*{\bfk}{{\ensuremath{\boldsymbol{\kappa}}}}

\newcommand*{\Puk}{\ensuremath{P_{U,k}}}
\newcommand*{\uimk}{{\ensuremath{\bfu^k_{\text{img}}}}}
\newcommand*{\uim}{\ensuremath{\bfu_{\text{img}}}}
\newcommand*{\utim}{\ensuremath{\tilde{\bfu}_{\text{img}}}}
\newcommand*{\ubim}{\ensuremath{\bar{\bfu}_{\text{img}}}}
\newcommand*{\utimk}{{\ensuremath{\tilde{\bfu}_{\text{img}}^k}}}
\newcommand*{\ubimk}{{\ensuremath{\bar{\bfu}_{\text{img}}^k}}}
\newcommand*{\wim}{{\ensuremath{\bfw_{\text{img}}}}}
\newcommand*{\kimk}{{\ensuremath{{{\bfk^k_{\bfy}}_{\text{img}}}}}}
\newcommand*{\fimk}{{\ensuremath{{\bff^k_{\text{img}}}}}}
\newcommand*{\stre}{\ensuremath{\mathrm{r}}}
\newcommand*{\jure}{\ensuremath{\mathrm{j}}}
\newcommand*{\no}[1]{\ensuremath{\text{nodes}(#1)}}
\newcommand*{\ed}[1]{\ensuremath{\text{edges}(#1)}}
\newcommand*{\spco}{\ensuremath{\ast^{\text{sp}}}}

\newcommand*{\tbfu}{\ensuremath{\tilde{\mathbf{u}}}}
\newcommand*{\bbfu}{\ensuremath{\bar{\mathbf{u}}}}
\newcommand*{\tbfv}{\ensuremath{\tilde{\mathbf{v}}}}
\newcommand*{\bbfv}{\ensuremath{\bar{\mathbf{v}}}}
\newcommand*{\tbfw}{\ensuremath{\tilde{\mathbf{w}}^k}}
\newcommand*{\bbfw}{\ensuremath{\bar{\mathbf{w}}^k}}
\newcommand*{\img}{\ensuremath{\text{img}}}
\newcommand*{\fem}{\ensuremath{\mathcal{C}}}
\newcommand*{\TVl}[1]{\ensuremath{\mathcal{T}_V^{#1}}}
\newcommand*{\bfeta}{\ensuremath{\boldsymbol{\eta}}}
\newcommand*{\M}{\ensuremath{\mathcal{M}}}
\newcommand*{\Iab}[1]{\ensuremath{\overline{\mathcal{I}_V^{#1}}}}
\newcommand*{\Vab}[1]{\ensuremath{\overline{V^{#1}}}}
\newcommand*{\uab}[2]{\ensuremath{\overline{\bfu_{#1}^{#2}}}}
\newcommand*{\vab}[2]{\ensuremath{\overline{\bfv_{#1}^{#2}}}}

\newcommand*{\Aab}[1]{\ensuremath{\overline{A_\bfy^{#1}}}}
\newcommand*{\tran}{\ensuremath{T}}
\newcommand*{\vimk}{\ensuremath{\bfv_{\text{img}}}^k}
\newcommand*{\FkT}{\ensuremath{{F^k}^\intercal}}

\title{Multilevel CNNs for Parametric PDEs based on Adaptive Finite Elements}

\author{
	\textbf{Janina Enrica Schütte} \\
	Weierstrass Institute for \\ Applied Analysis  and Stochastics \\
	Berlin, Germany \\
	\texttt{schuette@wias-berlin.de}
	\And
	\textbf{Martin Eigel} \\
	Weierstrass Institute for \\ Applied Analysis  and Stochastics \\
	Berlin, Germany \\
	\texttt{eigel@wias-berlin.de}
}

\renewcommand{\shorttitle}{CNNs can approximate AFEM}

\usepackage{nicematrix}
\usepackage{nicefrac}
\usepackage[nameinlink,capitalize]{cleveref}

\newtheorem{assump}[theorem]{Assumption}

\begin{document}
\maketitle

\begin{abstract}
A neural network architecture is presented that exploits the multilevel properties of high-dimensional parameter-dependent partial differential equations, enabling an efficient approximation of parameter-to-solution maps, rivaling best-in-class methods such as low-rank tensor regression in terms of accuracy and complexity.
The neural network is trained with data on adaptively refined finite element meshes, thus reducing data complexity significantly.
Error control is achieved by using a reliable finite element a posteriori error estimator, which is also provided as input to the neural network.

The proposed U-Net architecture with CNN layers mimics a classical finite element multigrid algorithm.
It can be shown that the CNN efficiently approximates all operations required by the solver, including the evaluation of the residual-based error estimator.
In the CNN, a culling mask set-up according to the local corrections due to refinement on each mesh level reduces the overall complexity, allowing the network optimization with localized fine-scale finite element data.

A complete convergence and complexity analysis is carried out for the adaptive multilevel scheme, which differs in several aspects from previous non-adaptive multilevel CNN.
Moreover, numerical experiments with common benchmark problems from Uncertainty Quantification illustrate the practical performance of the architecture.

\end{abstract}

\section{Introduction}
In recent years, the intersection of partial differential equations (PDEs) and neural networks has emerged as a powerful and promising field of research.
In a wider sense, this increasingly popular research area is called scientific machine learning (SciML), which strives to make use of modern deep learning methods for the solution of differential equations that are common to model physical processes in engineering and the natural sciences.
Opposite to many tasks in classification or generation of images, videos, sounds or text, the data in SciML typically has specific properties that can be exploited, e.g. regularity or sparsity of functions.
Moreover, data often can be generated synthetically by running (possibly computationally very costly) simulations with classical solvers such as finite elements (FE).

We consider parametric PDEs as a flexible mathematical model to describe real-world phenomena, allowing for the incorporation of variable, stochastic parameters that capture uncertainties and changing properties. Problems of this type have been examined extensively in Uncertainty Quantification (UQ) in recent years.
They can be approached with sampling methods or by computing functional surrogates in different model classes such as low-rank tensors~\cite{eigel2018adaptive,eigel2023adaptive}, by which a larger part of or the entire statistics of the quantity of interest is approximated. Neural network surrogate models in an infinite-dimensional setting have been analyzed, e.g. the DeepONet (deep operator network) architecture in~\cite{chenchen, deeponetError, deeponet, deeponetPINN, exp_deeponet, schwab2019deep}, neural operators based on model reduction in~\cite{stuart2020modelreductionnn}, and the FNO (Fourier neural operator) in~\cite{li2021fourier,JMLR:v22:21-0806} and references therein. In a discretized setting the problem is combined with reduced basis methods in~\cite{gitta, moritz, DALSANTO2020109550}. 
In~\cite{balancedafem, CABOUSSAT2024116784}, adaptively created meshes are used to train a fully connected neural network mapping, the parameter and the point in the physical domain to the evaluation of the corresponding solution.
A multilevel collocation approach to the pPDE problem can be found in~\cite{teckentrupp} and a neural network multilevel method for recovering a quantity of interest is presented in~\cite{lye2020multilevel}.

Many results on NN parameter complexity estimates for function approximation are based on the pivotal work \cite{yarotsky2017error}, where it is shown that NNs with a ReLU activation function are able to efficiently represent polynomials.
In this work (as in~\cite{cosi}) the analysis is based on an approximation of the multiplication operator with a fixed number of trainable parameters independent of the desired accuracy as shown in \cite[Corollary C.3]{mones2021}. Here, it is assumed that the activation function is three times continuously differentiable in a neighborhood of some point with nonzero second derivative, see~\cref{ass: activation function}. Then parameter bounds for an architecture as described in~\cite{schuette2024adaptive} can be derived.

\subsection{Adaptive neural network approach}
In this paper we present an approach to solve parametric PDEs based on training data generated by an adaptive FE discretization.
This is combined with a multilevel neural network (ML-Net) architecture, which mimics a classical multilevel solver and supports local corrections, corresponding to local mesh refinements.

In~\cite{cosi} the ML-Net architecture is derived to approximate the finite element coefficients of the solutions on uniformly refined grids.
We generalize this approach by introducing local corrections with respect to the global discretization mesh with data being generated by an efficient adaptive solver.
We show that CNNs are able to efficiently approximate a posteriori finite element error estimators and construct a culling mask based on the estimations, which only adds parts of the domain where fine scale corrections are needed.
By this, only small parts on each level are considered in the multigrid scheme. As a consequence of this data reduction, in principle much finer meshes (and hence a higher approximation accuracy) can be used for the training.

We consider the \emph{parametric stationary diffusion PDE} (also known as "parametric Darcy problem") with a possibly countably infinite dimensional parameter space $\Gamma\subset \mathbb{R}^{\mathbb{N}}$ and a physical domain $D\subset \mathbb{R}^d, d\in\{1,2\}$.
The objective is to find $u:D \times \Gamma\to \mathbb{R}$ such that 
\begin{align*}
    - \nabla \cdot (\kappa(\cdot,\bfy) \nabla u(\cdot, \bfy)) &= f \qquad \text{on } D,\\
    u(\cdot, \bfy) &=0 \qquad \text{on } \partial D
\end{align*}
for every $\bfy\in\Gamma$ with a parameter field $\kappa: D \times \Gamma \to \mathbb{R}$ and a right-hand side $f\in H^{-1}(D)$.

We propose an adaptive finite element solver for this task. Starting with a coarse uniform triangulation as an initial discretization, the following well-known steps are executed iteratively:
\begin{align}\label{eq: AFEM overview steps}
    \text{ Solve } \to \text{ Estimate } \to \text{ Mark } \to \text{ Refine}.
\end{align}
To solve the PDE in each iteration, we derive a successive subspace correction algorithm (SSC), for which we refer to~\cite{doi:10.1137/1034116, chen}.
The algorithm is based on a multilevel discretization of the domain.
On coarser grids, the amplitude of the functions is larger. Given sufficient regularity, it decreases quickly on finer grids, where higher frequencies of the solution have to be represented. The principle is illustrated in~\Cref{fig:param-to-sol}.
\begin{figure}
    \centering
    {
\begin{tikzpicture}
    \node (kappa) at (0,3.8) {\includegraphics[width=4.3cm]{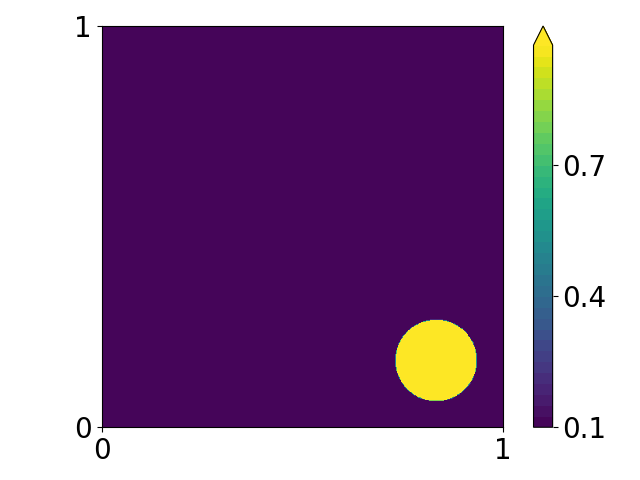}};
    \draw[|->] (2.8,3.8) -- (4.8,3.8) {};
  \node (img4) at (7.55,3.8) {\includegraphics[width=4cm]{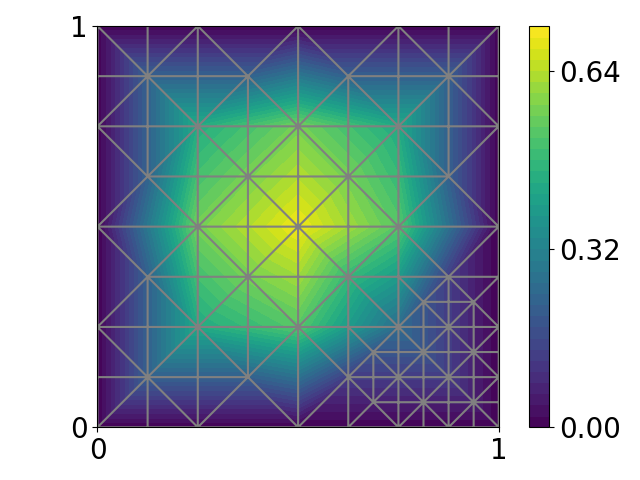}};
  
  \node (img3) at (7.7,0) {\includegraphics[width=4cm]{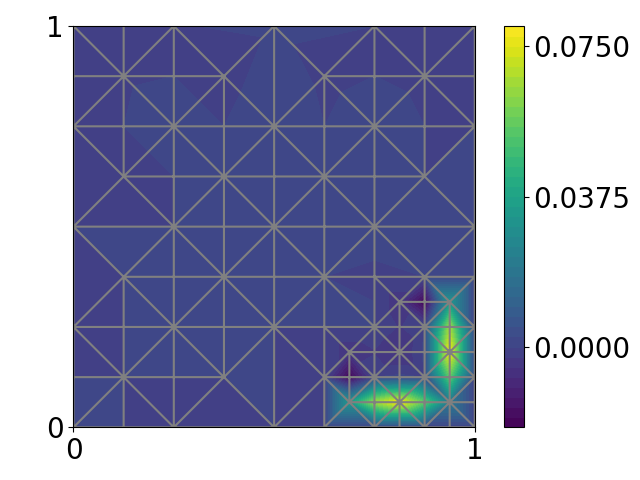}};
  
  \node (img2) at (3.8,0) {\includegraphics[width=4cm]{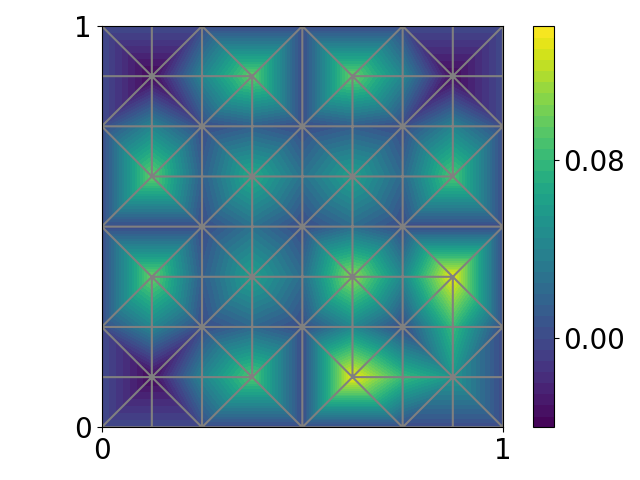}};
  
  \node (img1) at (0,0) {\includegraphics[width=4.2cm]{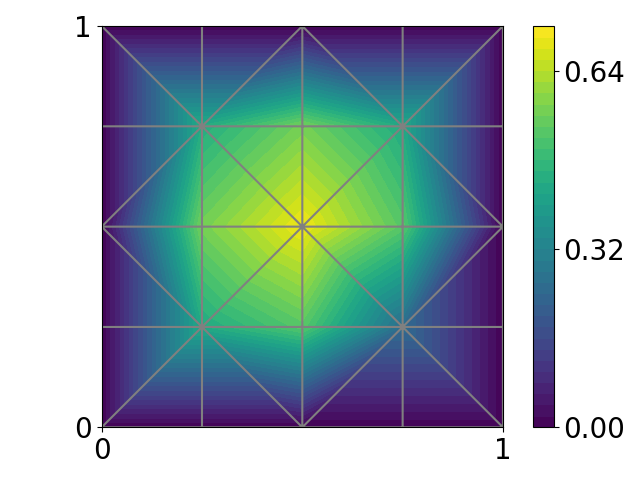}};

  \node (plus1) at (1.8,1.7) {+};
  
  \node (plus2) at (5.7,1.7) {+};
  
    \node (u) at (-2.5,0) {$u(\cdot,\bfy) = $};

  \node (v1) at (0,1.7) {$v_1(\cdot,\bfy)$};
  \node (v1) at (3.6,1.7) {$v_2(\cdot,\bfy)$};
  \node (v1) at (7.4,1.7) {$v_3(\cdot,\bfy)$};
  \node (v1) at (7.7,5.5) {$u(\cdot,\bfy)$};
    \node at (0,5.5) {$\kappa(\cdot,\bfy)$};
\end{tikzpicture}
}
    \caption{The first row depicts the parameter $\kappa$ to solution $u$ map for a realization of the parameter vector $\bfy\in\Gamma$ for~\eqref{eq: darcy linear equation system}. In the second row, the multigrid decomposition of the solution into a coarse grid function $v_1$ and finer grid corrections $v_2,v_3$ is visualized.}
    \label{fig:param-to-sol}
\end{figure}
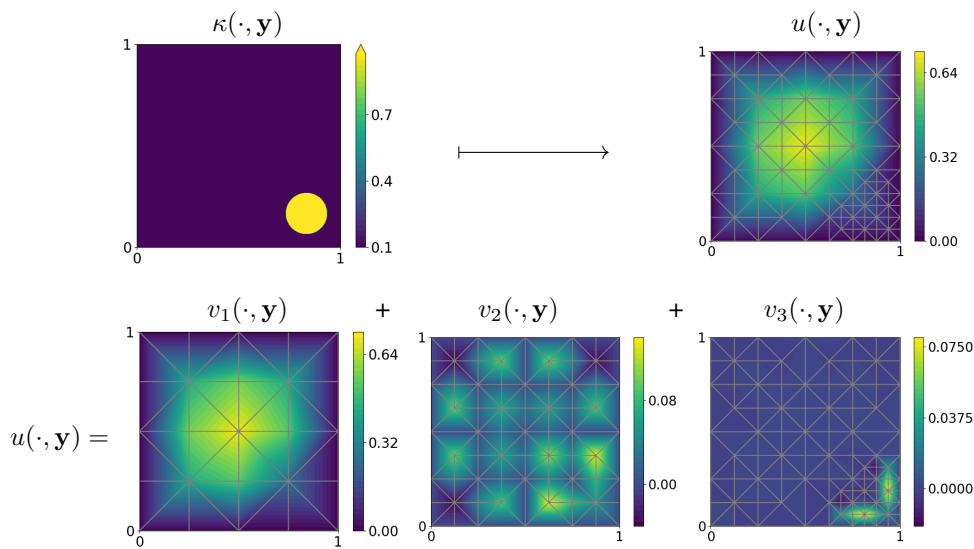
To estimate the approximation error in the energy norm, a classical residual based finite element error estimator is implemented.
The local error on each triangle $T$ with side length $h_T$ is bounded by
\begin{align*}
    \eta_T^2 \coloneqq h_T^2\norm{f+\nabla\cdot(\kappa(\cdot, \bfy)\nabla u)}_{L_2(T)}^2 + h_T \norm{\jump{\kappa(\cdot, \bfy)\nabla u}}_{L_2(\partial T)}^2.
\end{align*}
Given an estimation of the local error contributions, triangles are selected for refinement e.g. with a D\"orfler or a threshold marking strategy.
The approximation space is enriched by adding nodal basis functions from a uniformly refined grid to the current basis.
The structure of such a non-standard approximation space is illustrated in~\Cref{fig: uniform vs local meshes}.
It can be seen that the constructed space consists of a selection of FE basis functions on different levels, which does not resemble a typical FE space.
\begin{figure}
    \centering\includegraphics[width=0.5\linewidth]{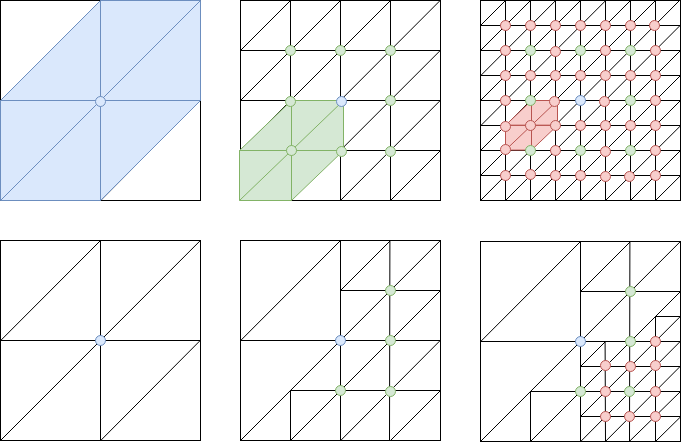}
    \caption{In the top row the support of the considered nodal basis functions on different levels is visualized. Uniformly refined meshes as used in \cite{cosi} are shown in the top row, locally refined meshes as used in this work in the bottom row. The local refinement is realized by using a subset of the nodes in the uniformly refined meshes.}
    \label{fig: uniform vs local meshes}
\end{figure}

We derive a suitable CNN architecture based on U-Nets for the problem and show that the architecture is expressive enough to accurately approximate each step in the adaptive solver in \cref{section: AFEM}. The local refinements are incorporated in a the network by $0/1$-masks. Our main result is summarized as follows.

\begin{theorem}[CNNs can approximate adaptive finite element solvers]
    Assume that $\kappa$ is uniformly bounded from below and above. Let $\varepsilon>0$ and $K, L\in\mathbb{N}$ be the number of iterations of the derived AFEM and the maximal refinements of each triangle, respectively. Consider a threshold marking strategy. Then there exists a CNN $\Psi$ such that the number of parameters of the network is in $\mathcal{O}(LK \log(\varepsilon^{-1})/\log(c_L^{-1}))$ with $c_L \coloneqq \frac{cL}{1+cL}, c>0$.
    Moreover, for any $\bfy\in\Gamma$ the network maps finite element coefficients of the parameter field $\kappa$ to coefficients of a finite element approximation of the solution $u$ of the pPDE~\eqref{eq: darcy linear equation system} such that
    \begin{center}
        $\norm{u(\cdot,\bfy) - \fem(\Psi(\bfk_\bfy, \bff))}_{H^1(D)} \leq \norm{u(\cdot,\bfy) - \fem(\mathrm{AFEM}(V_1, K))}_{H^1(D)} + \varepsilon,$ 
    \end{center}
    where $\mathcal{C}$ maps the finite element coefficients to the corresponding function.
\end{theorem}

The proof is based on the observation that U-Nets are able to approximate a successive subspace algorithm based on a multigrid decomposition. The required interpolation between grids can be achieved efficiently with strided and transpose strided convolutions.
Furthermore, we show that the estimator can be approximated and refinement in each step of the adaptive algorithm.
This is realized by $0/1$-masks multiplied by large images making it possible to work with sparse images for fine grids. 

Since the solution representation is based on local contributions (corresponding to small regions of images with higher resolution), the proposed architecture can significantly improve computational efficiency by exploiting representation sparsity.

\subsection{Main contributions}
A multigrid solver, error estimator and refinement strategy are chosen, such that the corresponding $\mathrm{AFEM}$ can provably be approximated by an introduced CNN architecture.
Complexity bounds for the approximation of the $\mathrm{AFEM}$ are shown for the architecture. In the course of the proof it is shown that CNNs can approximate a multigrid solver on locally refined grids.
Moreover, CNNs can approximate the error estimator and the refinement can be incorporated by a novel error estimator based masking.
This leads to an implicitly adaptive CNN tracking the error of individual outputs by error estimator prediction.
The sparsity introduced by the masks on high resolution grids leads to a smaller number of operations on each grid and a higher accuracy for the same number of nonzero entries.

In contrast to the work carried out in~\cite{cosi}, here the proof relies on actions on only parts of each image in the CNNs computations. Therefore, on each level, the boundary of the subsets has to be considered carefully and different index sets need to be considered in the analysis to always be able to represent all necessary information on each grid.
In practice, the locality translates to manifold sparse convolutions, where kernels are only applied to nonzero entries of each image.

\subsection{Structure of the paper}
After the problem statement and a short finite element introduction, the individual steps of the $\mathrm{AFEM}$ and the algorithm itself are introduced in~\cref{section: FEM}.
The multigrid solver in the algorithm is explained in detail and its convergence is shown in~\cref{section: solver}. \cref{section: CNN} is concerned with the the used data decomposition of continuous and discontinuous finite element functions to represent the solutions and estimators as images in the CNN. Furthermore, types of convolutions are discussed shortly.
The main results, i.e. the expressivity theorems for the solver, the estimator and the whole $\mathrm{AFEM}$ algorithm can be found in~\cref{section: expressivity}.
In~\cref{section: numerics} a numerical test is presented.
Summary and outlook are given in~\cref{section: outlook}.

\section{Finite element discretization and notation}
\label{section: FEM}
To generate data-efficient data, the finite element method (FEM) with an adaptive algorithm (coined AFEM) is used.
This section is concerned with the introduction of this AFEM, which is steered by an a posteriori error estimator.
This forms the basis of the subsequently derived neural network architecture.
A more detailed introduction to finite elements can e.g. be found in \cite{braess2007finite}.

\subsection{Problem setting}
We assume a regular conforming triangulation $\mathcal{T}$ of the (smoothly bounded) domain $D$, e.g. as depicted in \Cref{fig:param-to-sol}. Let $V_h = \spa \{ \varphi_j\}_{j=1}^{\dim V_h} \subset H_0^1(D)$ be a finite-dimensional subspace spanned by conforming first-order (Lagrange) basis functions (a FE function space). 
Any function $v_h\in V_h$ has a representation
\begin{align*}
    v_h = \sum_{i=1}^{\dim V_h} \bfv_i \varphi_i,
\end{align*}
where coefficient vectors with respect to the basis of $V_h$ are written in bold face. Throughout this paper it is assumed that the parameter field fulfills a uniform boundedness assumption $\mathfrak{c}\leq\kappa(x,\bfy)\leq\mathfrak{C}$ for all $x\in D,\bfy \in\Gamma$ and some constants $\mathfrak{c},\mathfrak{C}>0$ independent of $\bfy$.
We are concerned with finding a discrete solution $u_h\in V_h$ of the variational formulation for any $\bfy\in\Gamma$ such that for all test functions $w_h\in V_h$ it holds that
\begin{align}\label{eq: variational darcy}
    a_{\bfy,h}(u_h,w_h) \coloneqq \int_D \kappa_h(\cdot,\bfy)\langle \nabla u_h, \nabla w_h \rangle \dx= \int_D fw_h \dx \eqqcolon \mathfrak{f}(w_h).
\end{align}
This is equivalent to determining $\bfu\in \mathbb{R}^{\dim V_h}$ by solving the algebraic system
\begin{align}\label{eq: darcy linear equation system}
    A_{\bfy}\bfu = \bff
\end{align}
with the right-hand side $\bff \coloneqq (\mathfrak{f}(\varphi_j))_{j=1}^{\dim V_h}$ and discretized operator 
\begin{align}\label{eq: A operator}
    A_\bfy\coloneqq (a_{\bfy, h}(\varphi_i, \varphi_j))_{i,j=1}^{\dim V_h}.
\end{align}

We consider the following norms for $T\subset \mathbb{R}^d$ and $u:\mathbb{R}^d \to \mathbb{R}$
\begin{align*}
    \norm{u}_{L^2(T)}^2 \coloneqq \int_T u^2 \dx,\qquad
    \norm{u}_{H^1(T)}^2 \coloneqq \int_T u^2 \dx + \int_T \scpr{\nabla u, \nabla u} \dx,\qquad
    \anorm{u}^2 \coloneqq a_{\bfy,h}(u,u).
\end{align*}
For $\bfu\in \mathbb{R}^d$, we define the discrete norm
\begin{align*}
    \norm{\bfu}_{A_\bfy}^2 \coloneqq \bfu^\intercal A_\bfy \bfu.
\end{align*}
Note that 
\begin{align*}
    \norm{\bfu}_{A_\bfy}^2 = \sum_{i,j=1}^{\dim V_h} \bfu_i \bfu_j \ay(\varphi_i, \varphi_j) = \ay\left(\sum_{i}^{\dim V_h} \bfu_i \varphi_i,\sum_{j=1}^{\dim V_h} \bfu_j \varphi_j\right) = \ay(u_h, u_h) = \anorm{u_h}.
\end{align*}
Furthermore, we make use of the essential supremum norm $L^\infty$ and the discrete supremum norm $\ell^\infty$.

\subsection{Error estimation}\label{section: estimator}
We recollect the common residual based a posteriori error estimator for the Galerkin solution $u_h$ of the (parametric) Darcy problem in $V_h$ solving \eqref{eq: variational darcy}, cf.~\cite{verfurth,carstensen2012review} and~\cite{EigelGittelson2014asgfem} for the parametric setting. 
\begin{definition}[Jump \& error estimator]\label{def: jump and estimator}
    The \emph{jump} along the edge $\gamma$ between the triangles $T^1,T^2\in \mathcal{T}$ with $\nabla u^{(1)}_h$ and $\nabla u^{(2)}_h$ the gradients on the triangles, respectively, is defined by
    \begin{align}\label{def: jump}
        \jump{ \kappa_h(\cdot, \bfy)\nabla u_h} \coloneqq \kappa_h(\cdot,\bfy) \bra{ \scpr{ \nabla u_h^{(1)}, n_\gamma^{(1)} } + \scpr{ \nabla u_h^{(2)}, n_\gamma^{(2)} }},
    \end{align}
    where $n_\gamma^{(1)}, n_\gamma^{(2)}$ are the normal vectors of $\gamma$ pointing out of the triangles $T^1,T^2$, respectively.
    We define the local error contribution on each triangle $T$ by
    \begin{align}\label{eq: estimator}
        \eta_T^2 \coloneqq h_T^2\norm{f+\nabla\cdot(\kappa_h(\cdot, \bfy)\nabla u_h)}_{L_2(T)}^2 + h_T \norm{\jump{\kappa_h(\cdot, \bfy)\nabla u_h}}_{L_2(\partial T)}^2.
    \end{align}
\end{definition}
We henceforth assume that the data error $\norm{\kappa-\kappa_h}$ is negligible in the used norms.
Then, the estimator is reliable and efficient, i.e. there exist constants $c_\bfy,C$ such that
\begin{align*}
    \anorm{u - u_h}^2 &\leq C\sum_{T\in\T} \eta_T^2 \qquad\quad \text{ and}\\
    \eta_T &\leq c_\bfy \anorm{u - u_h} \text{ \hspace{1ex} for any } T\in\T.
\end{align*}
For the sake of a self-contained presentation, the derivation for the upper bound is recalled in \cref{section:error estimator derivation} while a full analysis of this and other error estimators is carried out in standard references such as~\cite{verfurth,braess2007finite}.

\subsection{Marking}\label{section: marking}
For a complete adaptive finite element scheme as in \eqref{eq: AFEM overview steps} and as discussed in the next subsection, different marking strategies can be considered. A popular marking for which a fixed error convergence of the $\mathrm{AFEM}$ over the degrees of freedom can be shown is the D\"orfler marking strategy \cite{dorfler, adaptivenochetto}.

\begin{definition}[D\"orfler marking]
    Let $\theta\in(0,1)$. Define $\mathcal{M}$ such that 
    \begin{align*}
        \sum_{T\in \mathcal{M}} \eta_T^2 \geq \theta \sum_{T\in\mathcal{T}} \eta_T^2.
    \end{align*}
\end{definition}

Alternatively, a maximum strategy can be considered \cite{Diening_2015}.
When performing a marking decision for each element, access to the estimator for all other elements has to be available.
Since the examined CNN architecture acts only locally on neighbouring elements, these marking strategies hence cannot be implemented and we resort to a threshold marking.

\begin{definition}[Threshold marking]\label{definition: threshold marking}
For $k\in[L]$ let $\delta_h>0$ be thresholds depending on the size of the triangles $h$, e.g. the maximal side length. Mark all elements $T\in\T$ with size $h$ for refinement if $\eta_{T}^2>\delta_h$.
\end{definition}

\subsection{Mesh refinement}\label{section: refinement}
The next step of the $\mathrm{AFEM}$ consists of refining the current mesh in marked areas. 
In this work, in the $L$th step of the AFEM the current space consists of the sum of subspaces of FE spaces on uniformly refined meshes
with nodes $\N_k$ for $k\in[L]$ and corresponding basis functions $\{\varphi^{k}_i\}_{i\in\N_k}$.
To refine a mesh element, all basis functions on a uniformly refined mesh (one level finer than the marked element) with overlapping support to the marked elements are included in $V_h$.
Let $\M = \bigcup_{k=1}^L \M_k$ be the decomposition of the marked elements into sets of elements in the same uniformly refined mesh.
Then the local mesh refinement is given by 
\begin{align*}
    V_h = V_h + \sum_{k=1}^L\spa \{ \varphi^{k+1}_i: \exists T\in \mathcal{M}_k \text{ in level } k \text{ with } \supp \varphi_i^{k+1} \cap T \neq \emptyset \}.
\end{align*}

\subsection[Adaptive finite element method]{Adaptive finite element method ($\mathrm{AFEM}$)}\label{section: AFEM}
Adaptive finite element methods are applied to find quasi optimal representations of PDE solutions by resolving local properties. Classical introductions can e.g. be found in \cite{braess2007finite,verfurth}.
A version of $\mathrm{AFEM}$ used in the present work is depicted in~\cref{alg: AFEM} and visualized in \Cref{fig: AFEM}.
A multigrid solver introduced in~\Cref{section: solver} and the estimator in~\eqref{eq: estimator} are employed to approximate the solution of the Darcy problem~\eqref{eq: darcy linear equation system} adaptively. 

Starting with an initial FE function space $V$ and an initial approximation $\bfu=0$, the following steps are executed iteratively in the solver.
The current solution corresponding to $\bfu$ is interpolated onto the current space $V$ and $\bfu$ is set to the coefficients of the interpolated solutions.
Then, the correction $\bfv$ to the best approximation in the current space is calculated by solving the system of linear equations
\begin{align*}
    A_\bfy \bfv = \bff - A_\bfy \bfu.
\end{align*}
The solver is described in~\cref{section: solver}.
The current solution is updated through $\bfu = \bfu + \bfv$. Local errors are estimated based on the a posteriori error estimator discussed in~\eqref{eq: estimator}. Large errors are marked and the the space $V$ is refined by including all basis elements of the next uniformly refined mesh in the current basis, which have an overlapping support with the marked regions.

To illustrate the practical performance, an example $\mathrm{AFEM}$ error convergence for the benchmark problem described in~\cref{section: numerics} can be compared to solutions on uniformly refined meshes in \Cref{fig: afem decay}.
In addition to the relative errors in the $H^1$ and the $L^2$ norms, the error estimator is plotted.
It can be observed, that the error estimator has the same decay as the true error in the $H^1$ norm as expected.

\begin{algorithm} 
\caption{Adaptive finite element method $\mathrm{AFEM}(\kappa, f, V, K)$}
\label{alg: AFEM}
Set $\bfu = 0$.\\
\For{$K$ iterations}{
Interpolate the current solution in $V$ and set $\bfu$ to its coefficients.\\
Find $\bfv$ such that $A_\bfy \bfv = \bff - A_\bfy \bfu$. (\cref{section: solver})\\
Update $\bfu = \bfu+\bfv$.\\
Estimate the error $\eta^2$. (\cref{section: estimator})\\
Set $L$ to the number of levels of $V$.\\
Mark elements  $\mathcal{M}_k$  on each level $k\in [L]$. (\cref{section: marking})\\
Refine the space $V = V + \sum_{k=1}^L\spa \{ \varphi^{k+1}_i: \exists T\in \mathcal{M}_k \text{ in level } k \text{ with } \supp \varphi_i^{k+1} \cap T \neq \emptyset \}$. (\cref{section: refinement})\\
}
\end{algorithm}

\begin{figure}
    \centering
    \includegraphics[width=\linewidth]{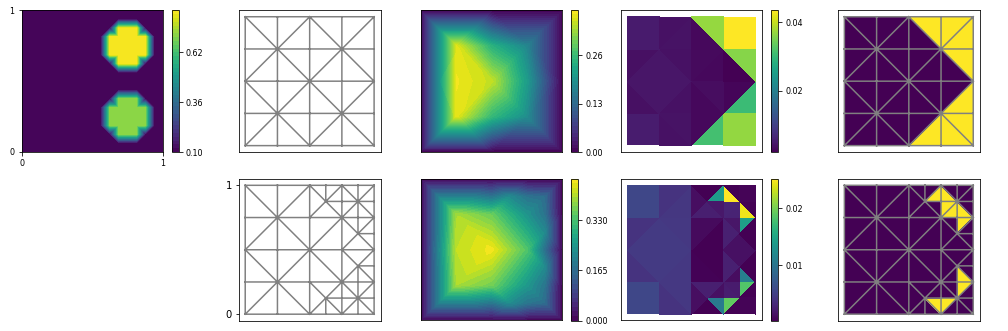}
    \caption{Two iterations of the adaptive finite element method on a unit square are depicted, where the first image on the left is a visualization of a possible parameter field $\kappa(\cdot,\bfy)$. In the rest of the first row, the first mesh, solution, local error estimator and marker are depicted. The second row shows these steps for a locally refined mesh.}
    \label{fig: AFEM}
\end{figure}

\begin{figure}
    \centering
    \includegraphics[width=0.49\linewidth]{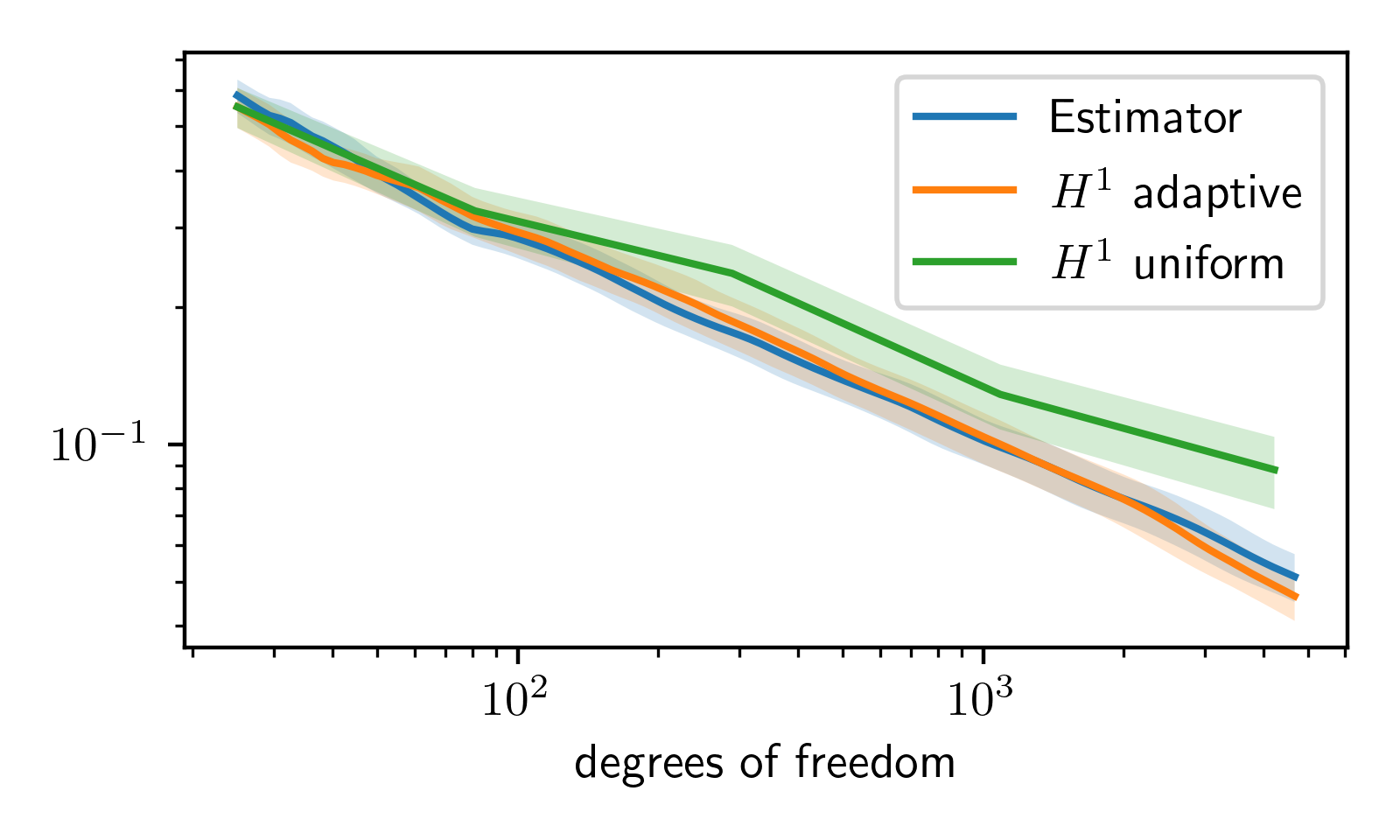} \includegraphics[width=0.5\linewidth]{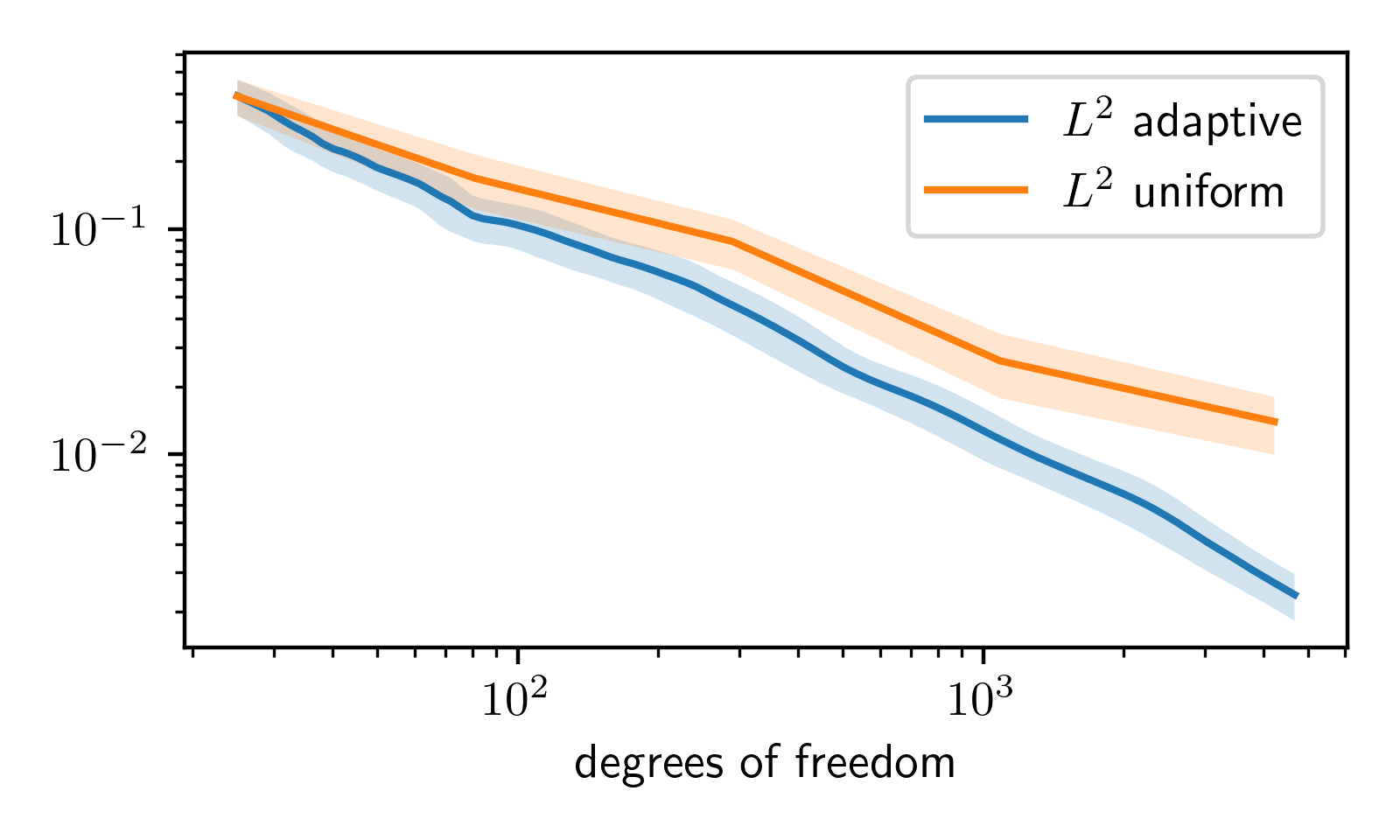}
    \caption{The two plots show the advantage of the $\mathrm{AFEM}$ in terms of degrees of freedom (FE coefficients) compared to solutions on uniformly refined meshes. Here, the mean and variance of the relative $H^1$ (left) and $L^2$ (right) errors of $100$ samples of the problem described in \cref{section: numerics} are plotted for the D\"orfler marking with $\theta=0.1$.}
    \label{fig: afem decay}
\end{figure}

\section{Solving on multiple grids}
\label{section: solver}

In this section, we derive a multigrid algorithm  
on the sum FE subspaces for uniformly refined grids to compute the corrections $\bfv$ in each step of \Cref{alg: AFEM}. It is closely related to classical FEM multigrid solvers, see e.g.~\cite{Hackbusch,hackbusch2013multi,braess2007finite}. Similarly, the convergence analysis is based on the more general framework of \emph{successive subspace correction} (SSC) algorithms, see~\cite{doi:10.1137/1034116, chen}.
It is exactly this algorithm that our CNN multilevel architecture is able to mimic. 
As shown in~\cite{cosi}, an accurate and efficient NN representation of a multigrid solver exists for regular and uniform grids.
In this work, the previous result is extended to locally refined grids and consequently to local multigrid corrections.
While this should lead to a significant complexity improvement of the architecture and the training process, several technical difficulties are inevitably introduced by the locality of the subspace corrections.

We start our considerations with a number of levels $L\in\mathbb{N}$, which corresponds to the current maximal refinement in the step of the $\mathrm{AFEM}$, where the solver is employed. Furthermore, a sequence of uniformly refined unit square grids with the set of nodes $(\mathcal N_\ell)_{\ell=1}^L$ indexed by\footnote{we use the convention $[n]:=\{1,\ldots,n\}$} $i\in[n_\ell] \times [n_\ell] \eqqcolon \I_U^\ell$ and set of triangles $(\T_\ell)_{\ell=1}^L$ is considered.
The corresponding spaces spanned by the piecewise linear nodal hat functions $\varphi^\ell_i: \mathbb{R}^2\to \mathbb{R}$ at nodes $i\in \I_U^\ell$ are denoted by $U^\ell~\coloneqq~\spa~\{\varphi^\ell_i:~i\in \I_U^\ell \}$.

Since we intend to work on locally refined grids, on each level $\ell=1,\dots,L$ only a subset of the index set $\I_V^\ell \subset \I_U^\ell$ and the corresponding triangles $\TVl{\ell}$ are considered.
These indices correspond to the nodal basis functions used in the local mesh refinement in \cref{section: refinement} on each level.
The discrete problem is then formulated with respect to $V_h \coloneqq \sum_{\ell=1}^L V^\ell$ for $V^\ell \coloneqq \spa \{ \varphi_i^\ell: i\in \I_V^\ell \}$ with level $\ell\in [L]$ and $v_h\in V_h$.
It can be represented by
\begin{align}\label{eq: v decomposition}
    v_h = \sum_{\ell=1}^L v^\ell = \sum_{\ell=1}^L \sum_{i\in \I_V^\ell} \bfv^\ell_i \varphi_i^\ell
\end{align}
with coefficients $\bfv^\ell
\in\mathbb{R}^{\I_V^\ell}$ for $\ell\in[L]$ as visualized in \Cref{fig:continuous decomposition}.
We set the closure of $\I_V^\ell$ to $\Iab{\ell} \coloneqq \I_V^k \cup \{i\in\I_U^\ell: \supp \varphi^\ell\cap \supp V^\ell \neq \emptyset \}$ to include all indices in $\I_U^\ell$, for which the corresponding functions have overlapping support with $V^\ell$.
Additionally, set the closure of $V^\ell$ to $\Vab{\ell} \coloneqq \spa \{\varphi_i^\ell : \supp \varphi_i^\ell \cap \supp V^\ell \neq \emptyset \}$.
The closure is visualized in \Cref{fig: closure}. This set is of importance in the SSC algorithm with CNN, when projecting information of all subsets to one refinement level. It is a formal technicality that is due to the introduction of local contributions to the solution that were not present in~\cite{cosi}.

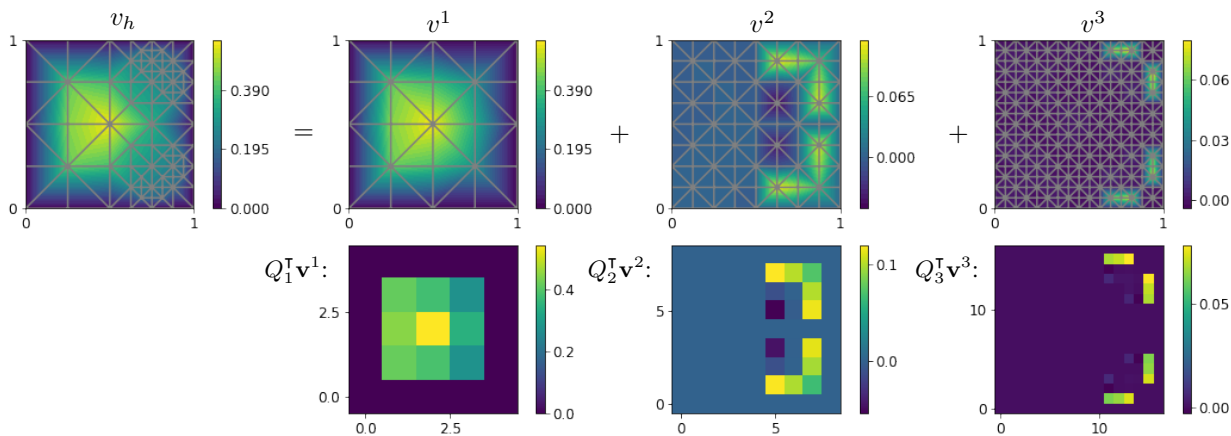
\begin{figure}
    \centering
    \begin{tikzpicture}
    \node at (0,0) {\includegraphics[width=\textwidth]{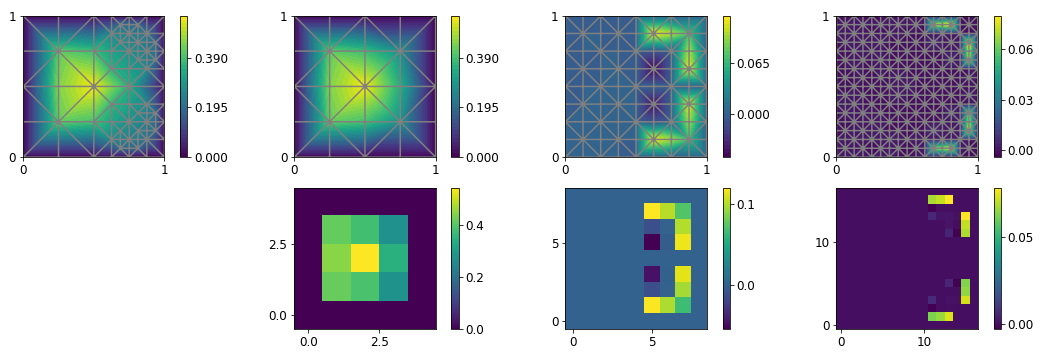}};
    \node at (-6.6,2.8) {$v_h$};
    \node at (-2.4,2.8) {$v^1$};
    \node at (1.9,2.8) {$v^2$};
    \node at (6.3,2.8) {$v^3$};
    \node at (-4.2,1.3) {$=$};
    \node at (0,1.3) {$+$};
    \node at (4.5,1.3) {$+$};
    \node at (-4.3,-.5) {\small $Q_1^\intercal\bfv^1$:};
    \node at (0,-.5) {\small $Q_2^\intercal\bfv^2$:};
    \node at (4.35,-.5) {\small $Q_3^\intercal\bfv^3$:};
\end{tikzpicture}
    \caption{Depicted is the decomposition of a continuous function $v\in V_h$ into coarse grid parts and fine grid corrections on uniformly refined grids. Each function on a uniformly refined grid can be represented by an image, where one pixel corresponds to the value of one node. For local corrections the images are sparse.}
    \label{fig:continuous decomposition}
\end{figure}

\begin{figure}
    \centering
    \includegraphics[width=0.6\linewidth]{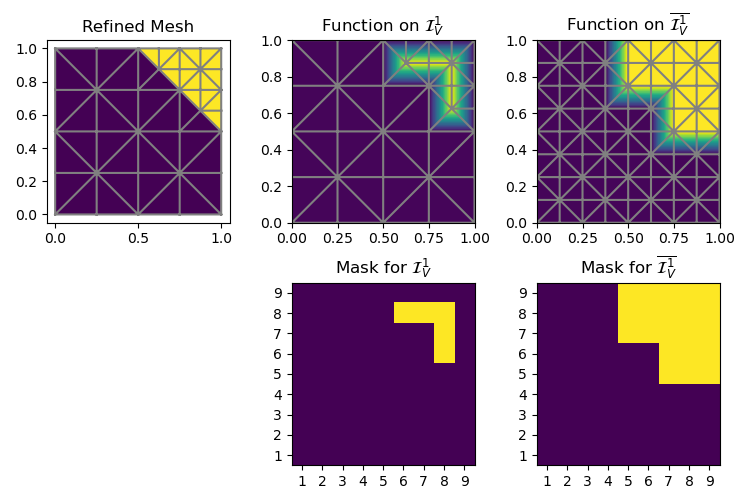}
    \caption{Refining a coarse mesh (compare first mesh in \Cref{fig: AFEM}) in the marked corner leads to the mesh depicted in the first row on the left-hand side. New degrees of freedom with indices in $\I_V^1$ stemming from this refinement (new nodes without the boundary nodes to incorporate the Dirichlet boundary condition) are depicted in the second image in the first row by a function, which is $1$ on $\I_V^1$ and $0$ otherwise. To visualize $\Iab{1}$, a function, which is $1$ on indices in $\Iab{1}$ and $0$ otherwise, is plotted in the last image in the first row. The corresponding masks on $\I_U^1$ are plotted in the second row.}
    \label{fig: closure}
\end{figure}

\subsection{Levelwise discretization}
Since CNNs can only act on one discretization level (corresponding to one image size) at a time, we derive a SSC acting on the different levels separately.
For this, define the $\ell^2$--projection, restricting an element in the whole space to one subspace by
\begin{align*}
    Q_k:\mathbb{R}^{\cup_{\ell\in[L]}        \I_V^\ell}
    \to \mathbb{R}^{\I_V^k} \quad\text{ with }\quad \bfv = (\bfv_{\ell,j})_{\ell\in[L],j\in \I_V^\ell} \mapsto (\bfv_{k,i})_{i\in\I_V^k} \eqqcolon \bfv^k.
\end{align*}
The transpose $Q_k^\intercal$ then trivially embeds an element of the subspace in the larger space. The above decomposition is visualized in \Cref{fig:continuous decomposition}, depicting $Q_k^\intercal \bfv^k$ for each level $k=1,\dots,L$.
Furthermore, we set $\bfv^{<k} \in \mathbb{R}^{\sum_{\ell=1}^L \I_U^\ell}$ to be the contribution of $\bfv$ corresponding to levels smaller than $k$ and $\bfv^{>k}\in\mathbb{R}^{\sum_{\ell=1}^L \I_U^\ell}$ to correspond to levels larger than $k$.
Formally, this is defined by
\begin{align*}
    \bfv^{<k} \coloneqq \sum_{\ell=1}^{k-1} Q_k^\intercal \bfv^k \quad \text{ and } \quad \bfv^{>k} \coloneqq \sum_{\ell=k+1}^{L} Q_k^\intercal \bfv^k
\end{align*}
with $\bfv^{<1} = \bfv^{>L} = 0 \in\mathbb{R}^{\cup_{\ell\in[L]}\I_V^\ell}$ such that
\begin{equation}\label{eq: decompose v}
    \bfv = \bfv^{<k} + Q_k^\intercal\bfv^k + \bfv^{>k}.    
\end{equation}
An SSC solving $A_\bfy \bfu = \bff$ for some $\bff\in\mathbb{R}^{\cup_{\ell\in[L]}\I_V^\ell}$ consists of \emph{smoothing updates}  
on each level, carried out in a successive manner.
One such update on level $k\in[L]$ has the form
\begin{equation*}
    \bfu^k \gets ~\bfu^k +~ \omega^k_\bfy(\bff^k - Q_k A_\bfy \bfu),
\end{equation*}
where for each index $i \in \I_V^k$ on level $k$ and each index on any level $(k_2,j) \in \cup_{\ell\in[L]} \{\ell\}\times \I_V^\ell$ the operator is set to $(Q_k A_\bfy)_{i, (k_2,j)} \coloneqq a_{\bfy,h}(\varphi^{k}_i, \varphi^{k_2}_j)$. 
Since the operator needs information of $\bfu$ on all levels, the operation has to be decomposed into contributions for each level individually.
Therefore, in order to calculate $Q_k A_\bfy \bfu$, 
we consider the different contributions of the decomposition separately using the following prolongation and weighted restriction operations, which transfer discrete functions from one level to a consecutive level. 
\begin{definition}[Prolongation \& weighted restriction]\label{def: prolongation}
    For $L\in\mathbb{N}$ and $k=1,\dots, L-1$ define 
    the \emph{prolongation} $P_k:\mathbb{R}^{\Iab{k}} \to \mathbb{R}^{\Iab{k+1}}$ as the nodal interpolation of $\Vab{k}$ onto $\Vab{k+1}$.
    We call $P_k^\intercal$ the \emph{weighted restriction} with $\varphi_i^k(x) = \sum_{j\in \Iab{k+1}} (P_k^\intercal)_{i,j}\varphi_j^{k+1}(x)$ for $x\in \supp V_{k+1}$. 
\end{definition}

These operators can be used to connect different levels and calculate the application of the operator $A_\bfy$ to the whole vector $\bfu$ (with contributions from all levels) for a smoothing step on each level.

\begin{theorem}[Levelwise calculation of $Q_k A_\bfy \bfu$]\label{thm: levelwise Au}
    Let $\Aab{k}$ be defined as in \eqref{eq: A operator} for indices in $\I_V^k \times \Iab{k}$ and functions $\varphi_i^k\in\I_U^k$. Furthermore, set $\uab{i}{k}\in\mathbb{R}^{\Iab{k}}$ equal to $\bfu_i^k$ for $i\in\I_V^k$ and zero otherwise.
    To calculate $Q_k A_\bfy \bfu^{<k}$ and $Q_k A_\bfy \bfu^{>k}$, for $k\in[L]$ we define the auxiliary vectors
    \begin{alignat}{2}\label{eq: u tilde k}
        \tilde \bfu ^1 &\coloneqq 0, \quad &\tilde{\bfu}^k &\coloneqq P_{k-1} \left(\tilde \bfu ^{k-1} + \uab{\empty}{k-1}\right)
        \qquad \text{ and }\\
        \bar \bfu ^L &\coloneqq 0, \quad &\bar{\bfu}^k &\coloneqq P_k^\intercal \left(\bar \bfu^{k+1} + \Aab{k+1}^\intercal \bfu^{k+1}\right),
        \label{eq: u bar k}
    \end{alignat}
    where $\tilde{\bfu}^k$ denotes the interpolation of $\bfu^{<k}$ into the current space and $\bar \bfu ^k$ denotes the projection of $\bfu^{>k}$ onto the current space.
    This makes it possible to represent the multiplication with $A_{\bfy}$ on each level only using levelwise calculations, 
    prolongations and weighted restrictions by
    \begin{align*}
        Q_kA_\bfy \bfu = \Aab{k}\left(\uab{\empty}k + \tilde \bfu^k\right) + \bar \bfu^k|_{\I_V^k}.
    \end{align*} 
\end{theorem}
\begin{proof} Since $A_\bfy$ is a linear operator, we considering the multiplication with the different parts of $\bfu$ separately. For $j\in\I_V^k$ 
it holds that
    \begin{align*}
        (Q_k A_\bfy Q_k^\intercal \bfu^k)_{j} = \sum_{i \in \mathcal{I}_V^k} \bfu^k_i \int \kappa_h \scpr{ \nabla \varphi_i^k, \nabla \varphi_j^k} \dx = \sum_{i \in \Iab{k}} \uab{i}{k} \int \kappa_h \scpr{ \nabla \varphi_i^k, \nabla \varphi_j^k} \dx = \left(\Aab{k} \uab{i}{k}\right)_j.
    \end{align*}
    \cref{thm: decomp_smoothing_coarse} and \cref{thm: decomp_smooting_fine} yield
    \begin{align*}
        Q_k A_\bfy \bfu^{<k} &= \Aab{k} \tilde \bfu^k \quad \text{ and }\\
        Q_k A_\bfy \bfu^{>k} &= \bar \bfu^k|_{\I_V^k},
    \end{align*}
    respectively. The claim follows with~\eqref{eq: decompose v}.
\end{proof}

With this, we can define an SSC using only levelwise actions and with restrictions and prolongations between two consecutive levels as depicted in the \emph{Levelwise Local Multigrid Algorithm} (LLMG) in \cref{alg: LLMG}.
It consists of residual corrections and smoothing steps in each subspace $\mathbb{R}^{\I_V^k}$ separately, starting with the finest level $L$, successively including coarser levels, and subsequently updating finer levels until each level has been updated twice.
After each update, the auxiliary variables $\bar\bfu, \tilde\bfu$ are updated. Since the definition in \eqref{eq: u bar k} of $\bar{\bfu}^{k}$ contains information only of $\bfu^\ell$ for $\ell > k$, i.e. information of finer levels, and $\tilde{\bfu}^k$ in \eqref{eq: u tilde k} only depends on $\bfu^\ell$ for $\ell <k$, i.e. information of coarser levels, the update of one $\bfu^k$ leads to a change of $\bar\bfu^\ell$ for $\ell < k$ and a change of $\tilde\bfu^\ell$ for $\ell > k$.
Therefore, when smoothing on a coarser level in the subsequent step, $\bar{\bfu}^{k-1}$ needs to be updated. When smoothing on a finer level in the subsequent step, the update has to be done for $\tilde{\bfu}^{k+1}$.

\begin{algorithm}
    \caption{Levelwise Local Multigrid Algorithm $\mathrm{LLMG}(\bfu, \bff, \bfy)$}
    \label{alg: LLMG}
    Calculate $\bar\bfu, \tilde\bfu$ as in \eqref{eq: u bar k}, \eqref{eq: u tilde k} \Comment{calculate auxiliary vectors}\\
    \For{$k=L,\dots 1$}{
    $\bfu^k \gets \bfu^k + \omega^k_\bfy (\bff^k - [\Aab{k} (\bfu^k + \tilde{\bfu}^k) + \bar{\bfu}^k|_{\I_V^k}])$ \Comment{smoothing on one level} \label{alg: LLMG: first update}\\
    \If{$k>1$}{
    $\bar{\bfu}^{k-1} \gets P_{k-1}^\intercal (\bar \bfu^{k} + \Aab{k}^\intercal \bfu^{k})$ \Comment{update auxiliary vector $\bar{\bfu}$ fine to coarse} \label{alg: LLMG: update u bar} 
    }
    }
    \For{$k=1,\dots,L$}{
    $\bfu^k \gets \bfu^k + \omega^k_\bfy (\bff^k - [\Aab{k} (\bfu^k + \tilde{\bfu}^k) + \bar{\bfu}^k|_{\I_V^k}])$ \Comment{smoothing on one level} \label{alg: LLMG: second update}\\
    \If{$k<L$}{
    $\tilde{\bfu}^{k+1} \gets P_{k} \left(\tilde \bfu ^{k} + \uab{\empty}{k}\right)$ \Comment{update auxiliary vector $\tilde{\bfu}$ coarse to fine} \label{alg: LLMG: update u tilde} 
    }
    }
\end{algorithm}

With \cref{thm: levelwise Au}, the LLMG is a standard SSC algorithm (see \cref{alg: ssc}) and convergence can be derived from known results.

\begin{theorem}[Convergence of the LLMG]\label{theorem: LLMGm}
    Assume that there exist constants $\mathfrak{c},\mathfrak{C}>0$ such that $\mathfrak{c}\leq \min_{\bfy\in\Gamma}\lambda_{\min}(A_\bfy)$ and $\mathfrak{C} \geq \max_{\bfy\in\Gamma} \lambda_{\max}(A_\bfy)$.
    Let $\bfu$ be the solution of $A_\bfy \bfu = \bff$ and $0<\omega^k_\bfy = \omega \leq C^{-1}$
    . There exists a constant $c>0$ such that for $c_L \coloneqq \frac{cL}{1+cL}$, $\varepsilon>0$ and $m\in\mathbb{N}$ with $m \geq \log(\varepsilon^{-1})/\log(c_L^{-1})$ it holds true that
    \begin{align*}
        \norm{\bfu - \mathrm{LLMG}^m(0,\bff,\bfy)}_{A_\bfy} \leq \varepsilon \norm{\bfu}_{A_\bfy},
    \end{align*}
    where $\mathrm{LLMG}^m$ denotes the application of the algorithm $m$ times.
\end{theorem}
\begin{proof}
    For a fixed $\bfy\in\Gamma$  the LLMG is equivalent to the local multigrid algorithm (LMG \cref{alg: LMG}) with~\cref{thm: levelwise Au} and therefore has the same convergence rate. We use the XZ-identity in \cref{lm: XZ-identity} shown in~\cite[Theorem 4]{chen} and a result similar to the Richardson smoothing contraction shown in~\cite[Lemma 4.3]{Hackbusch} in \cref{lm: subspace contraction} to deduce the convergence of the LMG in the appendix (see \cref{thm: X-Z identity} and \cref{section: LMG}), showing that there exists a constant $C>0$ such that for $c_L = \frac{CL}{1+CL}$ it holds
    \begin{align*}
        \norm{\bfu - \mathrm{LLMG}^m(\bfu^0, \bff,\bfy)}_{A_\bfy} \leq c_L^m \norm{\bfu - \bfu^0}_{A_\bfy}.
    \end{align*}
    Therefore, initializing with $\bfu^0=0$ and choosing $m\geq \log(\varepsilon^{-1})/\log(c_L^{-1})$ confirms the claim.
\end{proof}

\begin{remark}
    Note that $\kappa(\cdot,\bfy)>c$ for some $c>0$ and all $\bfy\in\Gamma$ implies with the Poincaré inequality that $\lambda_{\min}(A_\bfy)>C$ for some $C>0$ and all $\bfy\in\Gamma$. Furthermore, $\kappa(\cdot,\bfy)<c$ for some $c>0$ and all $\bfy\in\Gamma$ implies that $\lambda_{\max}(A_\bfy)<C$ for some $C>0$ and all $\bfy\in\Gamma$.
\end{remark}

\section{Convolutional neural networks (CNN) for finite element discretizations}
\label{section: CNN}

CNNs are a specific neural network architecture tailored to tasks involving image data such as image classification and segmentation. Inspired by~\cite{Fukushima1980NeocognitronAS}, they were first implemented with a backpropagation algorithm in~\cite{NIPS1989_53c3bce6} for image recognition.
Applying the action of a CNN to an image involves the application of local kernels to a hierarchy of scaled representations of the input.
This locality makes the architecture particularly suitable with partial differential equations, where local properties and interactions have to be resolved to obtain highly accurate representations.
To incorporate interactions on a larger scale with respect to the image domain, compression and decompression of the input images can be implemented with CNNs through strided and transpose strided convolutions, leading the popular CNN architecture U-Nets~\cite{Unet1}.
This architecture is heavily exploited in this work.

\subsection{Data decomposition}
In the analysis of the implemented CNN architecture, images with different resolutions are used as the representation of the solutions of the parametric PDE.
This is possible since FEM discretizations of functions determined by coefficient vectors are used on different grid levels, similar to the decomposition in~\cite{cosi}.
Additionally, the discontinuous error estimator\footnote{Note that the solution is assumed as a conforming P1 function and the estimator is a DG0 function, i.e. defined by a scalar value per mesh element.} defined on the triangles of the considered meshes are represented with images of different scales.

\emph{Continuous functions} Any $v_h\in V_h$ can be decomposed into its components on $v^\ell\in V^\ell$ by $v_h = \sum_{\ell=1}^L v^\ell$ as described in \Cref{section: solver}. The functions $v^\ell$ on each level can be represented by \emph{coefficient images} on the whole uniformly refined grids.
If $\I_V^k$ is a small subset of $\I_U^k$, i.e. in case of very local refinements in the $\mathrm{AFEM}$, these images are sparse.
The complete decomposition of $v_h\in V_h$ as depicted in \Cref{fig:continuous decomposition} reads
\begin{align*}
    v_h = \sum_{\ell=1}^L v^\ell = \sum_{\ell=1}^L \sum_{i\in \I_V^\ell} \bfv^\ell_i \varphi_i^\ell.
\end{align*}
\begin{definition}[Coefficient images]
    For $\bfw\in\mathbb{R}^{\I_V^k}$ the coefficient image $\wim \in\mathbb{R}^{\I_U^k}$ is defined for $i\in\I_U^k$ by 
    \begin{align*}
        (\wim)_i \coloneqq \begin{cases}
            \bfw_i, &\text{ if } i\in\I_V^k\\
            0, &\text{otherwise}
        \end{cases},
    \end{align*}
    where the indices are two dimensional image indices $i = (i_1,i_2)$.
\end{definition}

In the same manner as the continuous functions, the \emph{piecewise constant functions} can be decomposed into corrections on uniformly refined meshes as depicted in \Cref{fig:discontinuous decomposition}, namely
\begin{align}\label{eq: DG0}
    \eta = \sum_{\ell=1}^L \eta^\ell = \sum_{\ell=1}^L \sum_{i=1}^n\sum_{j=1}^{m^\ell} \bfeta^\ell_{(i,j)}\tilde{\varphi}_{i,j}^\ell,
\end{align}
where $n$ is the number of images needed for the representation.
It depends on the structure of the mesh ($n=1$ in the continuous case due to nodal representation) and $m^\ell, \ell=1,\dots,L$ the number of pixels of each image.
Here, $\tilde{\varphi}^\ell_{(i,j)}$ denotes the characteristic function on the triangle $(i,j)\in [n]\times [m^\ell]$ on discretization level $\ell\in[L]$ for the indexation of the triangles according to~\eqref{eq: DG0}.
In \Cref{fig:discontinuous decomposition images} each (continuous) correction is represented with $n=8$ \emph{sparse images}.
Piecewise constant (discontinuous) functions on meshes as depicted in \Cref{fig: uniform vs local meshes} can be represented with $n=2$ images, e.g. by one image containing the values of the upper triangle in the upper right square of each node and the other image containing the lower triangle in the same square as illustrated in \Cref{fig: estimator triangles}.

\begin{figure}
    \centering
    \begin{tikzpicture}
        \node at (0,0) {\includegraphics[width=\linewidth]{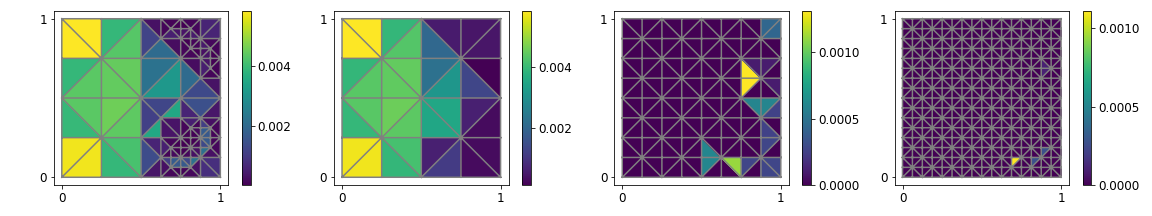}};
        \node at (-3.8,0.2) {$=$};
        \node at (0.,0.2) {$+$};
        \node at (4.3,0.2) {$+$};
        \node at (-6.2,1.7) {$\eta$};
        \node at (-2.2,1.7) {$c_1$};
        \node at (1.8,1.7) {$c_2$};
        \node at (5.8,1.7) {$c_3$};
    \end{tikzpicture}
    \caption{A piecewise constant discontinuous functions $\eta = \sum_{T\in\mathcal{T}} \eta_T\chi_T$ can be decomposed into a coarse grid piecewise constant function and a fine grid piecewise constant corrections.}
    \label{fig:discontinuous decomposition}
\end{figure}
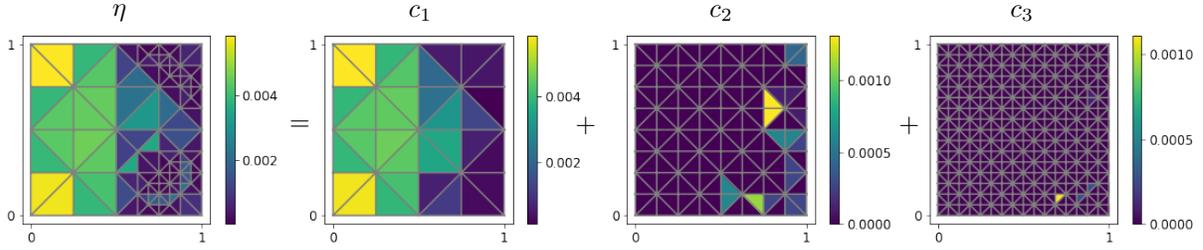
\begin{figure}
    \centering
    \begin{tikzpicture}
        \node at (-8,.2) {\includegraphics[width=0.19\linewidth]{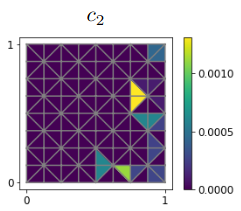}};
        \node at (0,0) {\includegraphics[width=0.75\textwidth]{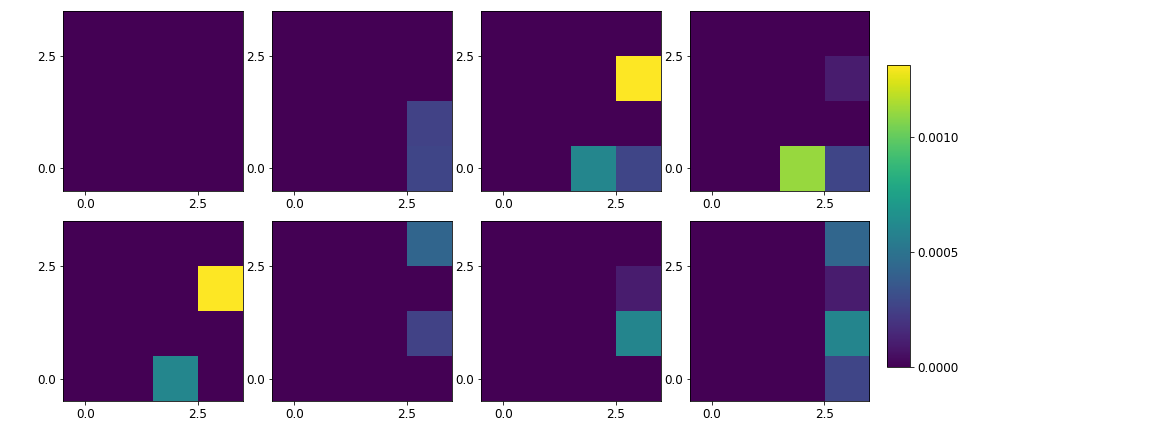}};
        \node at (-6,0) {:};
    \end{tikzpicture}
    \caption{Each piecewise constant function for the meshes depicted in \Cref{fig:discontinuous decomposition} can be represented with $8$ images. Every other node in each direction is surrounded by $8$ triangles. Each image corresponds to one of the triangles for every such node.}
    \label{fig:discontinuous decomposition images}
\end{figure}

\subsection{Submanifold sparse CNN}
Different types of convolutions are considered in this work. The \emph{vanilla} convolution $\ast$ sweeps a kernel over the input, the $2$-\emph{strided} convolution $\ast^{2\text{s}}$ applies the kernel on every other pixel of the input images, approximately halving the size of the input images.
Moreover, the $2$-\emph{transpose strided} convolution $\ast^{2\text{st}}$ sweeps the kernel over a dilated image with added zeros between every two pixels, doubling the input image size and the \emph{submanifold sparse} convolution $\spco$ as used in \cite{graham2017submanifold, 3DSemanticSegmentationWithSubmanifoldSparseConvNet} applies the kernel only to nonzero pixels of the input image and sets the remaining entries to zero.
The different convolutions are visualized in \Cref{fig: convolutions}, see also \cite{dumoulin2016guide, dumoulingit}.

Every step of the successive subspace correction algorithm~\Cref{alg: LLMG} and the error estimator~\eqref{eq: estimator} can be represented with a CNN by incorporating these different convolutions.
When additionally including a marking function as culling mask, the whole adaptive scheme~\Cref{alg: AFEM} can be approximated by the derived CNNs on sparse images.

\section{Expressivity results}\label{section: expressivity}

For the analysis in this work, the used activation function has to satisfy the following assumption.
\begin{assump}[Activation function]\label{ass: activation function}
    Let $\sigma\in L^\infty_{\text{loc}}$ such that there exists $x_0\in\mathbb{R}$, where $\sigma$ is three times continuously differentiable in a neighborhood and $\sigma''(x_0) \neq 0$. 
\end{assump}
These properties are fulfilled for a number of classical activation functions such as softplus, sigmoids and the exponential linear unit. 
In the following subsections, the individual parts of the $\mathrm{AFEM}$ (\cref{alg: AFEM}) are approximated individually.
The estimations are then collected for the overall convergence result.
To illustrate our constructions, the meshes depicted in \Cref{fig: uniform vs local meshes} are considered.

\subsection{NN approximation of the multigrid solver}

To approximate the solution on a fixed grid in each step of the $\mathrm{AFEM}$, the "Levelwise Local Multigrid Algorithm" $\mathrm{LLMG}$ (\Cref{alg: LLMG}) is employed.
Its main ingredient is the smoothing on each (locally refined) subspace.
For one smoothing step, the crucial part is the approximation of the action of the parametric operator $A_\bfy\bfu$ with a CNN.
The analysis is similar to~\cite[Theorem 6]{cosi}.
However, the local corrections impose several technical additions, for which some auxiliary vectors defined in the algorithm are introduced.
We are then able to show the following complexity bound.
\begin{theorem}\label{thm: LLMG_Approx}
    Assume that there exist constants $\mathfrak{c},\mathfrak{C}>0$ such that $\mathfrak{c}\leq \min_{\bfy\in\Gamma}\lambda_{\min}(A_\bfy)$ and $\mathfrak{C} \geq \max_{\bfy\in\Gamma} \lambda_{\max}(A_\bfy)$.
    There exists a positive constant $C>0$ such that for every $\varepsilon, M >0$ there exists a CNN $\Psi:\mathbb{R}^{2\times \I_U^L} \to \mathbb{R}^{\bigtimes_{k=1}^L \I_V^k}$ 
    such that
    \begin{enumerate}
        \item $\norm{\Psi({\bfk_\bfy}_{\text{img}},\bff_{\text{img}}) - \mathrm{LLMG}^m(0,\bff, \bfy))}_{A_\bfy} \leq \varepsilon \quad$ for all 
        $\kappa(\cdot,\bfy),f \in U^L$
        \item number of weights bounded by $M(\Psi) \leq CLm$.
    \end{enumerate}
\end{theorem}
The proof can be found in \cref{section: Solver approx}. Combining this result with \cref{theorem: LLMGm} leads to the following corollary, stating that the solution of the Darcy problem~\eqref{eq: darcy linear equation system} on a an adaptively refined mesh can be approximated arbitrarily well by a CNN with a prescribed bound for the number of parameters.

\begin{corollary}\label{cor: sol approx}
    Let $\bfu$ be the solution of $A_\bfy\bfu=\bff$. Choose $m \geq \log(\varepsilon^{-1})/\log(c_L^{-1})$ with chosen as in~\cref{theorem: LLMGm}. Then there exists a constant $C>0$ such that for any $\varepsilon>0$ there exists a CNN $\Psi:\mathbb{R}^{2\times\I_U^L} \to \mathbb{R}^{\bigtimes_{k=1}^L \I_V^k}$
    with the number of parameters bounded by $M(\Psi) \leq CL\log(\varepsilon^{-1})/\log(c_L^{-1})$ such that
    \begin{align*}
        \norm{\Psi({\bfk_\bfy}_\img, \bff_\img) - \bfu}_{A_\bfy} &\leq \norm{\Psi({\bfk_\bfy}_\img, \bff_\img) - \mathrm{LLMG}^{m}(0,\bff,\bfy)}_{A_\bfy} \\
        &\quad + \norm{\mathrm{LLMG}^{m}(0,\bff,\bfy) - \bfu}_{A_\bfy}\\
        &\leq \varepsilon (1 + \norm{\bfu}_{A_\bfy}).
    \end{align*}
\end{corollary}

\subsection{Estimator approximation}\label{sec: estimator approximation}
A central novelty of this work is the CNN representation of the a posteriori error estimator $\eta$ subject to $u_h$ as used in the $\mathrm{AFEM}$, see~\Cref{section: estimator}.
For the analysis of the approximation, the two parts of the estimator (the jump term and strong residual) are considered separately.
The analysis is carried out for a reference triangle in $\T_{V,k}$ from a triangulation $V^k$.

\begin{definition}[Strong residual \& jump images]\label{def: estimator images}
    Let $T^1_{k,i}$ and $T^2_{k,i}$ be the triangles in the top right quadrant of node $i\in \I_U^k$ as depicted in \Cref{fig: estimator triangles} and let $k\in [L]$.
    Define 
    $\stre^2_{k,T^q}, \jure^2_{k,T^q} \in \mathbb{R}^{\I_U^k}$ 
    for $q=1,2$ by
    \begin{align*}
        (\stre ^2_{k,T^q})_{i} &\coloneqq h_{T^q_{k,i}}^2 \norm{f+\nabla\cdot(\kappa_h(\cdot, \bfy)\nabla u_h)}_{L_2(T^q_{k,i})}^2 
        \quad \text{ as the strong residual image and }\\
        (\jure^2_{k,T^q})_{i} &\coloneqq h_{T^q_{k,i}} \norm{\jump{\kappa_h(\cdot, \bfy)\nabla u_h}}_{L_2(\partial T^q_{k,i})}^2 
        \quad \qquad\quad\text{ as the jump image}.
    \end{align*}
\end{definition}

\begin{figure}
    \centering
    \includegraphics[width=0.2\linewidth]{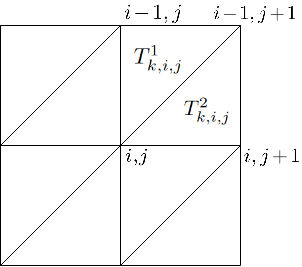}
    \caption{On each level $k\in [L]$ at each node $i\in \I_U^k$ we define the mesh elements $T^1_{k,i}$ and $T^2_{k,i}$ as the triangles in the upper right quadrant.}
    \label{fig: estimator triangles}
\end{figure}

Then the a posteriori error estimator in a triangle $T^q_{k,i}\in\T_{V,k}$ can be written as 
\begin{align*}
    \eta_{T^q_{k,i}}^2 = \left(\stre_{k,T^q}^2\right)_{i} + \left(\jure_{k,T^q}^2\right)_{i}.
\end{align*}

\begin{definition}[Uniform prolongation and weighted restriction]
    Extending the definition of the prolongation and weighted restriction from the subspaces $V_k$ to $U_k$, the \emph{uniform prolongation} is defined as $\Puk: \mathbb{R}^{\I_U^k}\to\mathbb{R}^{\I_U^{k+1}}$ such that $\varphi_i^k = \sum_{j\in\I_U^{k+1}} (\Puk)_{j,i} \varphi_j^{k+1}$ for all $i\in\I_U^k$. The transpose $\Puk^\intercal$ is called \emph{uniform weighted restriction}.
\end{definition}
Note that this is the prolongation as defined in~\cite[Defnition 2]{cosi}.

\begin{theorem}[Estimator approximation]\label{thm: estimator approx}
    Let $\eta_{k}^2, M^k\in\mathbb{R}^{2\times \I_U^k}$ 
    be defined for $k=1,\dots,L$, $q\in\{1,2\}$ and $i \in \mathcal{I}_U^k$ by $(\eta_{k}^2)[q]_{i} \coloneqq \eta^2_{T_{k,i}^q}$.
    Moreover, let $M^k[q]_i \coloneqq 1$, if $T_{k,i}^q \in \TVl{k}$ 
    and zero otherwise. 
    There exists a constant $C>0$ such that for every $\varepsilon, M>0$ there exists a CNN $\Psi: \mathbb{R}^{ \bigtimes_{\ell=1}^L 3\times \I_U^\ell}
    \to 
    \mathbb{R}^{\bigtimes_{\ell = 1}^L \I_U^\ell}$ 
    such that 
    \begin{enumerate}
        \item $\norm{M^\ell \odot \Psi(\bigtimes_{k=1}^L \uimk \times \kimk \times\fimk)[\ell] - \eta_{\ell}^2}_{\ell^\infty} \leq \varepsilon$ holds for all $\ell=1,\dots, L$ 
        and
        \item the number of parameters is bounded by $M(\Psi) \leq CL$.
    \end{enumerate}
\end{theorem}

\begin{proof}
    For the finest level $L$ and $q=1,2$ let the estimator images 
    $\stre^2_{L,T^q}, \jure^2_{L,T^q} \in \mathbb{R}^{\I_U^L}$
    be defined as in \cref{def: estimator images}.
    To represent the solution on the finest level let $\uim^{P_1}\coloneqq \uim^1$ and $\uim^{P_k} \coloneqq \Puk \uim^{P_{k-1}} + \uim^k$ 
    according to \cref{def: prolongation} for $k=1,\dots,L$.
    Then, with \cite[Remark 19]{cosi} $\uim^{P_k}$ can be calculated with a CNN and $\bfu^{P_L}$ contains the coefficients of the nodal interpolation of the function $u_h\in V$ defined by $\bigtimes_{k=1}^L \bfu^k$ in $U^L$.

    Now the estimator can be approximated on every level by approximating the residual and jump images in the fines level and combining them correctly.
    In \cref{thm: CNN for error estimator one level} we show that there exists a CNN architecture such that for every $\varepsilon, M >0$ and $q=1,2$ there exists a CNN $\Psi$ with
    \begin{align*}
        \norm{\Psi(\bfu^{P_L}, \bff^L, \bfk_\bfy^L)[q] - (\stre_{L,T^{q}}^2, \jure_{L,T^{q}}^2)}_\infty \leq \varepsilon.
    \end{align*}
    Observe that for triangles $T_k\in\T_k$ it holds that $h_{T_k} = h_0/2^k = 2 h_0/2^{k+1} = 2 h_{T_{k+1}}$. Furthermore, each triangle on level $k$ is equal to the union of four triangles on level $k+1$
    \begin{align*}
        T^q_{k,i} = \bigcup_{\tilde{q}\in\{1,2\}}\bigcup_{\substack{j\in\I_U^{k+1}\text{s.t.}\\ T^{\tilde{q}}_{k+1,j}\subset T^q_{k,i}}} T^{\tilde{q}}_{k+1,j}.
    \end{align*}
    This yields
    for $i\in\I_U^k$ and for triangles as in \Cref{fig: estimator triangles}
    \begin{align*}
        (\stre^2_{k,T^q})_i &= h_{T^q}^2\norm{f+\nabla\cdot(\kappa_h(\cdot, \bfy)\nabla u_h)}_{L_2(T^q_{k,i})}^2 = 2^2 
        \sum_{{\tilde{q}}\in\{1,2\}} \sum_{\substack{j\in\I_U^{k+1}\text{s.t.}\\ T^{\tilde{q}}_{k+1,j}\subset T^q_{k,i}}} h_{T_{k+1}}^2\norm{f+\nabla\cdot(\kappa_h(\cdot, \bfy)\nabla u_h)}_{L_2(T^{\tilde{q}}_{k+1,j})}^2\\
        &= 4 \sum_{{\tilde{q}}\in\{1,2\}}\sum_{\substack{j\in\I_U^{k+1}\text{s.t.}\\ T^{\tilde{q}}_{k+1,j}\subset T^q_{k,i}}} (\stre_{k+1,T^{\tilde{q}}}^2)_j.
    \end{align*}
    This can be implemented with one CNN layer with a sparse kernel and stride $2$ for each level.
    Since the jump term is zero on edges in the fine discretization, which have not been used to solve for $u_h$, jumps over edges of some triangle $T^{\tilde{q}}_{k+1,j}$ on level $k+1$ inside a triangle $T^q_{k,i}$ on level $k$ can be added up to yield the jumps only over edges on the coarser level. This yields
    \begin{align*}
        (\jure_{k,T^q}^2)_i &= h_{T_k} \norm{\jump{\kappa_h(\cdot, \bfy) \nabla u_h}}_{L^2(\partial T^q_{k,i})}^2 = 2 
        \sum_{{\tilde{q}}\in\{1,2\}}\sum_{\substack{j\in\I_U^{k+1}\text{s.t.}\\ T^{\tilde{q}}_{k+1,j}\subset T^q_{k,i}}} h_{T_{k+1}} \norm{\jump{\kappa_h(\cdot,\bfy) \nabla u_h}}^2_{L^2(\partial T^{\tilde{q}}_{k+1,j})}\\
        &= 2\sum_{{\tilde{q}}\in\{1,2\}}\sum_{\substack{j\in\I_U^{k+1}\text{s.t.}\\ T^{\tilde{q}}_{k+1,j}\subset T^q_{k,i}}} (\jure^2_{k+1,T^{\tilde{q}}})_j.
    \end{align*}
    This can also be realized by one CNN layer with a sparse kernel and a stride of $2$ for each level. Since adding the strong residual image and jump image yields the error estimator for triangles $T_{k,i}^q\in\T_{V,k}$, multiplying with a mask setting all other output entries to zero yields the claim.
\end{proof}

\subsection{AFEM approximation}

Combining the approximation of the multigrid solver and the error estimator with a marker based on the estimator leads to an approximation of the whole $\mathrm{AFEM}$ algorithm.
For the formulation of our main theorem, the following coefficient-to-function map is needed.

\begin{definition}
    Let $\fem: \mathbb{R}^{\I_U^1\times \dots \times \I_U^L} \to H_0^1$ be the function that maps the finite element coefficients to the corresponding function in $H_0^1$ by
    \begin{align*}
        \fem(\bfu) \coloneqq \sum_{k=1}^L\sum_{i\in\I_U^k} \bfu^k_i \varphi^k_i.
    \end{align*}
\end{definition}

Our main result then gives an upper bound on the number of parameters needed by the constructed networks architecture to approximate the solution of the Darcy problem as well as the result of the $\mathrm{AFEM}$ algorithm.
In summary, the derived bound depends linearly on the number of refinement levels as well as linearly on the number of steps of the $\mathrm{AFEM}$ and logarithmically on the inverse of the desired accuracy.
\begin{theorem}[Approximate $\mathrm{AFEM}$]\label{thm: AFEM approximation}
    Assume there exist $\mathfrak{c},\mathfrak{C}>0$ such that $\mathfrak{c}\leq \kappa(x,\bfy) \leq \mathfrak{C}$ for all $x\in D$ and $\bfy\in\Gamma$.
    Let $\varepsilon>0$ and $K, L\in\mathbb{N}$ be the number of $\mathrm{AFEM}$ iterations and maximal refinements of each triangle, respectively. Consider a threshold marking strategy. Then there exists a CNN $\Psi$ such that $M(\Psi) \lesssim LK \log(\varepsilon^{-1})/\log(c_L^{-1})$ with $c_L \coloneqq \frac{cL}{1+cL}, c>0$ and for any $\bfy\in\Gamma$
    \begin{center}
        $\norm{u(\cdot,\bfy) - \fem(\Psi(\bfk_\bfy, \bff))}_{H^1(D)} \leq \norm{u(\cdot,\bfy) - \fem(\mathrm{AFEM}(V_1, K))}_{H^1(D)} + \varepsilon.$ 
    \end{center}
\end{theorem}

The proof of this theorem can be found in \cref{section: afem approx}.
The complete architecture is depicted in~\Cref{fig: CNN arch}. Here, the solver in each step is approximated by U-Nets encoded by green arrows outputting approximations of the solution in a multigrid discretization (green boxes) as described in \cref{thm: LLMG_Approx}. The estimator (orange) is approximated based on the approximate solutions with networks as constructed in \cref{thm: estimator approx} and the refinement masks (purple boxes) are derived from the estimator and used in the next solver. Here, the space was refined uniformly in the first step leading to masks, which are $1$ everywhere. In the second step the space was refined locally leading to a $0/1$-mask on the finest level. Adding all continuous functions corresponding to the images in the green boxes as in \eqref{eq: v decomposition} leads to the full approximate solution.

\begin{remark}\label{remark: other marking strategies in AFEM approx}
    Global marking strategies such as D\"orfler marking cannot be implemented directly in a CNN due to its local action in a neighbourhood. However, such a marking can in principle be implemented outside the CNN based on the estimator prediction of the CNN. As an alternative, the marking strategy could be learned by a separate NN based on the locally adapted training data. This marking NN could then be combined with the proposed CNN. Recent research in this direction can e.g. be found in~\cite{gillette2024learning,sluzalec2023quasi,yang2023reinforcement,foucart2023deep,feischlrecurrent}.
\end{remark}

\begin{figure}
    \centering
    \begin{tikzpicture}
        \node (CNN) at (0,-.5) {\includesvg[width=\linewidth]{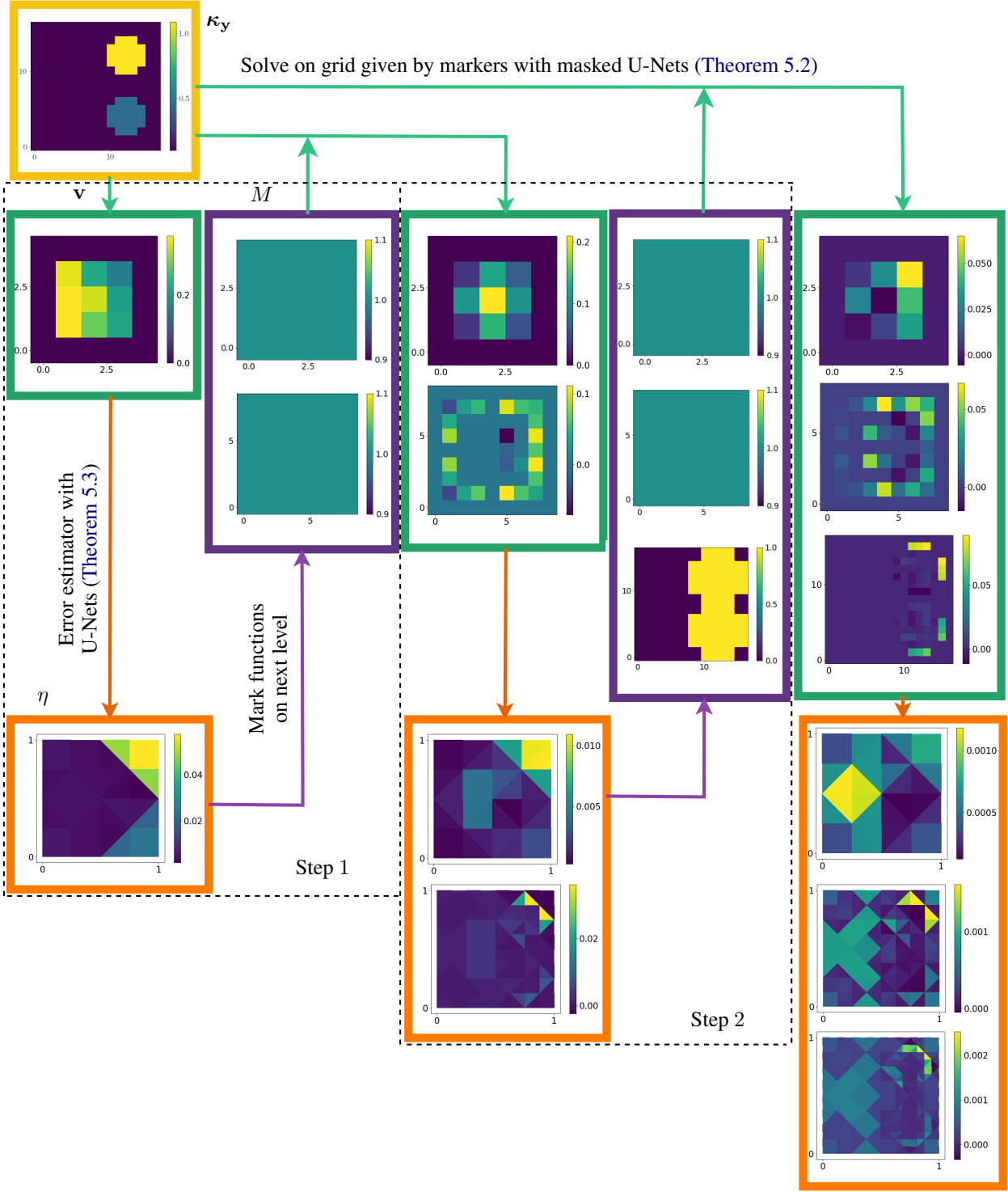}};
        \node (kappa) at (-4.7,8.9) {$\bfk_\bfy$};
        
        \node (thm) at (0.4,8.2) {Solve on grid given by markers with masked U-Nets (\cref{thm: LLMG_Approx})};
        \node[rotate=90, align=center] (thm) at (-7,0.1) {Error estimator with\\ U-Nets (\cref{thm: estimator approx})};
        \node[rotate=90, align=center] (thm) at (-3.9,-1.7) {Mark functions\\ on next level};
        
        \node at (-3, -5) {Step $1$};
        \node at (3.5, -7.5) {Step $2$};
        
        \node (u1) at (-7,6.1) {$\bfv$};
        \node (e1) at (-7.6,-2.2) {$\eta$};
        \node (m1) at (-4,6.1) {$M$};

    \end{tikzpicture}
    \caption{
    The derived CNN architecture is depicted for an approximation of three steps of the $\mathrm{AFEM}$. 
    The CNN mapping starts with the nodal interpolation of the parameter field $\bfk_\bfy$ on the finest level given as an input image.
    As in~\Cref{alg: AFEM}, in every step the solution $\bfv$ of the system of linear equations (\cref{alg: LLMG: first update}, \cref{alg: LLMG: second update}) is calculated (green arrows) and the solution is given in a multigrid decomposition~\eqref{eq: v decomposition} (green boxes), compare~\cref{thm: LLMG_Approx}.
    The approximation of the error estimator $\eta$ represented as in~\cref{thm: estimator approx} is depicted in orange errors and its multigrid decomposition as in \eqref{eq: DG0} in orange boxes. The derived markers are encoded by $0/1$-masks $M$ visualized in the purple boxes.
    The masks are then used in the network of the next iteration to enforce an action only on local parts of the larger images to imitate a local mesh refinement. Note that one $\mathrm{AFEM}$ iteration corresponds to one black dashed box.
    }
    \label{fig: CNN arch}
\end{figure}

\section{Numerical experiments}\label{section: numerics}

This section is concerned with the practical performance of the proposed architecture. Here, we present preliminary proof of concept results. The architecture should be tested for more steps of the adaptive solver and different expansions of the parameter field.
The numerical tests are implemented for a parametric stationary diffusion problem with parametric coefficient defined by
\begin{align*}
    \kappa(\cdot,\bfy) \coloneqq 0.1 + \bfy_1 \chi_{D_1} + \bfy_2 \chi_{D_2}.
\end{align*}
For this ``cookie problem'', we assume that $\bfy \sim U([0,1]^2)$, $D_1,D_2$ are disks of radius $r=0.15$ and centers at $(0.75, 0.25)$ and $(0.75, 0.75)$, respectively.
The architecture is implemented for $K=3$ steps of the $\mathrm{AFEM}$. For each step a solver was approximated with $3,2,1$ U-Nets. The overall number of trainable parameters is $2\, 441\, 516$. 
In \Cref{fig: network images}, the final network outputs are compared to a reference solution (obtained by solving on twice uniformly refined meshes) as well as error estimators and markers from the training data.
It can be observed that solution and estimator are approximated well with local errors magnitudes smaller than the actual values.
Note that the marker based on the error estimator as derived in the network differs from the marker used to generate the data. This inaccuracy in the prediction of the marked elements leads to nonzero elements in the network output in areas, which ideally should not be refined.
In~\Cref{fig: network images} this leads to the local error in the upper right corner in the solution approximation and to the difference in $H^1$-error decays in the first row in~\Cref{fig: error decay over levels,fig: error decay over parameters}.
In these figures, the graphs depict the $H^1$ and $L^2$ relative errors of the neural network approximation and the solutions of the $\mathrm{AFEM}$. In~\Cref{fig: error decay over levels}, the errors are plotted over the steps of the $\mathrm{AFEM}$ and in~\Cref{fig: error decay over parameters} they are plotted against the degrees of freedom used in the approximation and the $\mathrm{AFEM}$.
The graphs in the first rows show results for the fully adaptive CNNs, choosing the markers based on the approximated estimators without using the mesh refinement used in the \texttt{FEniCS}\cite{fenics} FEM package, which was used for the data generation. This element is still inexact and needs to be adjusted.
The second rows show the results based on a CNN using masks known from the data generation. Since the decays with known masks match the true error decay of the test data, the main step to optimize is the mask generation.

In summary, for a local refinement with known masks, the relative $H^1$ and $L^2$ errors of the network show the same decay as the true $\mathrm{AFEM}$ for three steps.
For a fully adaptive CNN, the $L^2$ errors match the true errors while the $H^1$ errors are larger, probably due to inexact masks, which will be a topic of future investigations.
\begin{figure}
    \centering
    \begin{tikzpicture}

        \node at (0,4.5) {$M$};
        \node at (0,3) {\includegraphics[width=0.21\linewidth]{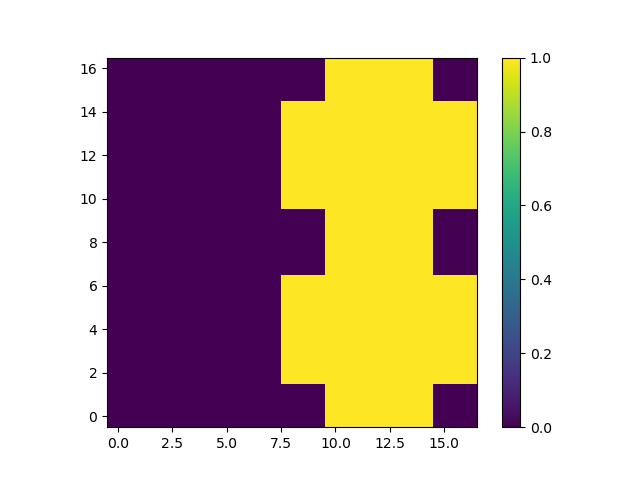}};
        
        \node at (4, 4.5) {$M^{\text{NN}}$};
        \node at (4,3) {\includegraphics[width=0.21\linewidth]{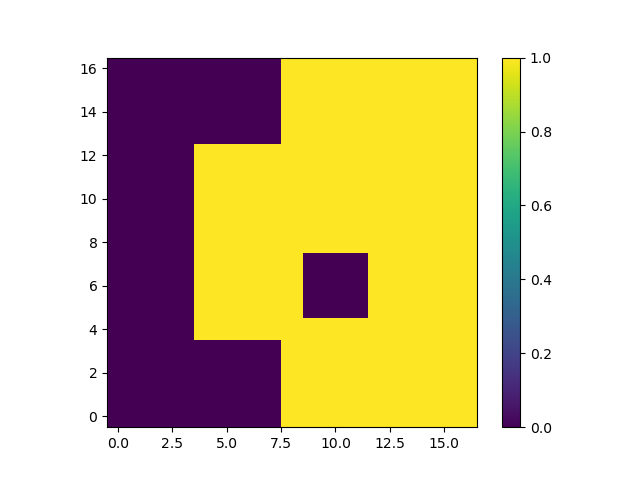}};
        
        \node at (7.7,4.5) {$\kappa$};
        \node at (7.8,3) {\includegraphics[width=0.21\linewidth]{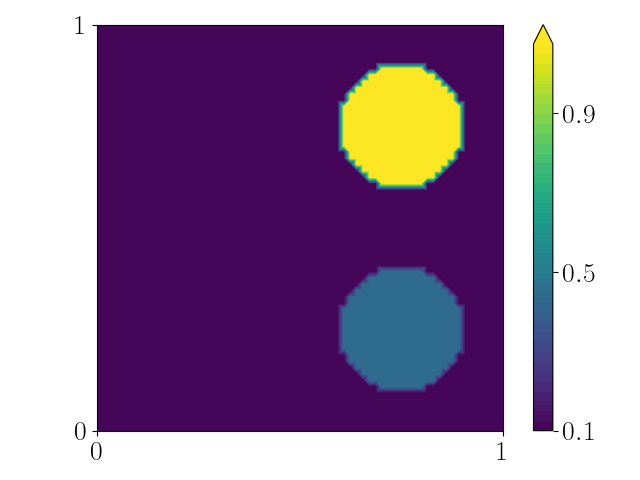}};
    
        \node at (0,1.5) {$u_h$};
        \node at (0,0) {\includegraphics[width=0.21\linewidth]{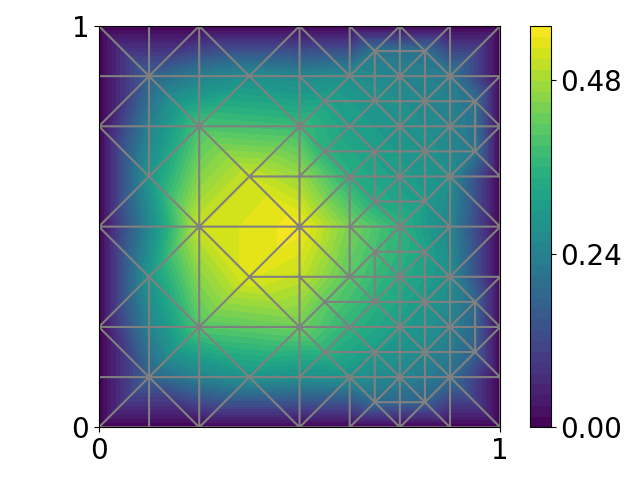}};
    
        \node at (4,1.5) {$u_h^{\text{NN}}$};
        \node at (4,0) {\includegraphics[width=0.21\linewidth]{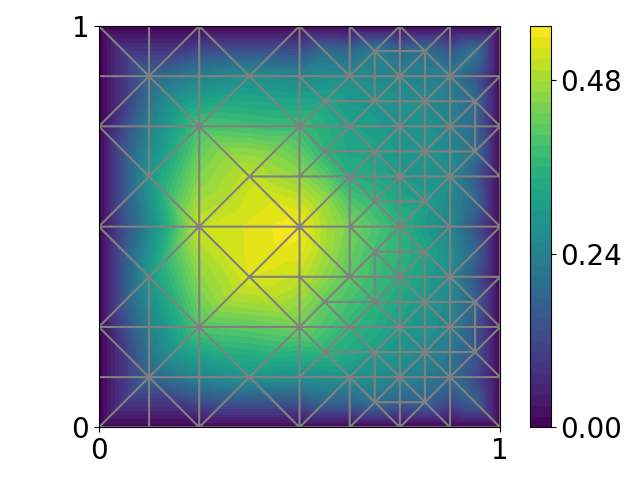}};
    
        \node at (7.7,1.5) {$\lvert u_h^{\text{NN}} - u^h \rvert$};
        \node at (8,0) {\includegraphics[width=0.21\linewidth]{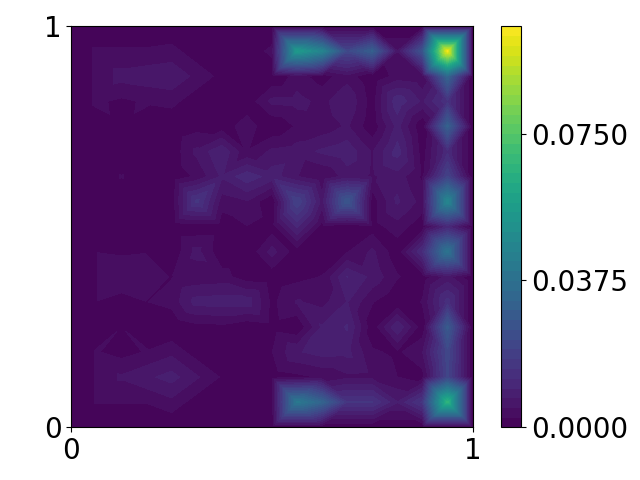}};
    
        \node at (0,-1.5) {$\eta$};
        \node at (0,-3) {\includegraphics[width=0.21\linewidth]{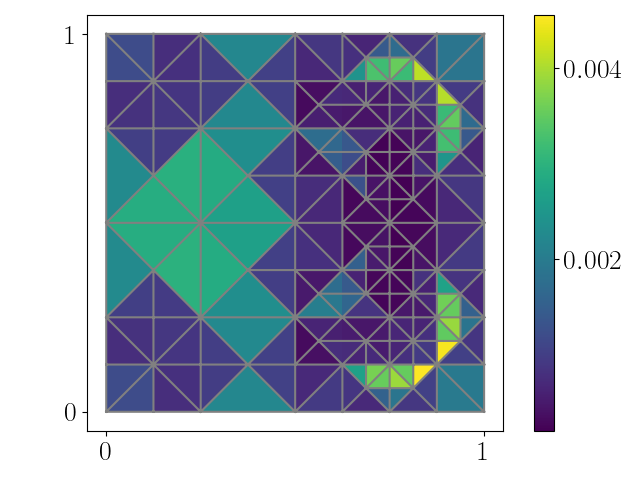}};
    
        \node at (4, -1.5) {$\eta^{\text{NN}}$};
        \node at (4,-3) {\includegraphics[width=0.21\linewidth]{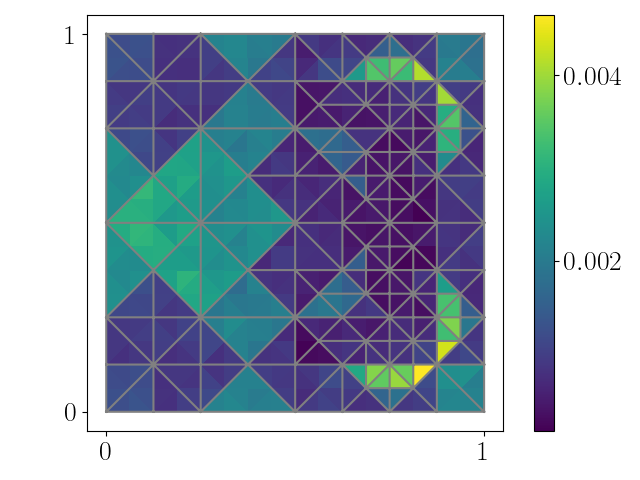}};
                
        \node at (7.7,-1.5) {$\lvert \eta^{\text{NN}} - \eta \rvert$};
        \node at (8,-3) {\includegraphics[width=0.21\linewidth]{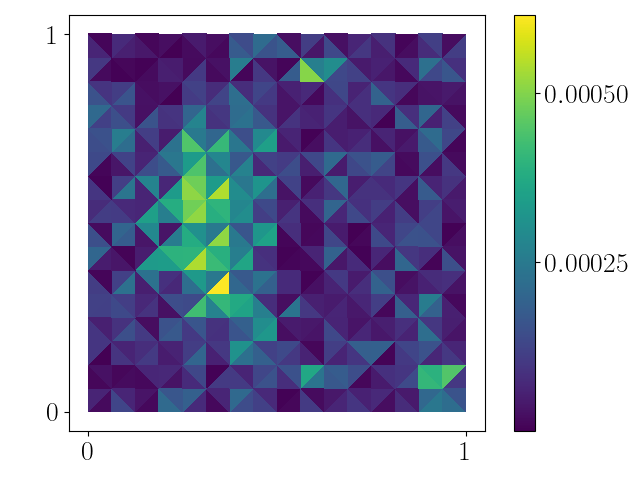}};
    \end{tikzpicture}
    \caption{True and network prediction solutions, estimators and markers are plotted for the third step of the $\mathrm{AFEM}$.
     From left to right, the first row shows the marking image on the third level, which was used for training, the marking image, which the network deduced from the estimator of the the second solution and the parameter field sample.
     The second row shows the Galerkin solution on the mesh used for training, the second plot shows the network output, and the last images shows the difference between the two.
    The last row shows the first the estimator of the third solution in the $\mathrm{AFEM}$ iteration, the network approximation of the estimator and the difference between the two. It can be seen that the pointwise distances are a magnitude smaller than the true values.}
    \label{fig: network images}
\end{figure}

\begin{figure}
    \centering
    \begin{tikzpicture}
        \node at (0,2.2) {$H^1$ error decay over levels};
        \node at (7,2.2) {$L^2$ error decay over levels};
        \node[rotate=90] at (-2.7,0.1) {Adaptive neural network};
        \node[rotate=90] at (-2.7,-3.7) {Fixed local refinement};
        
        \node at (0,0) {\includegraphics[width=0.3\linewidth]{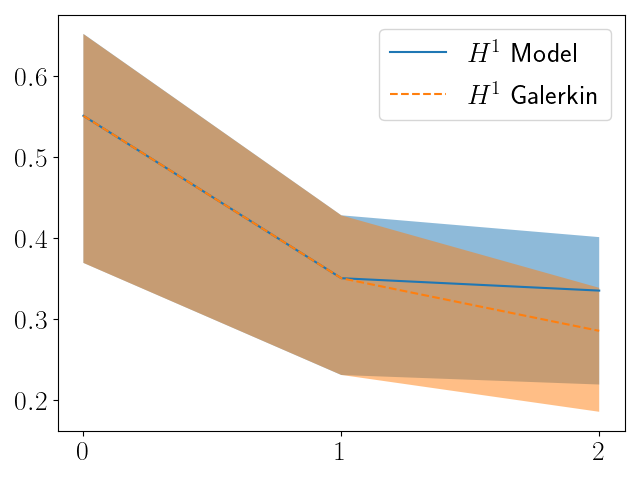}};
        \node at (0,-3.8) {\includegraphics[width=0.3\linewidth]{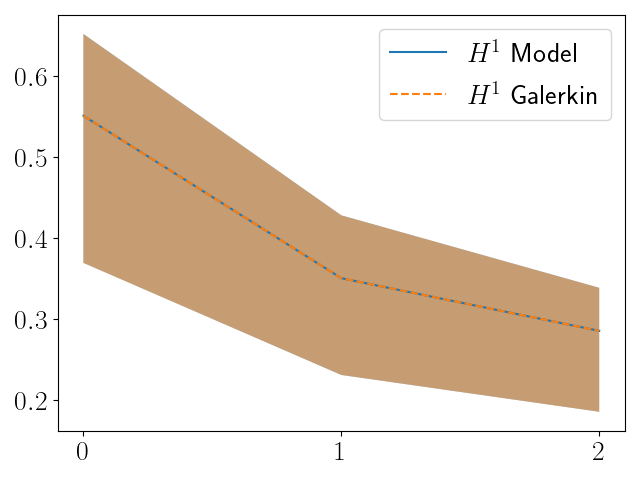}};
        
        \node at (7,0) {\includegraphics[width=0.3\linewidth]{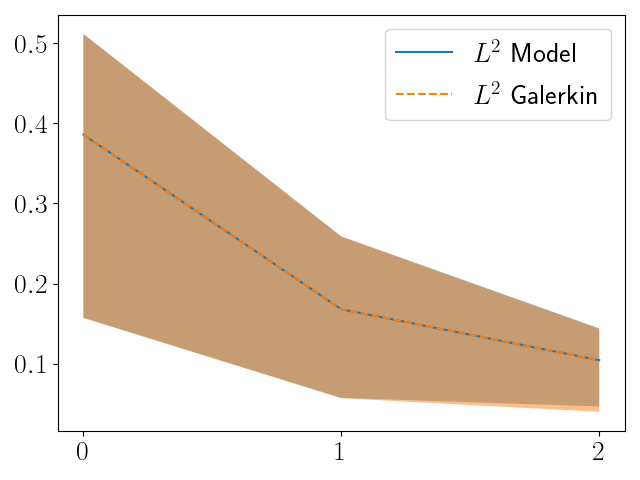}};
        \node at (7,-3.8) {\includegraphics[width=0.3\linewidth]{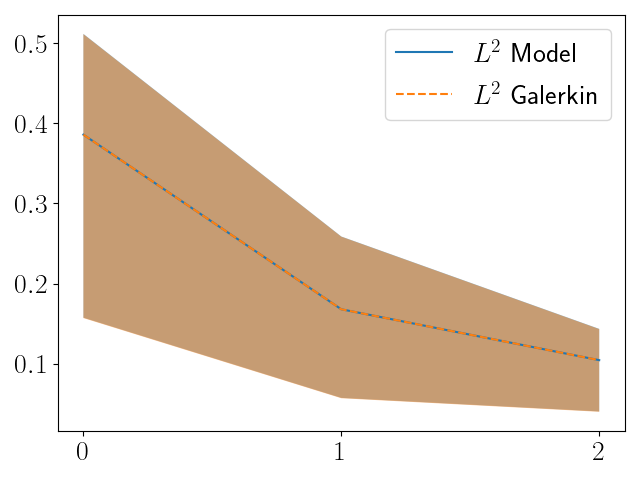}};

    \end{tikzpicture}
    \caption{The average relative $H^1$ and $L^2$ errors are plotted against the number of steps of the $\mathrm{AFEM}$ $K=1,2,3$ together with the error range from the minimal to the maximal error in every step.}
    \label{fig: error decay over levels}
\end{figure}
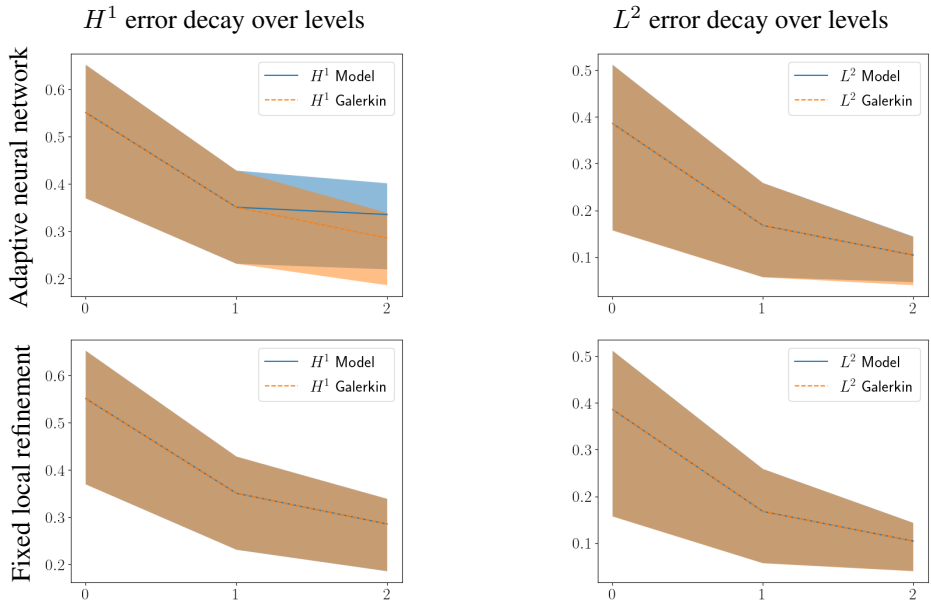

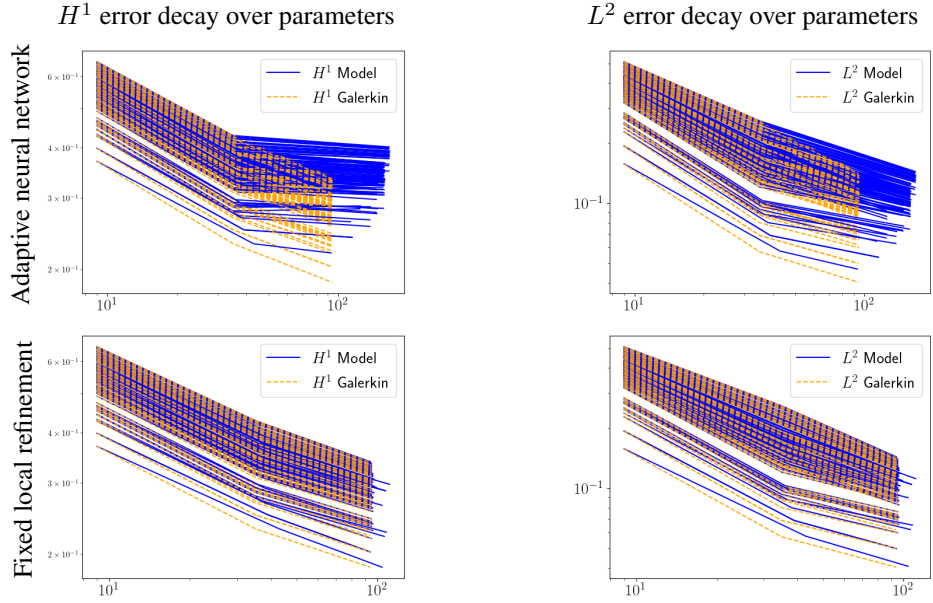
\begin{figure}
    \centering
    \begin{tikzpicture}
        \node at (0,2.2) {$H^1$ error decay over parameters};
        \node at (7,2.2) {$L^2$ error decay over parameters};
        \node[rotate=90] at (-2.7,0.1) {Adaptive neural network};
        \node[rotate=90] at (-2.7,-3.7) {Fixed local refinement};
        
        \node at (0,0) {\includegraphics[width=0.3\linewidth]{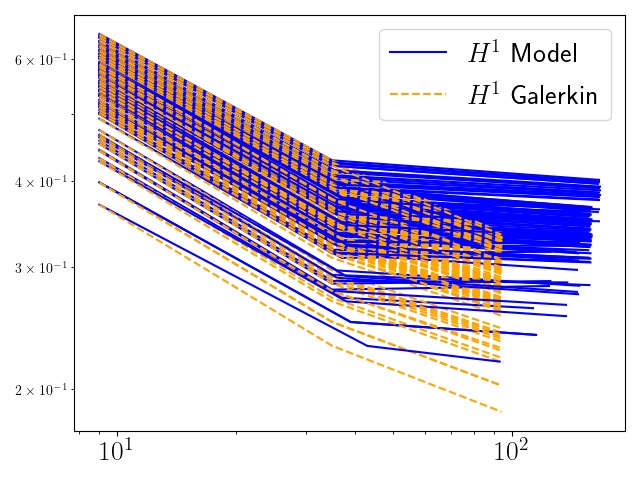}};
        \node at (0,-3.8) {\includegraphics[width=0.3\linewidth]{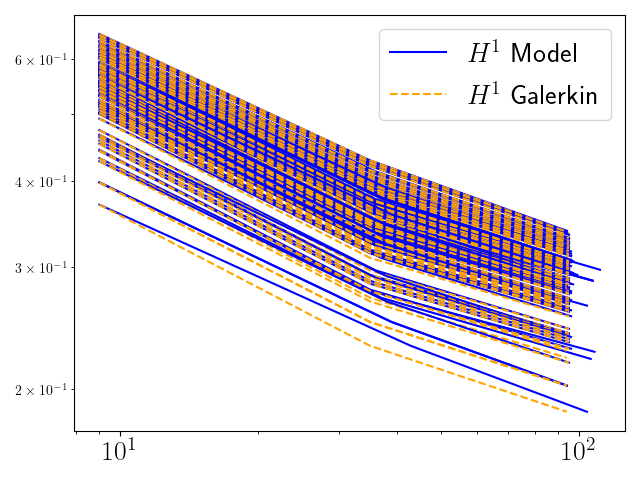}};
        
        \node at (7,0) {\includegraphics[width=0.3\linewidth]{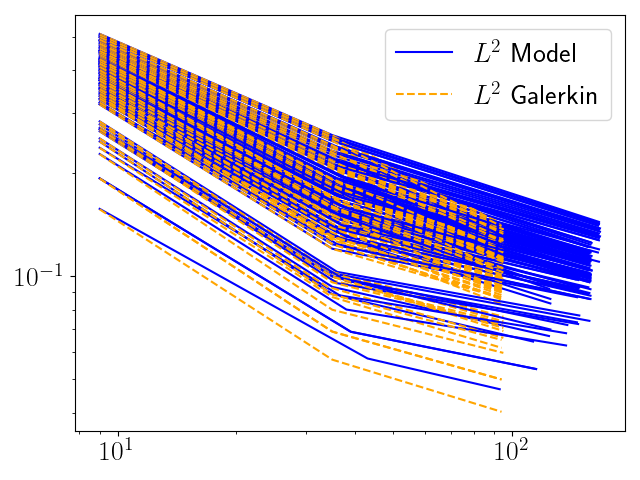}};
        \node at (7,-3.8) {\includegraphics[width=0.3\linewidth]{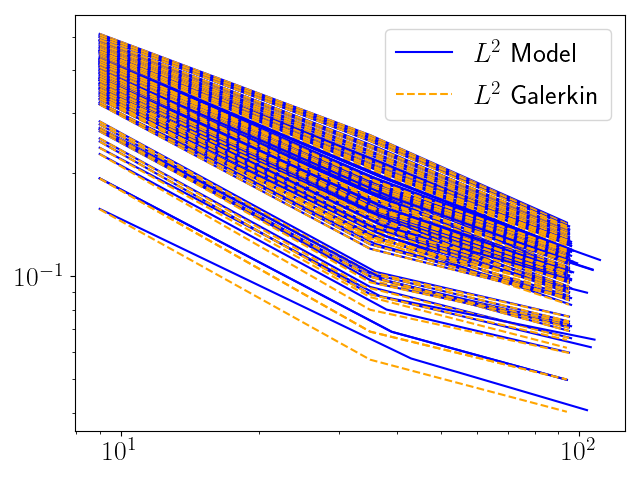}};
    \end{tikzpicture}
    \caption{The average relative $H^1$ and $L^2$ errors are plotted against the number of parameters used by the $\mathrm{AFEM}$ and the neural network.}
    \label{fig: error decay over parameters}
\end{figure}

\section{Outlook}\label{section: outlook}

In the paper, we derived an algorithm $\mathrm{LLMG}$, which approximates the parametric diffusion PDE on a fixed grid based on a multigrid decomposition of the solution and a successive subspace correction algorithm.
We showed that the derived algorithm can be approximated efficiently in the number of parameters by a derived CNN architecture.
Furthermore, we showed that an efficient and reliable finite element error estimator can be approximated by a specific CNN construction.
These results were combined to show upper bounds for the number of parameters of CNNs approximating a complete adaptive finite element scheme.

It is now interesting to put this architecture to use and explore the efficiency of the networks numerically, mainly with respect to two aspects. First, the number of calculations on each level should be reduced comparing to fully refined meshes~\cite{cosi} due to the submanifold sparse convolutions on sparse tensors. Note that the sparsity of the tensors stems from the used multigrid decomposition of the data.
Second, the efficiency with respect to the number of samples needed for training should be explored, considering that in each step corrections with decreasing influence on the whole solutions need to be learned. This should lead to fewer training samples on fine grids and hence a more efficient training and data generation process.

In addition, one might be interested in deriving convergence results for the proposed adaptive scheme with the presented refinement and for different marking strategies. The direction of showing the CNN approximation results for other meshes without hanging nodes might also be of interest.

\newpage
\bibliographystyle{abbrv}
\bibliography{lib}

\appendix
\newpage
\section{Error Estimator Derivation}\label{section:error estimator derivation}
We consider the residual in variational form for error $e := u - u_h$, where $u_h$ is the Bubnov-Galerkin approximation of $u$ on $V_h$ and $\mathcal{T}$ is an (exact) triangulation of domain $D$.
The residual based error estimator is common knowledge in the FEM literature, cf.~\cite{braess2007finite,carstensen2012review,verfurth}.
For the sake of a self-contained presentation, we provide the derivation in what follows since it may help the comprehension of the CNN approximation in this paper. It holds that
\begin{align*}
    a_{\bfy,h}(e,v) &= a_{\bfy,h}(u,v) - a_{\bfy,h}(u_h,v) = f(v) - a_{\bfy,h}(u_h,v) = \int_D fv -  \kappa_h(\cdot,\bfy)\scpr{ \nabla u_h, \nabla v } \dx\\
    &=\sum_{T\in\mathcal{T}} \int_T fv -  \kappa_h(\cdot,\bfy)\scpr{ \nabla u_h, \nabla v } \dx.
\end{align*}
Furthermore, let $n_T$ be the unit outward normal vector to $\partial T$ for $T\in\mathcal{T}$. Then, it holds that
\begin{align*}
    a_{\bfy,h}(e,v)&=\sum_{T\in\mathcal{T}} \int_T fv\dx  + \int_T v \nabla \cdot (\kappa_h(\cdot,\bfy)\nabla u_h)\dx - \int_{\partial T} v\kappa_h(\cdot,\bfy) \frac{\partial u_h}{\partial n_T} \mathrm{d}s\\
    &=\sum_{T\in\mathcal{T}} \int_T(f+\nabla\cdot(\kappa_h(\cdot,\bfy)\nabla u_h))v\dx + \sum_{\gamma\in\partial \mathcal{T}}\int_\gamma v\kappa_h(\cdot,\bfy) \bra{\scpr{\nabla u_h^{(1)}, n_\gamma^{(1)} } + \scpr{ \nabla u_h^{(2)}, n_\gamma^{(2)} }} \mathrm{d}s.
\end{align*}
Here, $n_\gamma^{(1)}$ and $n_\gamma^{(2)}$ are the unit outward normal vectors of the elements of the mesh containing $\gamma$ and $\nabla u_h^{(1)},\nabla u_h^{(2)}$ are the gradients of $u_h$ on the elements.
With the Galerkin projection $v_h$ of $v$ on $V_h$ and the definition of the jump \eqref{def: jump}, with some $\tilde C>0$ one gets the estimate
\begin{align*}
    a_{\bfy,h}(e,v) &=\sum_{T\in\mathcal{T}} \int_T(f+\nabla\cdot(\kappa_h(\cdot,\bfy)\nabla u_h))(v-v_h)\dx + \sum_{\gamma\in\partial \mathcal{T}}\int_\gamma \jump{\kappa(\cdot, \bfy)\nabla u_h \cdot \hat{n}} (v-v_h) \mathrm{d}s\\
    &\leq \sum_{T\in\mathcal{T}} \norm{f+\nabla\cdot(\kappa_h( \cdot, \bfy)\nabla u_h}_{L_2(T)}\norm{v - v_h}_{L_2(T)} + \sum_{\gamma\in\partial \mathcal{T}}\norm{\jump{\kappa_h(\cdot, \bfy)\nabla u_h}}_{L_2(\gamma)}\norm{v - v_h}_{L_2(\gamma)}\\
    &\leq \tilde{C}\norm{v}_{H^1(\Omega)} \left( \sum_{T\in\mathcal{T}} h_T^2\norm{f+\nabla\cdot(\kappa_h(\cdot, \bfy)\nabla u_h)}_{L_2(T)}^2 + \sum_{\gamma\in\partial\mathcal{T}}h_E \norm{\jump{\kappa_h(\cdot, \bfy)\nabla u_h}}_{L_2(\gamma)}^2 \right)^{1/2}.
\end{align*}
Setting $v=e$ and with $\norm{v}_{H^1(\Omega)}\leq C\anorm{v}$ we arrive at
\begin{align*}
    \anorm{e}^2 &= \frac{(\anorm{e}^2)^2}{\anorm{e}^2} = \frac{(a_{\bfy,h}(e,e))^2}{\anorm{e}^2}\\
    &\leq \frac{1}{\anorm{e}^2} \tilde{C}^2 C^2\anorm{e}^2 \left( \sum_{T\in\mathcal{T}} h_T^2\norm{f+\nabla\cdot(\kappa_h(\cdot, \bfy)\nabla u_h)}_{L_2(T)}^2 + \sum_{\gamma\in\partial\mathcal{T}}h_T \norm{\jump{\kappa_h( \cdot, \bfy)\nabla u_h}}_{L_2(\gamma)}^2 \right)\\
    &\leq \hat{C}\sum_{T\in\mathcal{T}} h_T^2\norm{f+\nabla\cdot(\kappa_h(\cdot, \bfy)\nabla u_h)}_{L_2(T)}^2 + h_T \norm{\jump{\kappa_h(\cdot, \bfy)\nabla u_h}}_{L_2(\partial T)}^2,
\end{align*}
which proofs reliability of the estimator with some $\hat C>0$.

\section{Proofs of convergence of the levelwise local multigrid algorithm}
The sequence of uniform meshes $(\T_k)_{k=1}^{L}$, the piecewise linear finite element function spaces over the meshes $U^k$ and the space $V_h = \sum_{k=1}^L V^k$ are introduced in \cref{section: solver} with $V^k\subseteq U^k$. Furthermore, recall that the operator $Q_k$ is defined as the $\ell^2$--projection of the coefficients of functions in $V_h$ onto the coefficients of $V^k$. The action of $A_\bfy$ restricted to the coefficient spaces of $V^k$ is defined by $A^k_\bfy$.

\subsection{Successive Subspace Correction}\label{subsection: SSC analysis}
The successive subspace algorithm ($\mathrm{SSC}$) approximates the solution $\bfu \in\mathbb{R}^{\sum_{k=1}^L n_k^2}$ to $A_\bfy \bfu = \bff$ by iteratively updating the solution on the individual subspaces $\mathbb{R}^{n_k^2}$ for $k=1,\dots,L$ by the weighted residual, see \cref{alg: ssc}.

\begin{algorithm}
\caption{Successive Subspace Correction $\mathrm{SSC}(\bfw)$}\label{alg: ssc}

\For{k=1,\dots, L}{
$\bfw \gets \bfw + \omega^k_\bfy(\bff^k - Q_k A_\bfy \bfw)$
}
return $\bfw$
\end{algorithm}

The error of the current approximation after each step of the algorithm denoting the update by $\bfw_{\text{update}}$ can be written as 
\begin{align*}
    \bfw_{\text{update}} - \bfu = \bfw + \omega^k_\bfy(\bff^k - Q_k A_\bfy \bfw) - \bfu = (I - \omega^k_\bfy Q_kA_\bfy)(\bfw - \bfu).
\end{align*}
Define the operator $T_k : \mathbb{R}^{\sum_{\ell=1}^L n_\ell} \to \mathbb{R}^{n_k^2}$ by $x \mapsto \omega^k_\bfy Q_k A_\bfy x$, where 
$A_\bfy^k$ is the restriction of $A_\bfy$ to $V^k$.
Let $\lambda_{\max}(B)$ denote the largest eigenvalue of matrix $B$ and $\omega^k_\bfy$ be the smoothing factor such that  
\begin{align}\label{def: omega}
    0<\omega^k_\bfy \leq \lambda_{\text{max}}(A^k_\bfy)^{-1}.
\end{align}
The error of the successive subspace correction algorithm then has the recursive form
\begin{align*}
    \bfu - \mathrm{SSC}(\bfw) = (I- T_L)(I - T_{L-1})\dots (I - T_0)(\bfu - \bfw).
\end{align*}
To bound the error, the X-Z identity from \cite{chen} can be used.

\begin{lemma}[{\cite[Theorem 4]{chen}}]\label{lm: XZ-identity}
    Suppose that $\norm{I - \omega^k_\bfy A^k_\bfy}_{A^k_\bfy} <1$ 
    for each $k=0,\dots ,L$. Then there exists a $c_0\geq0$ such that
    \begin{align*}
        \norm{ (I - T_L)(I - T_{L-1})\dots (I - T_1)}^2_{A_\bfy} = \frac{c_0}{1 + c_0}
    \end{align*}
    with 
    \begin{align*}
        c_0 = \sup_{\norm{v}_{A_\bfy}=1} \inf_{\sum_{k=1}^Lv_i = v} \sum_{k=1}^L \norm{ \omega_\bfy^k (Q_k A_\bfy \sum_{i=k}^L v_i - {\omega_\bfy^k}^{-1} v_k)}^2_{\bar{R}_k^{-1}},
    \end{align*}
    where $\bar{R}_k = (2I - \omega_\bfy^k A^k_\bfy)\omega_\bfy^k$.
\end{lemma}

To ensure the condition in \cref{lm: XZ-identity}, we consider the following result.

\begin{lemma}[similar to {\cite[Lemma 4.3]{Hackbusch}}]\label{lm: subspace contraction}
    Let $k\in[L]$ and $\kappa(\cdot,\bfy)>0$ everywhere. 
    Then for any $\bfw\in\mathbb{R}^{\mathcal{I}^k_V}$
    it holds that
    \begin{align*}
        \norm{(I - \omega^k_\bfy A_\bfy^k)\bfw}_{A_\bfy^k} < \norm{\bfw}_{A^k_\bfy},
    \end{align*}
    where $0<\omega^k_\bfy\leq \lambda_{\max}(A_\bfy^k)^{-1}$. 
\end{lemma}

\begin{proof}
    Let $\Omega_k = \supp V_k$. First, assume some $\varphi: \Omega_k \to \mathbb{R}$ with $\varphi = 0$ on $\partial \Omega_k$.
    If $\varphi$ is not constant zero this implies that there exists a point $x_0\in\Omega_k$ and $\varepsilon>0$ such that $\nabla \varphi \neq 0$ on an $\varepsilon$ neighborhood of $x_0$ denoted b $U_\varepsilon(x_0)$.
    Then, due to $\kappa(\cdot, \bfy) >0$ everywhere, we obtain that
    \begin{align*}
        a_{\bfy,k} (\varphi, \varphi) = \int_{\Omega_k} \kappa(\cdot, \bfy) \scpr{\nabla \varphi, \nabla \varphi} \mathrm{d}x \geq \int_{U_\varepsilon(x_0)} \kappa(\cdot, \bfy) \scpr{ \nabla \varphi, \nabla \varphi} \mathrm{d}x >0,
    \end{align*}
    where we set $a_{\bfy,k} = a_{\bfy,h}$ as in \eqref{eq: variational darcy} for $V_h = V_k$. Therefore, for any $\bfw \in\mathbb{R}^{\mathcal{I}^k_V}$ 
    we have that
    \begin{align*}
        \bfw^\intercal A_\bfy^k \bfw = a_{\bfy, k}\left(\sum_{i\in\mathcal{I}^k_V}\bfw_i\phi_i, \sum_{i\in\mathcal{I}^k_V}\bfw_i\phi_i\right) >0
    \end{align*}
    and hence $A_\bfy^k$ is positive definite. 
    Let $N_k \coloneqq |\mathcal{I}^k_V|$ and denote the eigenvalues and eigenvectors of $A_\bfy^k$ 
    by $\lambda_i, \bfv^i$ for $i=1,\dots,N_k$ with
    \begin{align*}
        A_\bfy^k \bfv^i &= \lambda_i \bfv^i \quad\quad\text{ such that }\\
        \delta_{i,j} &= \scpr{\bfv^i,\bfv^j}_{\ell^2} \text{ for all } i,j=1,\dots,N_k.
    \end{align*}
    Furthermore, for $\bfw\in \mathbb{R}^{N_k}$, let 
    \begin{align*}
        J_k\bfw &\coloneqq (I - \omega^k_\bfy A_\bfy^k)\bfw.
    \end{align*}
    Then, with $\bfw = \sum_{i=1}^{N_k} c_i\bfv^i$ 
    \begin{align*}
        J_k\bfw = \sum_{i=1}^{N_k} c_i (I - \omega^k_\bfy A_\bfy^k)\bfv^i = \sum_{i=1}^{N_k} c_i (1 - \omega^k_\bfy \lambda_i)\bfv^i.
    \end{align*}
    Second,
    \begin{align*}
        |\bfw|^2 \coloneqq \sum_{i=1}^{N_k} \lambda_i (1-\lambda_i\omega^k_\bfy) c_i^2
    \end{align*}
    defines a semi-norm due to $0<\omega^k_\bfy \leq \lambda_{\max}(A_\bfy^k)^{-1}$.
    Then, the following statements hold
        \begin{align*}
            |\bfw|^2 &= \sum_{i,j=1}^{N_k} (1-\lambda_i\omega^k_\bfy) \lambda_i\scpr{\bfv^i,\bfv^j}_{\ell^2} c_i c_j 
            = \scpr{ \sum_{i=1}^{N_k} c_i (1-\lambda_i\omega^k_\bfy)\bfv^i, \sum_{j=1}^{N_k} c_j\bfv^j }_{A_\bfy^k}
            = \scpr{J_k\bfw, \bfw}_{A_\bfy^k},\\
            \norm{\bfw}_{A_\bfy^k}^2 &= \scpr{A_\bfy^k\bfw,\bfw}_{\ell^2} = \sum_{i,j=1}^{N_k} \lambda_i c_i c_j \scpr{\bfv^i,\bfv^j}_{\ell^2} = \sum_{i=1}^{N_k} c_i^2 \lambda_i,\\
            |\bfw|^2 &= \sum_{i=1}^{N_k} \lambda_i c_i^2 - \sum_{i=1}^{N_k} \lambda_i\omega^k_\bfy c_i^2 < \sum_{i=1}^{N_k} \lambda_i c_i^2 = \norm{\bfw}_{A_\bfy^k}^2.
        \end{align*}
        The last inequality holds true for $\bfw \neq 0$ since $A_\bfy^k$ is positive definite and $\omega^k_\bfy >0$.
        Then, with the H\"older inequality,
        \begin{align*}
            \norm{J_k\bfw}_{A_\bfy^k} &= \sum_{i=1}^{N_k}\lambda_i (c_i(1-\lambda_i\omega^k_\bfy))^2 = \sum_{i=1}^{N_k} (\lambda_i^{1/3} |c_i|^{2/3}) (\lambda_i^{2/3} |c_i|^{4/3}(1-\lambda_i\omega^k_\bfy)^2)\\
            &\leq \left(\sum_{i=1}^{N_k} (\lambda_i^{1/3} |c_i|^{2/3})^3\right)^{1/3} \left(\sum_{i=1}^{N_k}(\lambda_i^{2/3} |c_i|^{4/3}(1-\lambda_i\omega^k_\bfy)^2)^{3/2}\right)^{2/3}\\
            &= \left(\sum_{i=1}^{N_k} \lambda_i |c_i|^{2}\right)^{1/3} \left(\sum_{i=1}^{N_k}\lambda_i |c_i|^{2}(1-\lambda_i\omega^k_\bfy)^3\right)^{2/3}.
        \end{align*}
        This yields the result by estimating
        \begin{align*}
            \norm{J_k\bfw}_{A_\bfy^k}^3 &= \left(\norm{J_k\bfw}_{A_\bfy^k}^2\right)^{3/2}\\
            &\leq \left(\sum_{i=1}^{N_k} \lambda_i |c_i|^{2}\right)^{1/2} \left(\sum_{i=1}^{N_k}\lambda_i |c_i|^{2}(1-\lambda_i\omega^k_\bfy)^3\right)\\
            &= \norm{\bfw}_{A_\bfy^k}^2 |J_k\bfw|\\
            &< \norm{\bfw}_{A_\bfy^k}^2 \norm{J_k\bfw}_{A_\bfy^k}
        \end{align*}
        and dividing by $\norm{J_k\bfw}_{A_\bfy^k}$.
\end{proof}

\begin{theorem}\label{thm: X-Z identity}
Assume that there exists a constant $C>0$ such that $\lambda_{\max}(A_\bfy^k) \leq \lambda_{\max}(A_\bfy) \leq C$ for all $\bfy\in\Gamma$ and choose $0<\omega^k \leq C^{-1}$. Then, the error decays with
\begin{align*}
    \norm{(I - T_L)(I - T_{L-1}) \dots (I - T_1)}^2_{A_\bfy} \leq \frac{c_0}{1+c_0}
\end{align*}
for some $c_0\leq \frac{\lambda_{\max}(A_\bfy)}{\lambda_{\min}(A_\bfy)}L$.
Furthermore, if there exist constants $c_1,c_2>0$ such that $c_1\leq \kappa(\cdot,\bfy)\leq c_2$ for all $x\in D$ and $\bfy\in\Gamma$  (uniform boundedness) leads to a bound of the convergence rate $c_0\leq cL$ independent of $\bfy$ for some $c>0$. 
\end{theorem}
\begin{proof}
    To apply \cref{lm: XZ-identity}, we only need to verify that the smoothing on the subspaces yields a contraction for each $k=1,\dots, L$, i.e.
    \begin{align*}
        \norm{I - T_k}_{A_\bfy^k} < 1.
    \end{align*}
    This is established in \cref{lm: subspace contraction}. The constant in \cref{lm: XZ-identity}
    \begin{align*}
        c_0 = \sup_{\norm{\bfv}_{A_\bfy}=1} \inf_{\sum_{k=1}^L Q_i^\intercal \bfv_i = \bfv} \sum_{k=1}^L \norm{ \omega_\bfy^k Q_k A_\bfy \sum_{i=k}^L Q_i^\intercal \bfv_i - \bfv_k}^2_{\bar{R}_k^{-1}},
    \end{align*}
    for $\bar{R}_k = (2I - \omega_\bfy^k A^k_\bfy)\omega_\bfy^k$ has to be bounded by a constant independent of $\bfy$.
    If $\norm{T}<1$ then $\norm{(I-T)^{-1}} \leq \frac{1}{1-\norm{T}}$. 
    Therefore, $T \coloneqq \frac{\omega^k}2 A^k_\bfy$, \eqref{def: omega} and $\norm{T} \leq \frac12 \lambda_{\max}(A_\bfy^k)^{-1} \norm{A_\bfy^k} = \frac12 <1$ implies that 
    \begin{align*}
        \norm{\bar{R}_k^{-1}} = \norm{(2\omega^k(I-T))^{-1}} \leq \frac{\lambda_{\max}(A_\bfy^k)}{2} \norm{(I-T)^{-1}} \leq \frac{\lambda_{\max}(A_\bfy^k)}{2} \frac{1}{1-\norm{T}} = \lambda_{\max}(A_\bfy^k).
    \end{align*}
    For $p_k:V\to V_k$ the $A_\bfy$ orthogonal projection onto $V_k$, it holds that $Q_kA_\bfy = A_\bfy^k p_k$.
    Therefore, 
    \begin{align*}
        c_0 &\leq  \sup_{\norm{\bfv}_{A_\bfy}=1} \inf_{\sum_{k=1}^L   Q_i^\intercal \bfv_i = \bfv} \sum_{k=1}^L \lambda_{\max}(A_\bfy^k)\norm{ \omega^k Q_k A_\bfy \sum_{i=k}^L  Q_i^\intercal \bfv_i - \bfv_k}^2\\
        &= \sup_{\norm{\bfv}_{A_\bfy}=1} \inf_{\sum_{k=1}^L  Q_i^\intercal \bfv_i = \bfv} \sum_{k=1}^L \lambda_{\max}(A_\bfy^k) \left( \norm{\omega^k Q_k A_\bfy \sum_{i=k+1}^L  Q_i^\intercal \bfv_i + (\omega^kA_\bfy^k-I)\bfv_k} \right)^2\\
        &= \sup_{\bfv\neq 0} \inf_{\sum_{k=1}^L  Q_i^\intercal \bfv_i = \bfv} \sum_{k=1}^L \lambda_{\max}(A_\bfy^k) \left( \norm{\omega^k Q_k A_\bfy \sum_{i=k+1}^L  Q_i^\intercal \frac{\bfv_i}{\norm{\bfv}_{A_\bfy}} + (\omega^kA_\bfy^k-I)\frac{\bfv_k}{\norm{\bfv}_{A_\bfy}}} \right)^2\\
        &\leq \sup_{\bfv\neq 0} \inf_{\sum_{k=1}^L  Q_i^\intercal \bfv_i = \bfv} \sum_{k=1}^L \lambda_{\max}(A_\bfy^k) \left( \norm{\omega^k Q_k A_\bfy \sum_{i=k+1}^L  \frac{Q_i^\intercal \bfv_i}{\lambda_{\min}(A_\bfy)^{\frac12}\norm{\bfv}} + \frac{(\omega^kA_\bfy^k-I)\bfv_k}{\lambda_{\min}(A_\bfy)^{\frac12}\norm{\bfv}}} \right)^2\\
        &= \lambda_{\min}(A_\bfy)^{-1}\sup_{\norm{\bfv}_2 = 1} \inf_{\sum_{k=1}^L  Q_i^\intercal \bfv_i = \bfv} \sum_{k=1}^L \lambda_{\max}(A_\bfy^k) \left( \norm{\omega^k Q_k A_\bfy \sum_{i=k+1}^L  Q_i^\intercal \bfv_i + (\omega^kA_\bfy^k-I)\bfv_k} \right)^2 \\
        &\leq \lambda_{\min}(A_\bfy)^{-1} \sup_{\norm{\bfv}_2 = 1} \inf_{\sum_{k=1}^L  Q_i^\intercal \bfv_i = \bfv} \sum_{k=1}^L \lambda_{\max}(A_\bfy^k) \left( \omega^k \norm{Q_k A_\bfy } \norm{\sum_{i=k+1}^L Q_i^\intercal \bfv_i} + \norm{\omega^k A_\bfy^k -I}\norm{\bfv_k} \right)^2\\
        &\leq \lambda_{\min}(A_\bfy)^{-1} \sup_{\norm{\bfv}_2 = 1} \inf_{\sum_{k=1}^L  Q_i^\intercal \bfv_i = \bfv} \sum_{k=1}^L \lambda_{\max}(A_\bfy^k) \left( \frac{\lambda_{\max}(A_\bfy)}{\lambda_{\max}(A_\bfy)} \norm{\sum_{i=k+1}^L Q_i^\intercal \bfv_i} + \left(1-\omega^k\lambda_{\min}(A_\bfy^k)\right)\norm{\bfv_k} \right)^2\\
        &\leq \frac{\lambda_{\max}(A_\bfy)}{\lambda_{\min}(A_\bfy)} \sup_{\norm{\bfv}_2 = 1} \inf_{\sum_{k=1}^L  Q_i^\intercal \bfv_i = \bfv} \sum_{k=1}^L \left( \norm{\sum_{i=k+1}^L Q_i^\intercal \bfv_i} + \norm{\bfv_k} \right)^2\\
        &\leq \frac{\lambda_{\max}(A_\bfy)}{\lambda_{\min}(A_\bfy)} \sup_{\norm{\bfv}_2 = 1} \inf_{\sum_{k=1}^L  Q_i^\intercal \bfv_i = \bfv} \sum_{k=1}^L 1\\
        &\leq \frac{\lambda_{\max}(A_\bfy)}{\lambda_{\min}(A_\bfy)} L.
    \end{align*}
    This is derived by using that $\norm{v}_{A_\bfy}^2 \geq \lambda_{\min}(A_\bfy) \norm{v}^2$, $\max_k \lambda_{\max}(A_\bfy^k)\geq \lambda_{\min}(A_\bfy)$ and $Q_i\bfv_i$ are orthogonal with respect to the $\ell^2$ scalar product for $i=1,\dots,L$.
    Therefore, the claim follows from the assumption that $\lambda_{\min}(A_\bfy)$ and $\lambda_{\max}(A_\bfy)$ are uniformly bounded from below and above for all $\bfy$, respectively.
\end{proof}

\subsection{Local multigrid algorithm}\label{section: LMG}

\begin{algorithm}
\caption{Local Multigrid Algorithm LMG($\bfu_0$)}\label{alg: LMG}
$\bfu=\bfu_0$\\
\For{k=L,\dots, 0}{
$\bfu = \bfu + \omega^k_\bfy(\bff_k - Q_k A_\bfy \bfu)$
}
\For{k=0,\dots, L}{
$\bfu = \bfu + \omega^k_\bfy(\bff_k - Q_k A_\bfy \bfu)$
}
return $\bfu$
\end{algorithm}

Building on the analysis of the successive subspace correction algorithm in \cref{subsection: SSC analysis}, the error of the Local Multigrid \cref{alg: LMG} can be expressed similar to the successive subspace correction algorithm as
\begin{align*}
    \bfu - LMG(\bfu_0) = (I - T_0)\dots (I - T_{L-1})(I-T_L)(I- T_L)(I - T_{L-1})\dots (I - T_0)(u - u_0).
\end{align*}
Since the order of the subspaces in \cref{thm: X-Z identity} are not specified, the same constant smoothing factor $\omega_\bfy^k = \omega$ can be chosen such that the same constant $c_0>0$ satisfies
\begin{align*}
    \norm{\bfu - \mathrm{LMG}(\bfu_0)}_{A_\bfy}^2 \leq \left( \frac{c_0}{1+c_0} \right)^2 \norm{\bfu - \bfu^0}_{A_\bfy}^2.
\end{align*}

\subsection{Smoothing with multiple levels}
For one smoothing step, the calculation of $Q_k A_\bfy \bfu$ is required.
Since $\bfu$ is given as a levelwise discretization, we consider the evaluation of the multiplication on each level separately.
Recall the auxiliary vectors $\tilde{\bfu}^k$ and $\bar{\bfu}^k$ in \eqref{eq: u tilde k} and \eqref{eq: u bar k}.
Using the decomposition $\bfu = \bfu^{<k} + Q_k^\intercal \bfu^k + \bfu^{>k}$ facilitates a levelwise calculation as described below.

\subsubsection{Coarse grid smoothing}
For $Q_kA_\bfy\bfu^{<k}$ we get the following result.
\begin{lemma}\label{thm: decomp_smoothing_coarse}
    Let $\tilde{\bfu}^k$ be defined as in \eqref{eq: u tilde k} by
    \begin{align*}
        \tilde \bfu ^1 \coloneqq 0, \quad \tilde{\bfu}^k \coloneqq P_{k-1} \left(\tilde \bfu ^{k-1} + \uab{\empty}{k-1}\right)
    \end{align*}
     with $\uab{i}{k}\in\mathbb{R}^{\Iab{k}}$ equal to $\bfu_i^k$ for $i\in\I_V^k$ and zero otherwise.
    Then
    \begin{align*}
        Q_k A_\bfy \bfu^{<k}  = \Aab{k} \tilde \bfu^k.
    \end{align*}
\end{lemma}

\begin{proof}[Proof of \cref{thm: decomp_smoothing_coarse}] 
    Note that $Q_1A_\bfy \bfu^{<1} = 0 = A_\bfy^k \tilde{\bfu}^1$.
    For $k=2,\dots,L$ we show that for $x\in \supp V^k$ 
    \begin{align*}
        \sum_{\ell=1}^{k-1}\sum_{i\in \I_V^{\ell}} \bfu_i^\ell \varphi_i^\ell(x) {=} \sum_{i\in\Iab{k}} \tilde \bfu^k_i \varphi_i^k(x).
    \end{align*}
    The proof is by induction. For $k=1$, both sides are equal to $0$. 
    Assuming the statement holds for $k$, we get for $k+1$ and $x\in \supp V^{k+1} \subset \supp V^k$ 
    that
    \begin{align*}
        \sum_{\ell=1}^{k}\sum_{i\in\I_V^{\ell}} \bfu_i^\ell \varphi_i^\ell(x)  
        &=  \sum_{\ell=1}^{k-1}\sum_{i\in\I_V^{\ell}} \bfu_i^\ell \varphi_i^\ell(x) + \sum_{i\in\I_V^{k}} \bfu_i^k \varphi_i^k(x)
        = \sum_{i\in\Iab{k}} \tilde \bfu^k_i \varphi_i^k(x)  + \sum_{i\in\Iab{k}} \uab{i}{k} \varphi_i^k(x)\\
        &= \sum_{i\in\Iab{k}} (\tilde \bfu^k_i + \uab{i}{k}) \varphi_i^{k}(x) 
        = \sum_{j\in\Iab{k+1}}\sum_{i\in\Iab{k}} (\tilde \bfu^k + \uab{\empty}{k})_i (P_k^\intercal)_{i,j}\varphi_j^{k+1}(x)\\ 
        &= \sum_{j \in \Iab{k+1}} (P_k(\tilde \bfu^k + \uab{\empty}{k}))_j \varphi_j^{k+1}(x)
        =\sum_{i\in\Iab{k+1}} \tilde \bfu^{k+1}_i\varphi_i^{k+1}(x).
    \end{align*}

    Then, for $j\in\mathcal{I}_V^k$ it holds that
    \begin{align*}
        \left( Q_k A_\bfy \bfu^{<k} \right)_{j} &= \sum_{\ell=1}^{k-1} \sum_{i\in\I_V^{\ell}} \bfu_i^\ell \int \kappa_h(x,\bfy) \scpr{ \nabla \varphi_i^\ell(x), \nabla \varphi_j^{k}(x)}\dx\\
        &=  \int_{\supp V^k} \kappa_h(x,\bfy) \scpr{ \nabla \sum_{\ell=1}^{k-1} \sum_{i\in\I_V^{\ell}} \bfu_i^\ell\varphi_i^\ell (x), \nabla \varphi_j^{k}(x)}\dx\\
        &= \int \kappa_h(x,\bfy) \scpr{ \nabla \sum_{i\in\Iab{k}} \tilde \bfu_i^k \varphi_i^k(x), \nabla \varphi_j^{k}(x)}\dx\\
        &= \sum_{i\in\Iab{k}} \tilde \bfu_i^k  \int \kappa_h(x,\bfy) \scpr{ \nabla \varphi_i^k(x), \nabla \varphi_j^{k}(x)}\dx\\
        &= (\Aab{k} \tilde \bfu^k)_j.
    \end{align*}
\end{proof}

\subsubsection{Fine grid smoothing}
We now consider  $Q_kA_\bfy\bfu^{>k}$.

\begin{lemma}\label{thm: decomp_smooting_fine}
Let $\bar{\bfu}^k$ be defined as in \eqref{eq: u bar k} by
\begin{align*}
    \bar \bfu ^L \coloneqq 0, \quad \bar{\bfu}^k \coloneqq P_k^\intercal \left(\bar \bfu^{k+1} + \Aab{k+1}^\intercal \bfu^{k+1} \right).
\end{align*}
Then,
\begin{align*}
    Q_k A_\bfy \bfu^{>k} = \bar \bfu^k|_{\I_V^k}.
\end{align*}
    
\end{lemma}

\begin{proof}
    We prove the statement again by induction. Note that for $k=L$, it holds that $Q_kA_\bfy \bfu^{>L} = 0 = \bar\bfu^L$.
    Assuming that the statement holds for $k+1$, i.e. for $j\in\mathcal{I}^{k+1}_V$
    \begin{align*}
        (Q_{k+1} A_\bfy \bfu^{>k+1})_j = \sum_{\ell=k+2}^L\sum_{i\in\I_V^{\ell}} \bfu_i^\ell & \int \kappa_h(\cdot, \bfy) \scpr{\nabla \varphi_i^\ell, \nabla \varphi_j^{k+1}} \dx = \bar \bfu^{k+1}_j,
    \end{align*}
    we show that the statement also is true for $k$.
    In fact, for $j\in\mathcal{I}^k_V$ we deduce that 
    \begin{align*}
        (Q_k A_\bfy \bfu^{>k})_j &= \sum_{\ell=k+1}^L\sum_{i\in\I_V^\ell} \bfu_i^\ell \int \kappa_h(\cdot, \bfy) \scpr{\nabla \varphi_i^\ell, \nabla\varphi_j^k} \dx\\
        &= \sum_{m\in\Iab{k+1}} (P_k^\intercal)_{jm}\sum_{\ell=k+1}^L \sum_{i\in\I_V^\ell} \bfu_i^\ell \int \kappa_h(\cdot, \bfy) \scpr{\nabla \varphi_i^\ell, \nabla \varphi_m^{k+1}} \dx\\
        &= \sum_{m\in\Iab{k+1}} (P_k^\intercal)_{jm}\left( \sum_{i\in \mathcal{I}^{k+1}_V} \bfu_i^{k+1} a_{\bfy,h}(\varphi^{k+1}_i, \varphi^{k+1}_m) 
        + \sum_{\ell=k+2}^L\sum_{i\in\mathcal{I}_V^{\ell}} \bfu_i^\ell a_{\bfy,h}(\varphi^{\ell}_i, \varphi^{k+1}_m) 
        \right)\\
        &= \sum_{m\in\Iab{k+1}} (P_k^\intercal)_{jm} \left( \left(\Aab{k+1}^\intercal \bfu^{k+1}\right)_m + \bar \bfu_m^{k+1}) \right)\\
        &= \left(P_k^\intercal \left(\Aab{k+1}^\intercal \bfu^{k+1} + \bar \bfu^{k+1}\right)\right)_j
        = \bar\bfu^{k}_j.
    \end{align*}
\end{proof}

\section{CNNs}
\Cref{fig: convolutions} illustrates the different convolutions used in our architecture in~\Cref{section: CNN}.

\begin{figure}
\begin{tikzpicture}[
    layer/.style={draw, rectangle, align=center, minimum width=.75cm, minimum height=.75cm},
    kernel/.style={draw, rectangle, align=center, minimum width=.75cm, minimum height=.75cm},
    connection/.style={-{Latex[length=2mm,width=3mm]}, shorten >=1pt, shorten <=1pt},
    myarrow/.style={-{Latex[length=2mm,width=3mm]}, shorten >=1pt, shorten <=1pt},
  ]

  \matrix (inputMatrix) at (0, 0) {
    \node[layer, fill=gray!50] (input00) {$a_{11}$}; 
    & \node[layer, fill=gray!50] (input01) {$a_{12}$}; 
    & \node[layer, fill=gray!50] (input02) {$a_{13}$}; 
    & \node[layer] (input03) {$a_{14}$}; \\
    \node[layer, fill=gray!50] (input10) {$a_{21}$}; 
    & \node[layer, fill=gray!80] (input11) {$a_{22}$}; 
    & \node[layer, fill=gray!50] (input12) {$a_{23}$}; 
    & \node[layer] (input13) {$a_{24}$}; \\
    \node[layer, fill=gray!50] (input20) {$a_{31}$}; 
    & \node[layer, fill=gray!50] (input21) {$a_{32}$}; 
    & \node[layer, fill=gray!50] (input22) {$a_{33}$}; 
    & \node[layer] (input23) {$a_{34}$}; \\
    \node[layer] (input30) {$a_{41}$}; 
    & \node[layer] (input31) {$a_{42}$}; 
    & \node[layer] (input32) {$a_{43}$}; 
    & \node[layer] (input33) {$a_{44}$}; \\
  };
  \node (label) [above=.2cm of inputMatrix] {Vanilla convolution: Sweep };
  \node[right=-.1cm of label] { the center of the kernel over the image multiplying with the image and adding the products.};
    \node (conv) [right=of inputMatrix, matrix anchor=west, xshift=-1cm] {$\ast$};
  \matrix (kernelMatrix) [right=of inputMatrix, matrix anchor=west, xshift=-.5cm] {
    \node[kernel, fill=gray!50] (kernel00) {$k_{11}$}; 
    & \node[kernel, fill=gray!50] (kernel01) {$k_{12}$}; 
    & \node[kernel, fill=gray!50] (kernel02) {$k_{13}$}; \\
    \node[kernel, fill=gray!50] (kernel10) {$k_{21}$}; 
    & \node[kernel, fill=gray!80] (kernel11) {$k_{22}$}; 
    & \node[kernel, fill=gray!50] (kernel12) {$k_{23}$}; \\
    \node[kernel, fill=gray!50] (kernel20) {$k_{31}$}; 
    & \node[kernel, fill=gray!50] (kernel21) {$k_{32}$}; 
    & \node[kernel, fill=gray!50] (kernel22) {$k_{33}$}; \\
  };
    \node (equal) [right=of kernelMatrix, matrix anchor=west, xshift=-1cm] {$=$};
  \matrix (outputMatrix) [right=of kernelMatrix, matrix anchor=west, xshift=-.5cm] {
    \node[layer, fill=gray!80] (input00) {$\sum_{j\in U_1} a_{22+j}k_{22+j}$}; 
    & \node[layer] (input01) {$\sum_{j\in U_1} a_{23+j}k_{22+j}$};  \\
    \node[layer] (input10) {$\sum_{j\in U_1} a_{32+j}k_{22+j}$}; 
    & \node[layer] (input11) {$\sum_{j\in U_1} a_{33+j}k_{22+j}$};  \\
  };

  \node (label) [below=.2cm of inputMatrix] {Transpose $2$ strided convol};
  \node[right=-.25cm of label] {ution: Dilate the image with the kernel.};

  \matrix (inputMatrixst) [below left =1cm and -2cm of label] {
    \node[layer] (input00) {$a_{11}$}; 
    & \node[layer] (input01) {$a_{12}$}; \\
    \node[layer] (input10) {$a_{21}$}; 
    & \node[layer] (input11) {$a_{22}$}; \\
  };
    \node (conv) [right=of inputMatrixst, matrix anchor=west, xshift=-1.1cm] {$\ast^{2\text{ts}}$};
  \matrix (kernelMatrixst) [right=of inputMatrixst, matrix anchor=west, xshift=-.5cm] {
    \node[kernel] (kernel00) {$k_{11}$}; 
    & \node[kernel] (kernel01) {$k_{12}$}; 
    & \node[kernel] (kernel02) {$k_{13}$}; \\
    \node[kernel] (kernel10) {$k_{21}$}; 
    & \node[kernel] (kernel11) {$k_{22}$}; 
    & \node[kernel] (kernel12) {$k_{23}$}; \\
    \node[kernel] (kernel20) {$k_{31}$}; 
    & \node[kernel] (kernel21) {$k_{32}$}; 
    & \node[kernel] (kernel22) {$k_{33}$}; \\
  };
    \node (equal) [right=of kernelMatrixst, matrix anchor=west, xshift=-1cm] {$=$};
  \matrix (inputMatrixstst) [right=0cm of equal] {
    \node[layer] (input00) {$a_{11}k_{11}$}; 
    & \node[layer] (input00) {$a_{11}k_{12}$}; 
    & \node[layer] (input01) {$a_{11}k_{13}$}; 
    & \node[layer] (input00) {$0$}; 
    & \node[layer] (input02) {$0$};  \\
    \node[layer] (input00) {$a_{11}k_{21}$}; 
    & \node[layer] (input00) {$a_{11}k_{22}$}; 
    & \node[layer] (input01) {$a_{11}k_{23} $}; 
    & \node[layer] (input00) {$0$}; 
    & \node[layer] (input02) {$0$};  \\
        \node[layer] (input00) {$a_{11}k_{21}$}; 
    & \node[layer] (input00) {$a_{11}k_{22}$}; 
    & \node[layer] (input01) {$a_{11}k_{23} $}; 
    & \node[layer] (input00) {$0$}; 
    & \node[layer] (input02) {$0$};  \\
        \node[layer] (input00) {$0$}; 
    & \node[layer] (input00) {$0$}; 
    & \node[layer] (input01) {$0$}; 
    & \node[layer] (input00) {$0$}; 
    & \node[layer] (input02) {$0$};  \\
        \node[layer] (input00) {$0$}; 
    & \node[layer] (input00) {$0$}; 
    & \node[layer] (input01) {$0$}; 
    & \node[layer] (input00) {$0$}; 
    & \node[layer] (input02) {$0$};  \\
  };

  \node (astst) [right=0cm of inputMatrixstst] {$+$};
  
  \matrix (inputMatrixstst) [right=0cm of astst] {
    \node[layer] (input00) {$0$}; 
    & \node[layer] (input02) {$0$};
    & \node[layer] (input00) {$a_{12}k_{11}$}; 
    & \node[layer] (input00) {$a_{12}k_{12}$}; 
    & \node[layer] (input01) {$a_{12}k_{13}$};   \\
    \node[layer] (input00) {$0$}; 
    & \node[layer] (input02) {$0$}; 
    & \node[layer] (input00) {$a_{12}k_{21}$}; 
    & \node[layer] (input00) {$a_{12}k_{22}$}; 
    & \node[layer] (input01) {$a_{12}k_{23} $};  \\
    \node[layer] (input00) {$0$}; 
    & \node[layer] (input02) {$0$};  
    & \node[layer] (input00) {$a_{12}k_{21}$}; 
    & \node[layer] (input00) {$a_{12}k_{22}$}; 
    & \node[layer] (input01) {$a_{12}k_{23} $}; \\
        \node[layer] (input00) {$0$}; 
    & \node[layer] (input00) {$0$}; 
    & \node[layer] (input01) {$0$}; 
    & \node[layer] (input00) {$0$}; 
    & \node[layer] (input02) {$0$};  \\
        \node[layer] (input00) {$0$}; 
    & \node[layer] (input00) {$0$}; 
    & \node[layer] (input01) {$0$}; 
    & \node[layer] (input00) {$0$}; 
    & \node[layer] (input02) {$0$};  \\
  };
  
  \node (astst) [right=0cm of inputMatrixstst] {$+ \dots$};

  \matrix (inputMatrix) [below left=1.8cm and -4cm of inputMatrixst] {
    \node[layer, fill=gray!50] (input00) {$a_{11}$}; 
    & \node[layer, fill=gray!50] (input01) {$a_{12}$}; 
    & \node[layer, fill=gray!50] (input02) {$a_{13}$}; 
    & \node[layer] (input03) {$a_{14}$};
    & \node[layer] (input03) {$a_{15}$}; \\
    \node[layer, fill=gray!50] (input10) {$a_{21}$}; 
    & \node[layer, fill=gray!80] (input11) {\textcolor{blue}{$a_{22}$}}; 
    & \node[layer, fill=gray!50] (input12) {$a_{23}$}; 
    & \node[layer] (input13) {\textcolor{blue}{$a_{24}$}}; 
    & \node[layer] (input03) {$a_{25}$};\\
    \node[layer, fill=gray!50] (input20) {$a_{31}$}; 
    & \node[layer, fill=gray!50] (input21) {$a_{32}$}; 
    & \node[layer, fill=gray!50] (input22) {$a_{33}$}; 
    & \node[layer] (input23) {$a_{34}$}; 
    & \node[layer] (input03) {$a_{35}$};\\
    \node[layer] (input30) {$a_{41}$}; 
    & \node[layer] (input31) {\textcolor{blue}{$a_{42}$}}; 
    & \node[layer] (input32) {$a_{43}$}; 
    & \node[layer] (input33) {\textcolor{blue}{$a_{44}$}}; 
    & \node[layer] (input03) {$a_{45}$};\\
    \node[layer] (input30) {$a_{51}$}; 
    & \node[layer] (input31) {$a_{52}$}; 
    & \node[layer] (input32) {$a_{53}$}; 
    & \node[layer] (input33) {$a_{54}$}; 
    & \node[layer] (input03) {$a_{55}$};\\
  };
  \node (label) [above=0cm of inputMatrix] {$2$ strided convolution: Apply };
  \node[right=-.25cm of label] {\vspace{1ex} the kernel to every other input pixel.};
    \node (conv) [right=of inputMatrix, matrix anchor=west, xshift=-1cm] {$\ast^{2\text{s}}$};
  \matrix (kernelMatrix) [right=of inputMatrix, matrix anchor=west, xshift=-.4cm] {
    \node[kernel, fill=gray!50] (kernel00) {$k_{11}$}; 
    & \node[kernel, fill=gray!50] (kernel01) {$k_{12}$}; 
    & \node[kernel, fill=gray!50] (kernel02) {$k_{13}$}; \\
    \node[kernel, fill=gray!50] (kernel10) {$k_{21}$}; 
    & \node[kernel, fill=gray!80] (kernel11) {$k_{22}$}; 
    & \node[kernel, fill=gray!50] (kernel12) {$k_{23}$}; \\
    \node[kernel, fill=gray!50] (kernel20) {$k_{31}$}; 
    & \node[kernel, fill=gray!50] (kernel21) {$k_{32}$}; 
    & \node[kernel, fill=gray!50] (kernel22) {$k_{33}$}; \\
  };
    \node (equal) [right=of kernelMatrix, matrix anchor=west, xshift=-1cm] {$=$};
  \matrix (outputMatrix) [right=of kernelMatrix, matrix anchor=west, xshift=-.5cm] {
    \node[layer, fill=gray!80] (input00) {$\sum_{j\in U_1} \textcolor{blue}{a}_{\textcolor{blue}{22}+j}k_{22+j}$}; 
    & \node[layer] (input01) {$\sum_{j\in U_1} \textcolor{blue}{a}_{\textcolor{blue}{24}+j}k_{22+j}$};  \\
    \node[layer] (input10) {$\sum_{j\in U_1} \textcolor{blue}{a}_{\textcolor{blue}{42}+j}k_{22+j}$}; 
    & \node[layer] (input11) {$\sum_{j\in U_1} \textcolor{blue}{a}_{\textcolor{blue}{44}+j}k_{22+j}$};  \\
  };

  \matrix (inputMatrix) [below=1cm of inputMatrix] {
    \node[layer, fill=gray!50] (input00) {$a_{11}$}; 
    & \node[layer, fill=gray!50] (input01) {$0$}; 
    & \node[layer, fill=gray!50] (input02) {$0$}; 
    & \node[layer] (input03) {$0$};
    & \node[layer] (input03) {$0$}; \\
    \node[layer, fill=gray!50] (input10) {$0$}; 
    & \node[layer, fill=gray!80] (input11) {$\textcolor{blue}{a_{22}}$}; 
    & \node[layer, fill=gray!50] (input12) {$0$}; 
    & \node[layer] (input13) {$0$}; 
    & \node[layer] (input03) {$0$};\\
    \node[layer, fill=gray!50] (input20) {$0$}; 
    & \node[layer, fill=gray!50] (input21) {\textcolor{blue}{$a_{32}$}}; 
    & \node[layer, fill=gray!50] (input22) {\textcolor{blue}{$a_{33}$}}; 
    & \node[layer] (input23) {$0$}; 
    & \node[layer] (input03) {$0$};\\
    \node[layer] (input30) {$0$}; 
    & \node[layer] (input31) {$0$}; 
    & \node[layer] (input32) {\textcolor{blue}{$a_{43}$}}; 
    & \node[layer] (input33) {$0$}; 
    & \node[layer] (input03) {$0$};\\
    \node[layer] (input30) {$0$}; 
    & \node[layer] (input31) {$0$}; 
    & \node[layer] (input32) {$0$}; 
    & \node[layer] (input33) {$0$}; 
    & \node[layer] (input03) {$0$};\\
  };
  \node (label) [above=0cm of inputMatrix] {Submanifold sparse convol};
    \node (conv) [right=of inputMatrix, matrix anchor=west, xshift=-1cm] {$\ast^{\text{sp}}$};
  \node[right=-.25cm of label] {ution: Apply the kernel only to nonzero elements keeping the shape of the vanilla convolution.};
    \node (conv) [right=of inputMatrix, matrix anchor=west, xshift=-1cm] {$\ast^{\text{sp}}$};
  \matrix (kernelMatrix) [right=of inputMatrix, matrix anchor=west, xshift=-.4cm] {
    \node[kernel, fill=gray!50] (kernel00) {$k_{11}$}; 
    & \node[kernel, fill=gray!50] (kernel01) {$k_{12}$}; 
    & \node[kernel, fill=gray!50] (kernel02) {$k_{13}$}; \\
    \node[kernel, fill=gray!50] (kernel10) {$k_{21}$}; 
    & \node[kernel, fill=gray!80] (kernel11) {$k_{22}$}; 
    & \node[kernel, fill=gray!50] (kernel12) {$k_{23}$}; \\
    \node[kernel, fill=gray!50] (kernel20) {$k_{31}$}; 
    & \node[kernel, fill=gray!50] (kernel21) {$k_{32}$}; 
    & \node[kernel, fill=gray!50] (kernel22) {$k_{33}$}; \\
  };
    \node (equal) [right=of kernelMatrix, matrix anchor=west, xshift=-1cm] {$=$};
  \matrix (outputMatrix) [right=of kernelMatrix, matrix anchor=west, xshift=-.5cm] {
    \node[layer, fill=gray!80] (input00) {$\sum_{j\in U_1} \textcolor{blue}{a}_{\textcolor{blue}{22}+j}k_{22+j}$}; 
    & \node[layer,minimum height=.75cm,minimum width=2.99cm] (input01) {$ 0 $};
    & \node[layer,minimum height=.75cm,minimum width=2cm] (input01) {$0$};  \\
    \node[layer,minimum height=.75cm,minimum width=2.99cm] (input10) {$\sum_{j\in U_1} \textcolor{blue}{a}_{\textcolor{blue}{32}+j}k_{22+j}$}; 
    & \node[layer,minimum height=.75cm,minimum width=2.99cm] (input11) {$\sum_{j\in U_1} \textcolor{blue}{a}_{\textcolor{blue}{33}+j}k_{22+j}$};
    & \node[layer,minimum height=.75cm,minimum width=2cm] (input01) {$0$};  \\
    \node[layer,minimum height=.75cm,minimum width=2.99cm] (input10) {$0$}; 
    & \node[layer] (input11) {$\sum_{j\in U_1} \textcolor{blue}{a}_{\textcolor{blue}{43}+j}k_{22+j}$};
    & \node[layer,minimum height=.75cm,minimum width=2cm] (input01) {$0$};  \\
  };
\end{tikzpicture}
    \caption{Visualization of the vanilla, the $2$ strided, the $2$ transpose strided and the submanifold sparse convolution, where $U \coloneqq \{(i,j): \abs{i}\leq \nicefrac{W-1}2,\abs{j}\leq \nicefrac{H-1}2\}$, where $W$ and $H$ denote the uneven width and height of the kernel respectively.}
    \label{fig: convolutions}
\end{figure}

\section{Proofs of CNN approximation theorems}

\begin{corollary}[Multiplication approximation {\cite[Corollary 13]{cosi}}]
    Let $\sigma$ satisfy \cref{ass: activation function}. Let $W\in\mathbb{N}$ be the input image size and $B>0$ the range of the input values and $\varepsilon\in (0,\nicefrac12)$. Then there exists a CNN $\Psi$ with activation function $\sigma$, two-channel input and one channel output, spatial dimension of the kernels $1$, $2$ layers and number of parameters at most $9$ such that
    \begin{align*}
        \norm{\Psi(\bfx,\bfy) - \bfx\odot\bfy}_{L^\infty([-B,B]^{2\times W\times W})}\leq \varepsilon.
    \end{align*}
\end{corollary}

\begin{lemma}[Concatenation approximation {\cite[Lemma 20]{cosi}}]\label{lem: CNN concatenation approximation}
    Let $n, d_1,\dots,d_{n+1}\in\mathbb{N}$, $i\in [n]$ and $f_i:\mathbb{R}^{d_i}\to\mathbb{R}^{d_{i+1}}$ be continuous and let $F:\mathbb{R}^{d_1}\to\mathbb{R}^{d_{n+1}}$ be the concatenation $F\coloneqq f_n \circ \dots \circ f_1$. Let $M,\varepsilon>0$. Then there exists $\tilde{M},\tilde{\varepsilon}>0$ such that $\norm{f_i - \tilde{f}_i}_{L^\infty([-\tilde{M},\tilde{M}]^{d_i})} \leq \tilde{\varepsilon}$ for each $i\in [n]$ and some $\tilde{f}_i:\mathbb{R}^{d_i}\to\mathbb{R}^{d_{i+1}}$ implies
    \begin{align*}
        \norm{ F - \tilde{f}_n\circ \dots \circ \tilde{f}_1}_{L^\infty([-M,M]^{d_1})}\leq\varepsilon.
    \end{align*}
\end{lemma}

\subsection{Solver approximation}\label{section: Solver approx}

Using the notation from \cite[Definition 14]{cosi}, for $\kappa\in H_0^1(D)$ $k=1,\dots,L$, $\ell=1,\dots,6$ and $i\in\I_U^k$ define
\begin{align*}
    \Upsilon (\kappa, \mathcal{T}_k, \ell, i) \coloneqq \int_{T_i^\ell}\kappa \dx, \quad \Upsilon(\kappa,\T_k,\ell) \coloneqq (\Upsilon (\kappa, \mathcal{T}_k, \ell, i))_{i\in\I_U^k} \quad \text{and} \quad \Upsilon(\kappa,\T_k) \coloneqq (\Upsilon (\kappa, \mathcal{T}_k, \ell))_{\ell=1,\dots,6}.
\end{align*}

In~\cite[Theorem 16]{cosi} it was shown that $\left(A^k_{\bfy} \bfu^k\right)_\text{img}$ can be written as an application of a kernel to $\uimk$ and a multiplication with $\Upsilon(\kappa_h(\cdot,\bfy), \T_k)$, where $A_\bfy^k$ only considers indices in $\I_V^k\times\I_V^k$.
We generalize the previous result to the application of $\Aab{k}$ with indices in $\I_V^k\times \Iab{k}$ by the use of multiple channels of in the following theorem. 

\begin{definition}[Translation]\label{def: translation}
    Let $m\in\mathbb{N}$ be the number of basis functions with overlapping support, i.e., for $i$ the index of an inner node let $m \coloneqq \abs{\{ \varphi_j^k : \sup \varphi_i^k\cap \sup \varphi_j^k \neq \emptyset \}}$.
    For $\vimk\in\mathbb{R}^{\I_U^k}$, let $\tran\vimk\in\mathbb{R}^{m\times \I_U^k}$ be defined by
    \begin{align*}
        \tran\vimk \coloneq \left[ \tran^{(1)}\vimk, \dots, \tran^{(m)}\vimk \right],
    \end{align*}
    where $T^{(1)}, \dots, T^{(m)}$ defines the translation such that for $i\in\I_U^k$
    \begin{align*}
        (\tran\vimk)_i = \left[(\vimk)_{i+p_1},\dots,(\vimk)_{i+p_m}\right],
    \end{align*}
    where $i+p_1, \dots, i+p_m$ denote the indices of the basis functions with overlapping support and $p_1,\dots,p_m$ denote the directions with $p_1=0$.
    Note that the directions $p_1,\dots,p_m$ are constant for all $i$ due to the used uniformly refined meshes as depicted in the first row in \Cref{fig: uniform vs local meshes}.
    Note that for the depicted meshes, $m=7$ holds true and is independent of $k$.
\end{definition}

\begin{figure}
    \centering
    \includesvg[width=\linewidth]{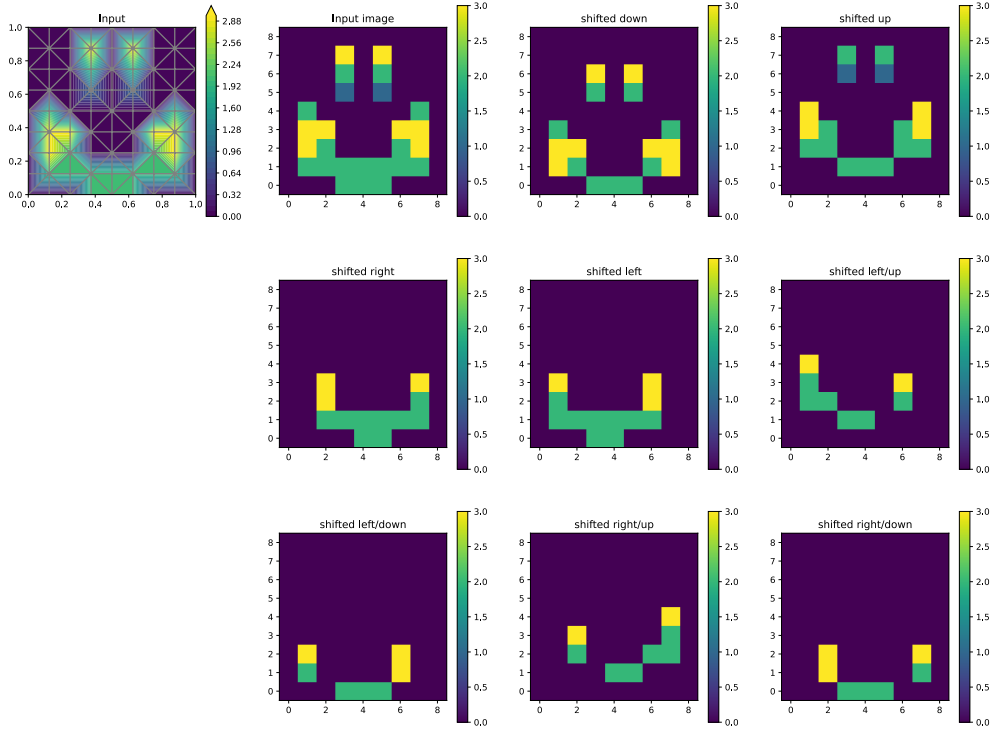}
    \caption{This is a visualization of the translation~\cref{def: translation}. On the left the corresponding function to some input image is plotted. Since the mesh is also plotted, it can be seen that each nodal hat function has an overlapping support with nine other nodal hat functions including itself, i.e. $m=9$. The output of the translation is a stack of the other nine images plotted. Each image shows the original image shifted in the direction of the nodal hat functions with overlapping support, but only on the support of the input image. This can be interpreted as shifting the image and then multiplying with a mask. Fixing one pixel index, the values of the pixel and all surrounding pixels are now saved in the pixel index in different images instead.}
    \label{fig:enter-label}
\end{figure}

Based on the definition of the translated images, the following theorem can be formulated.

\begin{theorem}\label{thm: F}
    Let $m\in\mathbb{N}$ be as in \cref{def: translation}. There exist kernels $K^{(\ell)}\in \mathbb{R}^{1\times m\times 1\times 1}, \ell=1,\dots,6$ such that for
\begin{align*}
    F^k:\mathbb{R}^{(m+7)\times \I_V^k} \to \mathbb{R}^{\I_V^k}, \quad \left(\vimk, \bar\bfk^{(1)}, \dots, \bar\bfk^{(6)}\right) \mapsto \sum_{\ell=1}^6 \bar\bfk^{(\ell)} \odot (\vimk \spco K^{(\ell)}),
\end{align*}
for $\bfv^k\in\mathbb{R}^{\Iab{k}}$, $M^k = 1\in\mathbb{R}^{\I_V^k}$ it holds that
\begin{align*}
    F^k\left(M^k_{\text{img}} \odot T\vimk, \Upsilon (\kappa_h(\cdot,\bfy), \T_k)\right) = \left(\Aab{k} \hspace{.5ex} \vab{\empty}k\right)_\text{img}.
\end{align*}
\end{theorem}

Note that since $\left(\Aab{k} \hspace{.5ex} \vab{\empty}k\right)_\text{img}$ is zero at every index $i\in \I_U^k\backslash\I_V^k$, in contrast to the definition in \cite{cosi} the convolution here only needs to be applied to nodes $i\in\I_V^k$ since otherwise zero nodes would be assigned nonzero values after one convolution. This submanifold sparse convolution is denoted by $\spco$ here.

\begin{proof}
    First, let $i\in\I_U^k\backslash\I_V^k$.
    Since $\Aab{k}:\mathbb{R}^{\Iab{k}}\to\mathbb{R}^{\I_V^k}$ only maps to indices in $\I_V^k$, it holds that ${(\Aab{k} \hspace{.5ex} \vab{\empty}{k})_\text{img}}_{i} = 0$.
    Furthermore, for $\bfw \coloneqq M^k_{\text{img}} \odot T\vimk$ we have that $\bfw_i = 0$ and $\spco$ only acts on the indices $\I_V^k$, leaving everything else at $0$. Therefore,
    \begin{align*}
        {(\Aab{k} \hspace{.5ex}\vab{\empty}{k})_\text{img}}_{i} = 0 = \sum_{\ell=1}^6 \bar\bfk^{(\ell)}_i \cdot 0 = \sum_{\ell=1}^6 \bar\bfk^{(\ell)}_i (\bfw \spco K^{(\ell)})_i.
    \end{align*}
    Second, for $i\in\I_V^k$ and $j\in\Iab{k}$ as in \cite[proof of Theorem 16]{cosi} with $C_{ijk} \coloneqq \int_{T_i^k} \langle \nabla\varphi_i, \nabla \varphi_j \rangle dx$ it holds that
    \begin{align*}
        (\Aab{k})_{ij} = \int_D \kappa_h(\cdot,\bfy) \langle \nabla\varphi_i, \nabla \varphi_j \rangle dx = \sum_{\ell=1}^6 \int_{T_i^\ell}\kappa_h(\cdot,\bfy) \langle \nabla\varphi_i, \nabla \varphi_j \rangle dx = \sum_{\ell=1}^6 \Upsilon(\kappa_h(\cdot,\bfy),\T_k,\ell,i) C_{ij\ell}.
    \end{align*}
    Since $C_{ijk}=0$ if $\varphi_i^k\cap\varphi_j^k \neq \emptyset$
    , for the corrections $p_1,\dots,p_m$ defined in \cref{def: translation} we get that
    \begin{align*}
        {(\Aab{k} \bfv^k)_\text{img}}_{i} &= \sum_{j\in\Iab{k}} \bfv^k_j  \sum_{\ell=1}^6 \Upsilon(\kappa_h(\cdot,\bfy),\T_k,\ell,i) C_{ij\ell}
        = \sum_{\ell=1}^6 \Upsilon(\kappa_h(\cdot,\bfy),\T_k,\ell,i)\sum_{j\in\Iab{k}} \bfv^k_j   C_{ij\ell}\\
        &= \sum_{\ell=1}^6 \Upsilon(\kappa_h(\cdot,\bfy),\T_k,\ell,i)\sum_{t=1}^m \bfv^k_{i+p_t}   C_{i,i+p_t,\ell}
        = \sum_{\ell=1}^6 \Upsilon(\kappa_h(\cdot,\bfy),\T_k,\ell,i)\sum_{t=1}^m \bfw_{t,i}   C_{i,i+p_t,\ell}.
    \end{align*}
    Since $C_{ijk}$ only depends on the difference $i-j$, for each $\ell$ the inner sum can be expressed with the same constant for every $i\in\I_V^k$ by the convolution with a $m\times 1\times 1$ kernel $K^{(\ell)}$, i.e.,
    \begin{align*}
        (\Aab{k} \bfv^k)_i &= \sum_{\ell=1}^6 \Upsilon(\kappa_h(\cdot,\bfy),\T_k,\ell,i)\sum_{t=1}^m \bfw_{t,i}   C_{1,1+p_t,\ell}\\
        &=  \sum_{\ell=1}^6 \Upsilon(\kappa_h(\cdot,\bfy),\T_k,\ell,i) (\bfw \spco K^{(\ell)})_i = F^k(\bfw, \Upsilon (\kappa_h(\cdot,\bfy), \T_k)).
    \end{align*}
\end{proof}
Consequently, for $\kappa^k_h\in V^k$ we have that
\begin{align*}
    \Upsilon (\kappa^k_h, \mathcal{T}_k, \ell, i)= \sum_{\{j:\supp \varphi_j^k \cap T_i^\ell \not= \emptyset\}} \bfk^k_j \frac{h^2}{3}.
\end{align*}

\begin{theorem}\label{thm: F approx}
    For every $\varepsilon, M>0$ there exists a CNN $\Psi: \mathbb{R}^{7\times \I_U^k 
    } \to \mathbb{R}^{\I_U^k}
    $ of constant size consisting of submanifold sparse convolutions such that
    \begin{align*}
        \norm{\Psi - F}_{L^\infty\left([-M,M]^{7 \times \I_U^k 
        }\right)} \leq \varepsilon,
    \end{align*}
    where $F(M^k_{\text{img}} \odot T\vimk, \Upsilon(\kappa ,\mathcal{T})) = (\Aab{k} \hspace{.5ex} \vab{\empty}k)_{\text{img}}$.
\end{theorem}
\begin{proof}
    Exchanging the convolution with the submanifold sparse convolution $\spco$, the proof works similarly to the proof of \cite[Theorem 18]{cosi} in three steps with the difference that the kernels have width $1$ but $m$ channels opposed to one channel and width $3$.
    This is due to the fact that the input images are one image translated in space in different directions $m$ times to account for surrounding information in each node $i\in\I_V^k$ in multiple channels instead of in the surrounding nodes $j\in\Iab{k}$ close to $i$.
    In this way, the sparse convolution $\spco$ can act on $i\in\I_V^k$ without losing information.
    In the first step there exists a one-layer CNN realizing the mapping
    \begin{align*}
        \left(\bfw, \bar\bfk^{(1)}, \dots, \bar\bfk^{(6)}\right) \mapsto \left( \bfw \spco K^{(\ell)}, \bfk^{(\ell)} \right)_{\ell=1}^6.
    \end{align*}
    The pointwise multiplication in the second step can be approximated by a CNN of constant size arbitrarily well.
    Moreover, the addition of the channels in the third step can be realized by a one-layer CNN with a $1\times 1$ kernel as described in the proof of \cite[Theorem 18]{cosi}.
    Concatenating these CNNs yields the claim.
\end{proof}

In a similar way the following theorem can be proven.
\begin{theorem}\label{thm: FT}
    Let $m\in\mathbb{N}$ be as in \cref{def: translation}. There exist kernels $K^{(\ell)}\in \mathbb{R}^{1\times m\times 1\times 1}, \ell=1,\dots,6$ such that for
\begin{align*}
    \FkT:\mathbb{R}^{(m+7)\times \I_V^k} \to \mathbb{R}^{\I_V^k}, \quad \left(\vimk, \bar\bfk^{(1)}, \dots, \bar\bfk^{(6)}\right) \mapsto \sum_{\ell=1}^6 \bar\bfk^{(\ell)} \odot (\vimk \spco K^{(\ell)})
\end{align*}
for $\bfv^k\in\mathbb{R}^{\Iab{k}}$, $M^k = 1\in\mathbb{R}^{\I_V^k}$, it holds that
\begin{align*}
    \FkT\left(M^k_{\text{img}} \odot T\vimk, \Upsilon (\kappa_h(\cdot,\bfy), \T_k)\right) = M^k \odot T \left(\Aab{k}^\intercal  \vab{\empty}k\right)_\text{img}.
\end{align*}
Furthermore, $\FkT$ can be approximated arbitrarily well by a CNN with submanifold sparse convolutions of constant size.
\end{theorem}

\begin{remark}[CNN for prolongation and weighted restriction]\label{rem: CNN prolongation and weighted restriction}
    As noted in \cite[Remark 19]{cosi}, the prolongation and weighted restriction can be represented by the application of a $2$ strided convolution to the whole image.
    We now argue that the kernel can also be applied to nonzero entries as in $\spco$ and still be able to represent the operators.
    \begin{enumerate}
        \item Weighted restriction: When applying the kernel only to every other entry, which is nonzero in level $k$, i.e. $\I_V^k$ instead of every other entry in $\I^k_U$, output values can be set to zero, which would include values of entries between the entries used for the convolution. 
        To account for this error, the translation is applied to the image $T\vimk$ and the operations are applied to the original image on all nonzero indices of the translation.

        \item Prolongation: The prolongation can be represented as in \cite[Remark 19]{cosi}, only acting on nonzero entries of the input images and multiplied with a mask, which is $1$ for entries in $\Iab{k}$ and $0$ otherwise.
    \end{enumerate}
    
\end{remark}

\begin{proof}[Proof of \cref{thm: LLMG_Approx}] 
Let $\bfk_\bfy$ be the coefficient image of the interpolation of $\kappa_h(\cdot,\bfy)$ in $U^L$.
The proof works similarly to the proof of \cite[Theorem 6]{cosi}. We write the levelwise local multigrid \cref{alg: LLMG} $\mathrm{LLMG}: \bigtimes_{k=1}^L \mathbb{R}^{10\times \I_U^k} \to \bigtimes_{k=1}^L \mathbb{R}^{\I_U^k}$  as the concatenation of functions, which can be represented or arbitrarily approximated by CNNs.
To simplify notation, denote the masked translations as in \cref{def: translation} by $\tbfw \coloneqq M^k_{\text{img}}\odot T\utim^k$ and $\bbfw \coloneqq M^k_{\text{img}}\odot T\ubim^k$, where the mask $M^k \coloneqq 1\in\mathbb{R}^{\I_V^k}$ is applied to every channel.
\begin{enumerate}[label=(\roman*)]
    \item \textbf{Integrating the diffusion coefficient.}
    Let $K\in \mathbb{R}^{1\times 6 \times 3\times 3}$ be defined as in \cite[Lemma 15(i)]{cosi} and define
    \begin{align*}
        f_{in}: \mathbb{R}^{2\times \I_U^L} \to \mathbb{R}^{7\times \I_U^L}, \quad \begin{pmatrix}
            \bfk \\ \bff
        \end{pmatrix} \mapsto \begin{pmatrix}
            \bfk \ast K\\ \bff
        \end{pmatrix}.
    \end{align*}
    Then $f_{in}([\bfk_\bfy,\bff]) = [\Upsilon(\kappa_h(\cdot,\bfy), \T_L),\bff]$. 
    
    \item \textbf{Smoothing iteration (\cref{alg: LLMG: first update}, \cref{alg: LLMG: second update} in \cref{alg: LLMG}).}
    For each level $k=1,\dots,L$, define the smoothing function 
    $$
    f_{\text{sm}}^k: \bigtimes_{\ell=1}^k \mathbb{R}^{(4m+7)\times \I_U^\ell} \to \bigtimes_{\ell=1}^k \mathbb{R}^{(4m+7)\times \I_U^\ell}
    $$ 
    by its action on the level $k$ input images in $\mathbb{R}^{(4m+7)\times \I_U^k}$ with $\bfv,\tbfv,\bbfv \in \mathbb{R}^{m\times \I_U^k},\bff\in\mathbb{R}^{\I_U^k}, \bar\bfk\in\mathbb{R}^{6\times \I_U^k}$
    \begin{align*}
        \begin{pmatrix}
            \bfv\\ \tbfv\\ \bbfv\\0 \\ \bar\bfk\\ \bff
        \end{pmatrix} \mapsto \begin{pmatrix}
            \bfv + \omega (\bff - [F^k(\bfv + \tbfv, \bar\bfk) + \bbfv])\\ \tbfv\\ \bbfv\\0 \\ \bar\bfk\\ \bff
        \end{pmatrix}.
    \end{align*}
    Except for the operation on $\bfv$, the other inputs are directly passed to the output.
    Then, for any $\bfv\in \bigtimes_{\ell=1}^{k-1} \mathbb{R}^{(4m+7) \times \I_U^\ell}$ \cref{thm: F} and since $\tbfw_{\ell i}=0$ for $\varphi_{i+p_\ell}^k\notin V^k$,
    \begin{align*}
        &f_{\text{sm}}^k\left(\begin{bmatrix}
            \tran\uimk, \tbfw, \bbfw, 0, \Upsilon(\kappa_h(\cdot,\bfy),\T_k),  \bff\\
            \bfv
        \end{bmatrix}\right)\\
        &= \begin{bmatrix}
            \tran(\bfu^k + \omega (\bff - [\Aab{k}(\uab{\empty}k+\tbfu^k) + \bbfu^k|_{\I_V^k}]))_{\text{img}}, \tbfw, \bbfw, 0, \Upsilon(\kappa_h(\cdot,\bfy),\T_k),  \bff\\
            \bfv
        \end{bmatrix}.
    \end{align*}
    In \cref{thm: F approx} it is shown that this operation can be approximated arbitrarily well by a CNN using submanifold sparse convolutions $\spco$ on level $k$.
    
    \item \textbf{Update of $\bbfu$ and restriction (\cref{alg: LLMG: update u bar} in \cref{alg: LLMG}).}
    We also define the update $\bar{\bfu}^k$ and the restriction to the coarser level
    $$
    f_{\text{upd}}^k, f_{\text{rest}}: \bigtimes_{\ell=1}^k \mathbb{R}^{ (4m+7)\times \I_U^\ell} \to \bigtimes_{\ell=1}^k \mathbb{R}^{ (4m+7)\times \I_U^\ell}.
    $$
    The update function $f^k_{\text{upd}}$ is defined by its action on the level $k$ input images
    \begin{align*}
        \begin{pmatrix}
            \bfv\\ \tbfv\\ \bbfv \\ 0\\ \bar\bfk\\ \bff
        \end{pmatrix} \mapsto \begin{pmatrix}
            \bfv \\ \tbfv\\ \bbfv\\ \bbfv + \FkT (\bfv, \bar\bfk) \\ \bar\bfk\\ \bff
        \end{pmatrix},
    \end{align*}
    where again the other inputs are passed to the output as they are. Then, with \cref{thm: FT} it holds that
    \begin{align*}
        f_{\text{upd}}^k\left(\begin{bmatrix}
            \uimk, \tbfw,  \bbfw, 0, \bar\bfk,  \bff\\
            \bfv
        \end{bmatrix}\right) = \begin{bmatrix}
            \uimk , \tbfw,\bbfw,  M^k_{\text{img}} \odot T (\bbfu^k + \Aab{k}^\intercal\uab{\empty}k)_{\text{img}},  \bar\bfk,  \bff\\
            \bfv
        \end{bmatrix}.
    \end{align*}
    The operation can be approximated by a CNN due to \cref{thm: FT}.
    Furthermore, we define the restriction $f_{\text{rest}}$ 
    by its action on the inputs on level $k$ and $k-1$ by
    \begin{align*}
        \begin{pmatrix}
            \bfv^k\\ \tbfv^k\\ \bbfv^k \\ \bfz^k\\ \bar\bfk^k\\ \bff^k
        \end{pmatrix} \times \begin{pmatrix}
            \bfv^{k-1}\\ \tbfv^{k-1}\\ \bbfv^{k-1} \\0\\ \bar\bfk^{k-1}\\ \bff^{k-1}
        \end{pmatrix}
        \mapsto 
        \begin{pmatrix}
            \bfv^{k-1} \\ \tbfv^{k-1}\\ P^\intercal_{k-1}\bfz^k\\0 \\ \bar\bfk^{k-1}\\ \bff^{k-1}
        \end{pmatrix},
    \end{align*}
    where the weighted restriction $P^\intercal$ is applied to every image in $\bfv^k$. The first to $(k-2)$th inputs are passed to the output unaltered.
    Then, for any $\bfv,\tbfv,\bbfv\in\mathbb{R}^{m\times \I_U^k},\bff\in\mathbb{R}^{\I_U^{k-1}}, \bar\bfk\in \mathbb{R}^{6\times\I_U^{k-1}}, \bfz\in\mathbb{R}^{\bigtimes_{\ell=1}^{k-2}(4m+7)\times \I_U^\ell}$
    \begin{align*}
        (f_{\text{rest}}^k\circ f_{\text{upd}}^k) \left(\begin{bmatrix}
            T\uimk, \tbfw, \bbfw, 0, \bar\bfk^k, \bff^k\\
            \bfv, \tbfv, \bbfv, 0, \bar\bfk, \bff\\
            \bfz
        \end{bmatrix}\right)
        = \begin{bmatrix}
            \bfv, \tbfv, \ubim^{k-1}, 0, \bar\bfk, \bff\\
            \bfz
        \end{bmatrix}.
    \end{align*}
    Due to \cref{rem: CNN prolongation and weighted restriction}, $f_{\text{rest}}^k$ can be represented by a CNN.
    
    \item \textbf{Update $\tbfu$ with coarse grid solution and prolongation (\cref{alg: LLMG: update u tilde} in \cref{alg: LLMG}).}
    For the last recursion step, define the prolongation to update the auxiliary vector $\bbfu$
    \begin{align*}
        f^k_{\text{prol}}: \bigtimes_{\ell=1}^k \mathbb{R}^{(4m+7)\times \I_U^\ell} \to \bigtimes_{\ell=1}^k \mathbb{R}^{(4m+7)\times \I_U^\ell}
    \end{align*}
    by its action on input functions from level $k$ and $k-1$
    \begin{align*}
        \begin{pmatrix}
            \bfv^k\\ \tbfv^k\\ \bbfv^k \\ \bfz^k\\ \bar\bfk^k\\ \bff^k
        \end{pmatrix} \times \begin{pmatrix}
            \bfv^{k-1}\\ \tbfv^{k-1}\\ \bbfv^{k-1} \\ \bfz^{k-1}\\ \bar\bfk^{k-1}\\ \bff^{k-1}
        \end{pmatrix}
        \mapsto 
        \begin{pmatrix}
            \bfv^{k} \\ P_{k-1}(\tbfv^{k-1} + \bfv^{k-1})\\ \bbfv^k\\ \bfz^k \\ \bar\bfk^{k}\\ \bff^{k}
        \end{pmatrix} \times 
        \begin{pmatrix}
            \bfv^{k-1}\\ \tbfv^{k-1}\\ \bbfv^{k-1} \\ \bfz^{k-1}\\ \bar\bfk^{k-1}\\ \bff^{k-1}
        \end{pmatrix},
    \end{align*}
    where the prolongation is only applied to indices in $\Iab{k-1}$.
    For any $(\bfv^k, \tbfv^{k}, \bbfv^k, \bfz^k, \bar\bfk^k, \bff^k) \in\mathbb{R}^{(4m+7)\times\I_U^k}, \bfv^{k-1},\bfz^{k-1}\in\mathbb{R}^{m\times\I_U^{k-1}}, \bar\bfk^{k-1}\in\mathbb{R}^{6\times\I_U^{k-1}}, \bff^{k-1}\in\mathbb{R}^{\I_U^{k-1}}$ and $z\in\mathbb{R}^{\bigtimes_{\ell=1}^{k-2}(4m+7)\times\I^\ell_U}$ it holds that
    \begin{align*}
        f_{\text{prol}}^{k}\left(\begin{bmatrix}
            \bfv^k, \tbfv^{k}, \bbfv^k, \bfz^k, \bar\bfk^k, \bff^k\\
            \bfv^{k-1}, \tilde\bfw^{k-1}, \bar\bfw^{k-1}, \bfz^{k-1}, \bar\bfk^{k-1}, \bff^{k-1}\\
            \bfz
        \end{bmatrix}\right) = 
        \begin{bmatrix}
            \bfv^k, \tbfu^k, \bbfv^k, \bfz^k, \bar\bfk^k, \bff^k\\
            \bfv^{k-1}, \tilde\bfw^{k-1}, \bar\bfw^{k-1}, \bfz^{k-1}, \bar\bfk, \bff\\
            \bfz
        \end{bmatrix}.
    \end{align*}
    With \cref{rem: CNN prolongation and weighted restriction} it can be seen that this operation can be represented by a CNN.
    
    \item \textbf{Return solution.}
    The last required function is the output function
    $$f_{\text{out}}: \bigtimes_{\ell=1}^k \mathbb{R}^{(4m+7)\times \I_U^\ell} \to \bigtimes_{\ell=1}^k \mathbb{R}^{\I_U^\ell} 
    $$
    defined for each level $k=1,\dots, L$ by
    \begin{align*}
        \begin{pmatrix}
            \bfv^k\\ \tbfv^k\\ \bbfv^k \\ \bfz^k\\ \bfk^k\\ \bff^k
        \end{pmatrix} \mapsto \begin{pmatrix}
            \bfv^k_0
        \end{pmatrix},
    \end{align*}
    where $\bfv_0^k\in\mathbb{R}^{\I_U^k}$ denotes the first image of $\bfv^k\in\mathbb{R}^{m\times \I_U^k}$.
    Then, for $\bfu\in\mathbb{R}^{\bigtimes_{k=1}^L\I_V^k}, \bfv\in\mathbb{R}^{(2m + 8)\times\I_U^k}$ $f_{\text{out}}(\bigtimes_{k=1}^L T\bfu^k_{img} \times \bfv^k) = \bigtimes_{k=1}^L \bfu^k_{\text{img}}$.
\end{enumerate}
Eventually combining the algorithmic components described above, the $\mathrm{LLMG}$ for $k$ levels can be expressed as
\begin{align*}
    \mathrm{LMGV}^k &= f^k_{\text{sm}} \circ f^k_{\text{prol}} \circ (Id_k, LMG^{k-1}) \circ f_{\text{rest}}^k \circ f^k_{\text{upd}} \circ f^k_{\text{sm}}\\
    \mathrm{LMGV}^1 &= f_{\text{sm}}^1,
\end{align*}
where $Id_k$ passes the the inputs on level $k,\dots,L$ to the output and $\mathrm{LLMG}^m = f_{\text{out}} \circ (\bigcirc_{i=1}^m \mathrm{LMGV}^L) \circ f_{\text{in,res}} \circ f_{\text{in}}$.
Here,
\begin{align*}
    f_{in,res}(\Upsilon(\kappa_y(\cdot,\bfy),\T_L), \bff) = \begin{bmatrix}
        0,0,0,0,\bar\kappa^k
    \end{bmatrix} .
\end{align*}
Since every component can be approximated by a CNN with constant size, \cref{lem: CNN concatenation approximation} implies that the whole algorithm can be approximated by a CNN with constant size.
\end{proof}

\subsection{Estimator approximation}
The strong residual images and jump images can be approximated on any level by a CNN.
\begin{theorem}\label{thm: CNN for error estimator one level}
For any $k\in[L]$ and $\varepsilon >0$ there exists a CNN $\Psi$ 
such that for all $\kimk,\fimk,\uimk\in [-M,M]^{\I_U^k}$
\begin{align*}
    \norm{\Psi(\uimk,\kimk_h,\fimk) - \left(\stre^2_{k,T^q},\jure^2_{k,T^q} \right)_{q=1,2}}_\infty \leq \varepsilon,
\end{align*}
where the strong residual images and jump images are defined as in \cref{def: estimator images} with respect to the input coefficients.
For $\Psi$, there exist three fixed bias and kernel sizes with width and height at most $5$.
\end{theorem}
\begin{proof}
We first note that for $q=1,2$ and each triangle $T = T^q_{k,i}$ as illustrated in \Cref{fig: estimator triangles},
\begin{align*}
    (\stre^2_{k,T^q})_{i} = &\norm{f_h + \divg (\kappa_h \nabla u_h)}_{L^2(T)}^2 \\
    = &\norm{f_h}^2_{L^2(T)} + 2\scpr{ f_h, \divg (\kappa_h \nabla u_h) }_{L^2(T)} + \norm{\divg (\kappa_h \nabla u_h)}^2_{L^2(T)}.
\end{align*}
Here,
\begin{align*}
    \norm{f_h}_{L^2(T)}^2&= \sum_{m,n\in \no{T}} \bff_{m}\bff_{n} e_{mn} && \text{for} \quad e_{mn} \coloneqq \int_T \varphi_m\varphi_n \mathrm{d}x,\\
    \scpr{f_h, \divg (\kappa_h \nabla u_h)}_{L^2(T)}&= \sum_{j\in \no{T}} \bff_j c_j \sum_{m,n\in \no{T}} \bfk_m\bfu_n  d_{mn} &&\text{for}\quad c_j \coloneqq \int_T{\varphi_j}\mathrm{d}x, \quad d_{mn} \coloneqq \nabla \varphi_m \cdot \nabla \varphi_n,\\
    \norm{\divg (\kappa_h \nabla u_h)}_{L^2(T^q)}^2 &=  \left( \sum_{m,n\in \no{T}} \bfk_m\bfu_n  d_{mn}\right)^2 \int_T \mathrm{d}x.
\end{align*}
Note that $c_m$, $d_{jm}$ and $\int_T \mathrm{d}x$ are independent of the node $i$, i.e., $d_{mn} = d_{m-i,n-i}$, $c_j = c_{j-i}$, $e_{mn} = e_{m-i,n-i}$.
Furthermore,
\begin{align*}
    (\jure_{k,T^q}^2)_i = \norm{\jump{ \kappa_h \nabla u_h}}_{L^2(\partial T^{(q)})}^2 
    &= \sum_{K\in \ed{T}} \norm{\kappa_h}_{L^2(K)}^2 \jump{u_h}^2_{(K)}.
\end{align*}
With triangle $\tilde{T}_K$ such that $K \in \ed{\tilde{T}_K}$ for $K\in\ed{T}$,
\begin{align*}
    \norm{\kappa_h}_{L^2(K)}^2 &= \frac{\sqrt{2}h}{3}\sum_{j,m\in \no{K}} \bfk_{j}\bfk_m,\\
    (\jump{u_h}_{(K)})^2 &= \left(\frac1h\sum_{j\in \no{\tilde{T}}\cup \no{T}} (-1)^{1- \chi_K (j)} \bfu_j\right)^2.
\end{align*}
Since shifting and addition can be represented by convolutional kernels, multiplication and squaring can be approximated by the concatenation of a convolutional kernel, the application of the activation function and another convolutional kernel, every operation can be represented or approximated by a CNN.
Since the addition and shifting always only includes nodes in the vicinity of the considered node, the kernels have a small bounded width and height.
\end{proof}

\subsection{Adaptive FEM approximation}\label{section: afem approx}
The solver and the estimator are combined with a marking strategy and a refinement implemented with masks to arrive at a CNN approximation of the whole AFEM \cref{alg: AFEM}.

\begin{proof}[Proof of \cref{thm: AFEM approximation}]
    To show that \cref{alg: AFEM} can be approximated entirely, the required steps are considered separately.
    Starting with $V = U^1$ and $\bfu = 0$, the algorithm consist of the following steps.
    \begin{enumerate}[label=(\roman*)]
        \item \textbf{Update $\bfu \gets \bfu + \bfv$.} 
        
        By \cref{cor: sol approx}, the solution to $A_\bfy\bfv = \bff - A_\bfy \bfu$ can be approximated up to any $\varepsilon_\text{sol}$ by a CNN $\Psi_{\text{sol}}$ with input images ${\bfk_\bfy}_{\text{img}},(\bff-A_\bfy\bfu)_\img$ and number of parameters bounded by $M(\Psi_{\text{sol}}) \lesssim L\log(\varepsilon_{\text{sol}}^{-1})/\log(c_L^{-1})$.
        Using \cref{lem: CNN concatenation approximation}, for any $\varepsilon_{\text{cor}}>0$ there exists a CNN $\Psi_{\text{cor}}$ with parameters bounded by $M(\Psi_{\text{cor}}) \lesssim L\log(\varepsilon_{\text{cor}}^{-1})/\log(c_L^{-1})$ such that
        \begin{align*}
            \norm{\Psi_{\text{cor}}\left(\bigtimes_{k=1}^L(\uimk,\utimk,\ubimk,\Upsilon(\kappa_h,\T_k), \bff^k)\right) - (\bfu + \bfv)}_{\infty} \leq \varepsilon_{\text{cor}}.
        \end{align*}
        
        \item \textbf{Local error estimator $\eta^2_T$.}
        
        The error estimator for the updated solution $\bfu$ can be approximated with \cref{thm: estimator approx}.
        For any $\varepsilon_{\text{eta}}>0$ there exists a CNN $\Psi_{\text{est}}$ with number of parameters bounded by $M(\Psi_{\text{est}})\lesssim L$ such that
        \begin{align}\label{eqation: estimator approximation in AFEM approximation proof}
            \norm{\Psi_{\text{est}}\left(\bigtimes_{k=1}^L (\uimk,\kimk, \bff^k_{\text{img}})\right)[\ell]-  \eta_\ell^2}_\infty \leq \varepsilon_{\text{eta}}.
        \end{align}
        
        \item \textbf{Marking.}
        
        Different marking strategies can be considered. Here, a threshold marking strategy as in \cref{definition: threshold marking} is used.
        The operator mapping estimator images to marker images inside a CNN is defined by the mapping $\eta_k^2 \in \mathbb{R}^{2\times\I_U^k}\mapsto M^k \in \{0,1\}^{2\times \I_U^k}$ on each level $k=1,\dots, L$, where for $q=1,2$ $M^k[q]_{i} = 1$ if $\eta_k^2[q]_i > \delta_k$ and $M^k[q]_i=0$ otherwise for $i\in\I_U^k$.
        Let the CNN $\Psi_{\text{mark}}$ be composed of
        a convolutional layer subtracting the threshold and approximation error of the error estimator $\delta_k - \varepsilon_{\text{eta}}$ followed by the application of a heavyside activation function $h_{0,1}$.
        It then maps the approximated error estimator $\tilde\eta_k^2$ to marker images $\Psi_{\text{mark}}(\tilde{\eta}_k^2)\in \{ 0,1 \}^{2\times \I_U^k}$, such that the inequality
        $$M^k\leq \Psi_{\text{mark}}(\tilde\eta_k^2)$$
        holds entrywise.
        This can be derived by considering that $M^k[q]_i = 1$ implies $\eta_k^2[q]_i>\delta_k$ and $\tilde\eta_k^2[q]_i \geq \eta_k^2[q]_i - \varepsilon_{\text{eta}}$ holds with \eqref{eqation: estimator approximation in AFEM approximation proof}.
        This yields $\tilde\eta_k^2[q]_i - (\delta_k - \varepsilon_{\text{eta}}) >0$ and therefore $\Psi_{\text{mark}}(\tilde\eta_k^2)[q]_i = h_{0,1}(\tilde\eta_k^2[q] - (\delta_k-\varepsilon_{\text{eta}})) = 1$ for $q=1,2$ and $i\in\I_U^k$ for $k=1,\dots, L$.
        Therefore, all triangles marked with the threshold marking strategy with the true estimator of the true solution are also marked by the CNN.
        As discussed in \cref{remark: other marking strategies in AFEM approx}, other global marking strategies such as D\"orfler marking could be implemented outside of the network.
        In each step the chosen strategy has to incorporate the error in the estimator approximation. For instance, in addition to the markings based on the selected strategy and the approximated estimator, one could mark all triangles for which the approximated error estimator is larger than the lowest approximated estimator of the already marked triangles minus twice the estimated approximation error.
        
        \item \textbf{Refinement.} 
        
        The refinement is incorporated in the CNN using the marking masks $M^k\in\{0,1\}^{2\times \I_U^k}$ on each level, corresponding to piecewise constant functions discretized as in \eqref{eq: DG0}. These are mapped to a mask $M_V^{k+1}\in\{0,1\}^{\I_U^{k+1}}$ corresponding to continuous piecewise linear functions as in \eqref{eq: v decomposition}. To mimic the refinement, they should fulfill $(M_V^{k+1})_i = 1$ if $\supp\varphi_i^{k+1}\cap T_{kj}^q \neq \emptyset$ for some $q=1,2, j\in\I_U^k$ with $M^k[q]_j = 1$ as described in \cref{section: refinement}. 
        
        This mapping can be constructed in two steps. First, $M^k$ is mapped to some $\bar{M}^{k+1}\in\mathbb{R}^{\I_U^{k+1}}$, which is $>0$ on on the nodes corresponding to the required $\varphi^{k+1}_i$ and zero otherwise. This can be done by applying a transpose convolution with stride $2$ and a kernel of size $2\times 1 \times 3 \times 3$. Secondly, the heaviside function $h_{0,1}$ can be applied entrywise to arrive at the desired $0/1$-masks $M_V^{k+1}$.
        
        The derived function spaces satisfy $\tilde V^k\supset V^k$ since the approximate marking covers the exact marking based on the true error estimator of the true Galerkin solution on the current space. 
    \end{enumerate}
    
    We can now combine all above estimations for the components of the AFEM.
    Concatenating these steps into one CNN as in~\cref{lem: CNN concatenation approximation} leads to a CNN $\Psi$ such that
        \begin{align*}
            &\norm{\fem\left(\Psi\left(\bigtimes_{k=1}^L(\uimk,\utimk,\ubimk,\Upsilon(\kappa_h,\T_k), \bff^k)\right)\right) - u}_{H^1_0} \\
            &\leq \norm{\fem\left(\Psi\left(\bigtimes_{k=1}^L(\uimk,\utimk,\ubimk,\Upsilon(\kappa_h,\T_k), \bff^k)\right)\right) - \fem(\bfu_{\tilde{V}})}_{H^1_0} + \norm{\fem(\bfu_{\tilde{V}}) - u}_{H^1_0}\\
            &\leq \varepsilon_{\text{cor}} + \norm{ \fem (\bfu_{V}) - u}_{H_0^1}\\
            & = \varepsilon_{\text{cor}} + \norm{\mathrm{AFEM}(U^1,L) - u}_{H^1_0},
        \end{align*}
        where $\bfu_{W}$ is defined as the coefficients of the Galerkin projection of $u(\cdot,\bfy)$ onto $W$ for some $W\subset H_0^1$, $V$ is the space constructed by the $\mathrm{AFEM}$ after $L\in\mathbb{N}$ steps and $\tilde V$ is the the space constructed by the CNN after the same number of steps.
\end{proof}

\end{document}